\documentclass{article}

\usepackage[utf8]{inputenc} 
\usepackage[T1]{fontenc}    
\usepackage{hyperref}       
\usepackage{url}            
\usepackage{booktabs}       
\usepackage{amsfonts}       
\usepackage{nicefrac}       
\usepackage{microtype}      
\usepackage[table, dvipsnames]{xcolor}
\usepackage[margin=1in]{geometry}
\usepackage{natbib}

\usepackage{multirow}
\usepackage{makecell}
\usepackage{threeparttable}

\usepackage{array}
\newcommand{\PreserveBackslash}[1]{\let\temp=\\#1\let\\=\temp}
\newcolumntype{C}[1]{>{\PreserveBackslash\centering}p{#1}}
\usepackage{hhline}
\newcolumntype{?}{!{\vrule width 1pt}}

\usepackage{amsmath}
\usepackage{amsthm}
\usepackage{amssymb}
\usepackage{thmtools}
\usepackage{thm-restate}
\usepackage[title]{appendix}
\usepackage{comment}
\usepackage{dsfont}
\usepackage{cleveref}

\usepackage{algorithm}
\usepackage{algpseudocode}

\definecolor{darkblue}{rgb}{0,0,0.95}
\algdef{SE}[IF]{CIF}{ENDIF}%
   [2][default]{\textcolor{darkblue}{\algorithmicif\ #2\ \algorithmicthen}}%
   {\algorithmicend\ \algorithmicif}%

\usepackage{thmtools}
\usepackage{thm-restate}

\usepackage{subfigure}
\usepackage{wrapfig}

\newtheorem{theorem}{Theorem}
\newtheorem{lemma}{Lemma}

\newtheorem{fact}{Fact}

\newtheorem{definition}{Definition}
\newenvironment{proofsketch}{%
  \proof}{\endproof}
\usepackage{bbm}

\newtheorem{theorem-rst}[theorem]{Theorem}
\newtheorem{lemma-rst}[lemma]{Lemma}
\newtheorem{proposition-rst}[lemma]{Proposition}
\newtheorem{assumption-rst}[lemma]{Assumption}
\newtheorem{claim-rst}[lemma]{Claim}
\newtheorem{corollary-rst}[lemma]{Corollary}

\usepackage{mathtools}
\DeclarePairedDelimiter\br{(}{)}
\DeclarePairedDelimiter\brs{[}{]}
\DeclarePairedDelimiter\brc{\{}{\}}

\DeclarePairedDelimiter\abs{\lvert}{\rvert}

\DeclarePairedDelimiter\floor{\lfloor}{\rfloor}
\DeclarePairedDelimiter\ceil{\lceil}{\rceil}

\newcommand{\E}{\mathbb{E}}
\newcommand{\R}{\mathbb{R}}
\newcommand{\G}{\mathbb{G}}
\newcommand{\D}{\mathcal{D}}
\newcommand{\F}{\mathcal{F}}
\newcommand{\K}{\mathcal{K}}

\newcommand{\Ocal}{\mathcal{O}}

\newcommand{\I}{\mathcal{I}}
\newcommand{\J}{\mathcal{J}}

\newcommand{\Narms}{K}

\newcommand{\dr}[1]{\Delta_{#1}}
\newcommand{\Ind}[1]{\mathds{1}\brc*{#1}}

\newcommand{\hmu}[2]{\hat\mu_{#1}\br*{#2}}

\newcommand{\unu}{{\underline{\nu}}} 

\newcommand{\kl}{\mathrm{KL}}
\newcommand{\klBin}{d}
\newcommand{\modklBin}{\klBin_\epsilon}

\usepackage[colorinlistoftodos]{todonotes}

\newcommand{\Beta}{\mathrm{Beta}}

\makeatletter
\newcommand{\printfnsymbol}[1]{%
  \textsuperscript{\@fnsymbol{#1}}%
}

\title{Lenient Regret for Multi-Armed Bandits}
\author{
  Nadav Merlis\\
  Technion -- Institute of Technology\\
 \texttt{merlis@campus.technion.ac.il} \\
  \and
  Shie Mannor \\
  Technion -- Institute of Technology\\
  Nvidia Research, Israel\\
 \texttt{shie@ee.technion.ac.il} \\
  }

\date{}

\begin{document}

\maketitle

\begin{abstract}
We consider the Multi-Armed Bandit (MAB) problem, where an agent sequentially chooses actions and observes rewards for the actions it took. While the majority of algorithms try to minimize the regret, i.e., the cumulative difference between the reward of the best action and the agent's action, this criterion might lead to undesirable results. For example, in large problems, or when the interaction with the environment is brief, finding an optimal arm is infeasible, and regret-minimizing algorithms tend to over-explore. To overcome this issue, algorithms for such settings should instead focus on playing near-optimal arms. To this end, we suggest a new, more lenient, regret criterion that ignores suboptimality gaps smaller than some $\epsilon$. We then present a variant of the Thompson Sampling (TS) algorithm, called $\epsilon$-TS, and prove its asymptotic optimality in terms of the lenient regret. Importantly, we show that when the mean of the optimal arm is high enough, the lenient regret of $\epsilon$-TS is bounded by a constant. Finally, we show that $\epsilon$-TS can be applied to improve the performance when the agent knows a lower bound of the suboptimality gaps.
\end{abstract}

\section{Introduction}

Multi-Armed Bandit (MAB) problems are sequential decision-making problems where an agent repeatedly chooses an action (`arm'), out of $\Narms$ possible actions, and observes a reward for the selected action \citep{robbins1952some}. In this setting, the agent usually aims to maximize the expected cumulative return throughout the interaction with the problem. Equivalently, it tries to minimize its regret, which is the expected difference between the best achievable total reward and the agent's actual returns.

Although regret is the most prevalent performance criterion, many problems that should intuitively be `easy' suffer from both large regret and undesired behavior of regret-minimizing algorithms. Consider, for example, a problem where most arms are near-optimal and the few remaining ones have extremely lower rewards. For most practical applications, it suffices to play any of the near-optimal arms, and identifying such arms should be fairly easy. However, regret-minimizing algorithms only compare themselves to the optimal arm. Thus, they must identify an optimal arm with high certainty, or they will suffer linear regret. This leads to two undesired outcomes: (i) the regret fails to characterize the difficulty of such problems, and (ii) regret-minimizing algorithms tend to over-explore suboptimal arms.

\textbf{Regret fails as a complexity measure:} It is well known that for any reasonable algorithm, the regret dramatically increases as the suboptimality gaps shrink, i.e., the reward of some suboptimal arms is very close to the reward of an optimal one \citep{lai1985asymptotically}. Specifically in our example, if most arms are almost-optimal, then the regret can be arbitrarily large. In contrast, finding a near-optimal solution in this problem is relatively simple. Thus, the regret falsely classifies this easy problem as a hard one.

\textbf{Regret-minimizing algorithms over-explore:} As previously stated, any regret-minimizing agent must identify an optimal arm with high certainty or suffer a linear regret. To do so, the agent must thoroughly explore all suboptimal arms. In contrast, if playing near-optimal arms is adequate, identifying one such arm can be done much more efficiently. Importantly, this issue becomes much more severe in large problems or when the interaction with the problem is brief.

The origin of both problems is the comparison of the agent's reward to the optimal reward. Nonetheless, not all bandit algorithms rely on such comparisons. Notably, when trying to identify good arms (`best-arm identification'), many algorithms only attempt to output $\epsilon$-optimal arms, for some predetermined error level $\epsilon>0$ \citep{even2002pac}. However, this criterion only assesses the quality of the output arms and is unfit when we want the algorithm to choose near-optimal arms throughout the interaction.

In this work, we suggest bringing the leniency of the  $\epsilon$-best-arm identification into regret criteria. Inspired by the $\epsilon$-optimality relaxation in best-arm identification, we define the notion of \emph{lenient regret}, that only penalizes arms with gaps larger than $\epsilon$. 
Intuitively, ignoring small gaps alleviates both previously-mentioned problems: first, arms with gaps smaller than $\epsilon$ do not incur lenient regret, and if all other arms have extremely larger gaps, then the lenient regret is expected to be small. Second, removing the penalty from near-optimal arms allows algorithms to spend less time on exploration of bad arms. Then, we expect that algorithms will spend more time playing near-optimal arms.

From a practical perspective, optimizing a more lenient criterion is especially relevant when near-optimal solutions are sufficient while playing bad arms is costly. Consider, for example, a restaurant-recommendation problem. For most people, restaurants of similar quality are practically the same. On the other hand, the cost of visiting bad restaurants is very high. Then, a more lenient criterion should allow focusing on avoiding the bad restaurants, while still recommending restaurants of similar quality. 

In the following sections, we formally define the lenient regret and prove a lower bound for this criterion that dramatically improves the classical lower bound \citep{lai1985asymptotically} as $\epsilon$ increases. Then, inspired by the form of the lower bound, we suggest a variant of the Thompson Sampling (TS) algorithm \citep{thompson1933likelihood}, called $\epsilon$-TS, and prove that its regret asymptotically matches the lower bound, up to an absolute constant. Importantly, we prove that when the mean of the optimal arm is high enough, the lenient regret of $\epsilon$-TS is bounded by a constant. We also provide an empirical evaluation that demonstrates the improvement in performance of $\epsilon$-TS, in comparison to the vanilla TS. Lastly, to demonstrate the generality of our framework, we also show that our algorithm can be applied when the agent has access to a lower bound of all suboptimality gaps. In this case, $\epsilon$-TS greatly improves the performance even in terms of the standard regret.

\subsection{Related Work}
For a comprehensive review of the MAB literature, we refer the readers to \citep{bubeck2012regret,lattimore2018bandit,slivkins2019introduction}. MAB algorithms usually focus on two objectives: regret minimization \citep{auer2002finite,garivier2011kl,kaufmann2012thompson} and best-arm identification \citep{even2002pac,mannor2004sample,gabillon2012best}. Intuitively, the lenient regret can be perceived as a weaker regret criterion that borrows the $\epsilon$-optimality relaxation from best-arm identification. Moreover, we will show that in some cases, the lenient regret aims to maximize the number of plays of $\epsilon$-optimal arms. Then, the lenient regret is the most natural adaptation of the $\epsilon$-best-arm identification problem to a regret minimization setting. Another related concept is the satisficing regret \citep{russo2018satisficing} -- a Bayesian discounted regret criterion that do not penalize a predetermined distortion level. However, their work analyzes a Bayesian regret, through an information ratio \citep{russo2016information}, while we work in a frequentist setting. Moreover, we focus on gap-dependent regret bounds, which is left as an open problem in \citep{russo2018satisficing} . Thus, the two works complements each other.

Another related concept can be found in sample complexity of Reinforcement Learning (RL) \citep{kakade2003sample,lattimore2013sample,dann2015sample,dann2017unifying}. In the episodic setting, this criterion maximizes the number of episodes where an $\epsilon$-optimal policy is played, and can therefore be seen as a possible RL-formulation to our criterion. However, the results for sample complexity significantly differ from ours -- first, the lenient regret allows representing more general criteria than the number of $\epsilon$-optimal plays. Second, in the RL setting, algorithms focus on the dependence in $\epsilon$ and in the size of the state and action spaces, while we derive bounds that depend on the suboptimality gaps. Finally, we show that when the optimal arm is large enough, the lenient regret is constant, and to the best of our knowledge, there is no equivalent result in RL. In some sense, our work can be viewed as a more fundamental analysis of sample complexity that will hopefully allow deriving more general results in RL.

To minimize the lenient regret, we devise a variant of the Thompson Sampling algorithm \citep{thompson1933likelihood}. The vanilla algorithm assumes a prior on the arm distributions, calculates the posterior given the observed rewards and chooses arms according to their probability of being optimal given their posteriors. Even though the algorithm is Bayesian in nature, its regret is asymptotically optimal for any fixed problem \citep{kaufmann2012thompson,agrawal2013further,korda2013thompson}. The algorithm is known to have superior performance in practice \citep{chapelle2011empirical} and has variants for many different settings, i.e., linear bandits \citep{agrawal2013thompson}, combinatorial bandits \citep{wang2018thompson} and more. 
For a more detailed review of TS algorithms and their applications, we refer the readers to \citep{russo2018tutorial}. In this work, we present a generalization of the TS algorithm, called $\epsilon$-TS, that minimizes the lenient regret when ignoring gaps smaller than $\epsilon$. Specifically, when $\epsilon=0$, our approach recovers the vanilla TS.

As previously stated, we also prove that if all gaps are larger than a known $\epsilon>0$, then our algorithm improves the performance also in terms of the standard regret. Specifically, we prove that the regret of $\epsilon$-TS is bounded by a constant when the optimal arm is larger than $1-\epsilon$. This closely relates to the results of \citep{bubeck2013bounded}, which proved constant regret bounds when the algorithm knows both the mean of the optimal arm and a lower bound on the gaps. This was later extended in \citep{lattimore2014bounded} for more general structures. Notably, one can apply the results of \citep{lattimore2014bounded} to derive constant regret bounds when all gaps are larger than $\epsilon$ and the optimal arm is larger than $1-\epsilon$. Nonetheless, and to the best of our knowledge, we are the first to demonstrate improved performance also when the optimal arm is smaller than $1-\epsilon$.

\section{Setting}
We consider the stochastic multi-armed bandit problem with $\Narms$ arms and arm distributions $\unu=\brc*{\nu_a}_{a=1}^\Narms$.  At each round, the agent selects an arm $a\in\brs*{\Narms}\triangleq\brc*{1,\dots,\Narms}$. Then, it observes a reward generated from a fixed distribution $\nu_a$, independently at random of other rounds. Specifically, when pulling an arm $a$ on the $n^{th}$ time, it observes a reward $X_{a,n}\sim \nu_a$. We assume that the rewards are bounded in $X_{a,n}\in\brs*{0,1}$ and have expectation $\E\brs*{X_{a,n}}=\mu_a$. We denote the empirical mean of an arm $a$ using the $n$ first samples by $\hat\mu_{a,n}=\frac{1}{n}\sum_{k=1}^{n}X_{a,k}$ and define $\hat\mu_{a,0}=0$. We also denote the mean of an optimal arm by $\mu^*=\max_a\mu_a$ and the suboptimality gap of an arm $a$ by $\dr{a}=\mu^*-\mu_a$. 

Let $a_t$ be the action chosen by the agent at time $t$. For brevity, we write its gap by $\dr{t}=\dr{a_t}$. Next, denote the observed reward after playing $a_t$ by $X_t=X_{a_t,N_{a_t}(t+1)}$, where $N_a(t)=\sum_{\tau=1}^{t-1}\Ind{a_\tau=a}$ is the number of times an arm $a$ was sampled up to time $t-1$. We also let $\hat\mu_{a}(t)=\hat\mu_{a,N_a(t)}$, the empirical mean of arm $a$ before round $t$, and denote the sum over the observed rewards of $a$ up to time $t-1$ by $S_a(t)=\sum_{k=1}^{N_a(t)} X_{a,k} = N_a(t)\hmu{a}{t}$. Finally, we define the natural filtration $\F_t=\sigma\br*{a_1,X_1,\dots,a_t,X_t}$.

Similarly to other TS algorithms, we work with Beta priors. When initialized with parameters $\alpha=\beta=1$, $p~\sim\Beta(\alpha,\beta)$ is a uniform distribution. Then, if $p$ is the mean of $N$ Bernoulli experiments, from which there were $S$ `successes' (ones), the posterior of $p$ is $\Beta(S+1,N-S+1)$. We denote the cumulative distribution function (cdf) of the Beta distribution with parameters $\alpha,\beta>0$ by $F^{\Beta}_{\alpha,\beta}$. Similarly, we denote the cdf of the Binomial distribution with parameters $n,p$ by $F^B_{n,p}$ and its probability density function (pdf) by $f^B_{n,p}$. We refer the readers to \Cref{appendix: useful results} for further details on the distributions and the relations between them (i.e., the `Beta-Binomial trick'). We also refer the reader to this appendix for some useful concentration results (Hoeffding's inequality and Chernoff-Hoeffding bound).

Finally, we define the Kullback–Leibler (KL) divergence between any two distributions $\nu$ and $\nu'$ by $\kl(\nu,\nu')$, and let $\klBin(p,q)=p\ln\frac{p}{q}+(1-p)\ln\frac{1-p}{1-q}$ be the KL-divergence between Bernoulli distributions with means $p,q\in\brs*{0,1}$. By convention, if $p<1$ and $q\ge1$, or if $p>0$ and $q=0$, we denote $\klBin(p,q)=\infty$.

\subsection{Regret and Lenient Regret}
\begin{wrapfigure}{r}{0.35\textwidth}
\vspace{-1.5cm}
  \begin{center}
    \includegraphics[width=0.33\textwidth]{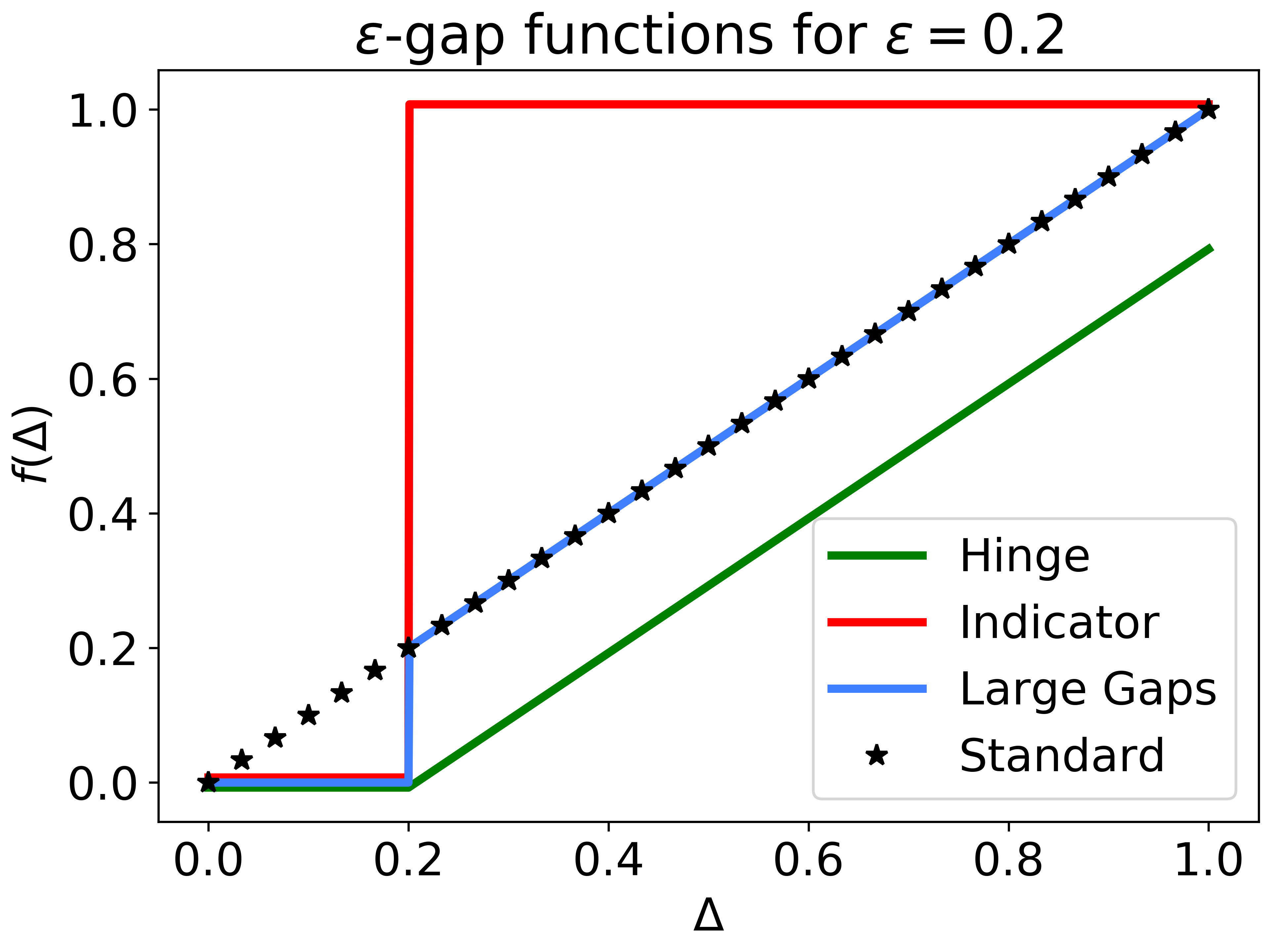}
  \end{center}
  \vspace{-.75cm}
  \caption{Illustration of different $\epsilon$-gap functions, in comparison to the standard regret $f(\dr{})=\dr{}$.}
  \vspace{-.5cm}
  \label{figure:eps-gap}
\end{wrapfigure}
Most MAB algorithms aim to maximize the expected cumulative reward of the agent. Alternatively, algorithms minimize their expected cumulative regret $R(T)= \E\brs*{\sum_{t=1}^T \dr{t}}$. However, and as previously discussed, this sometimes leads to undesired results. Notably, to identify an optimal arm, algorithms must sufficiently explore all suboptimal arms, which is sometimes infeasible. Nonetheless, existing lower bounds for regret-minimizing algorithms show that any reasonable algorithm cannot avoid such exploration \citep{lai1985asymptotically}. To overcome this issue, we suggest minimizing a weaker notion of regret that ignores small gaps. This will allow finding a near-optimal arm much faster. We formally define this criterion as follows:

\begin{definition}
For any $\epsilon\in\brs*{0,1}$, a function $f:\brs*{0,1}\rightarrow\R_+$ is called an $\epsilon$-gap function if $f(\dr{})=0$ for all $\dr{}\in\brs*{0,\epsilon}$ and $f(\dr{})>0$ for all $\dr{}>\epsilon$. The lenient regret w.r.t. an $\epsilon$-gap function $f$ is defined as $R_f(T)=\E\brs*{\sum_{t=1}^T f\br*{\dr{t}}}$.
\end{definition}

While it is natural to require of $f$ to increase with $\dr{}$, this assumption is not required for the rest of the paper. Moreover, assuming that $f(\dr{})>0$ for all $\dr{}>\epsilon$ is only required for the lower bound; for the upper bound, it can be replaced by $f(\dr{})\ge0$ when $\dr{}>\epsilon$. There are three notable examples for $\epsilon$-gap functions (see also \Cref{figure:eps-gap} for graphical illustration). First, the most natural choice for an $\epsilon$-gap function is the hinge loss $f(\dr{})=\max\brc*{\dr{}-\epsilon,0}$, which ignores small gaps and increases linearly for larger gaps. 
Second, we are sometimes interested in maximizing the number of steps where $\epsilon$-optimal arms are played. In this case, we can choose $f(\dr{})=\Ind{\dr{}>\epsilon}$. This can be seen as the natural adaptation of $\epsilon$-best-arm identification into a regret criterion. Importantly, notice that this criterion only penalizes sampling of arms with gaps larger than $\epsilon$. This comes with a stark contrast to best-arm identification, where \emph{all} samples are penalized, whether they are of $\epsilon$-optimal arms or not. 
Finally, we can choose $f(\dr{}) = \dr{}\cdot\Ind{\dr{}>\epsilon}$. Importantly, when all gaps are larger than $\epsilon$, then this function leads to the standard regret. Thus, all results for $\epsilon$-gap functions also hold for the standard regret when $\dr{a}>\epsilon$ for all suboptimal arms.

There are two ways for relating the lenient regret to the standard regret. First, notice that the standard regret can be represented through the $0$-gap function $f(\dr{})=\dr{}$. Alternatively, the standard regret can be related to lenient regret w.r.t. the indicator gap-function:
\begin{restatable}{claim-rst}{epsToStandardRelation}
\label{claim: epsilon to standard relation}
Let $R(T)= \E\brs*{\sum_{t=1}^T \dr{t}}$ be the standard regret and define $f_\epsilon(\dr{}) = \Ind{\dr{}>\epsilon}$. Then,
\begin{align*}
    R(T) = \int_{\epsilon=0}^1 R_{f_\epsilon}(T)d\epsilon \enspace.
\end{align*}
\end{restatable}
The proof is in \Cref{appendix: epsilon to standard relation}. Specifically, it implies that the standard regret aims to minimize the average lenient regret over different leniency levels. In contrast, our approach allows choosing which leniency level to minimize according to the specific application. By doing so, the designer can adjust the algorithm to its needs, instead of using  an algorithm that minimizes the average performance.

\section{Lower Bounds}
\label{section:lower bounds}
In this section, we prove a problem-dependent lower bound for the lenient regret. Notably, when working with $\epsilon$-gap functions with $\epsilon>0$, we prove that the lower bound behaves inherently different than the case of $\epsilon=0$. Namely, for some problems, the lower bound is sub-logarithmic, in contrast to the $\Omega(\ln T)$ bound for the standard regret.

To prove the lower bound, we require some additional notations. Denote by $\D$, a set of distributions over $\brs*{0,1}$ such that $\nu_a\in\D$ for all $a\in\brs*{\Narms}$. A bandit strategy is called \emph{consistent} over $\D$ w.r.t. an $\epsilon$-gap function $f$ if for any bandit problem with arm distributions in $\D$ and for any $0<\alpha\le1$, it holds that $R_f(T) = o\br*{T^\alpha}$. Finally, we use $\K_{\mathrm{inf}}$, as was defined in \citep{burnetas1996optimal,garivier2019explore}:
\begin{align*}
    \K_{\mathrm{inf}}(\nu,x,\D) = \inf\brc*{\kl(\nu,\nu'): \nu'\in\D, \E\brs*{\nu'}>x}\enspace,
\end{align*}
and by convention, the infimum over an empty set equals $\infty$. We now state the lower bound:
\begin{restatable}{theorem-rst}{dependentLowerBound}
\label{theorem:dependent_lower_bound}
For any consistent bandit strategy w.r.t. an $\epsilon$-gap function $f$, for all arms $k\in\brs*{\Narms}$ such that $\dr{k}>\epsilon$, it holds that
\begin{align}
\label{eq:dependent bound count}
    \liminf_{T\to\infty}\frac{\E\brs*{N_k(T+1)}}{\ln T} \ge \frac{1}{\K_{\mathrm{inf}}(\nu_k,\mu^*+\epsilon,\D)}\enspace.
\end{align}
Specifically, the lenient regret w.r.t. $f$ is lower bounded by 
\begin{align}
\label{eq:dependent regret lower bound}
    \liminf_{T\to\infty}\frac{R_f(T)}{\ln T} \ge \sum_{a:\dr{a}>\epsilon} \frac{f(\dr{a})}{\K_{\mathrm{inf}}(\nu_a,\mu^*+\epsilon,\D)}\enspace .
\end{align}
\end{restatable}
The proof uses the techniques of \citep{garivier2019explore} and can be found in \Cref{appendix:lower bounds}. Specifically, choosing $\epsilon=0$ leads to the bound for the standard regret \citep{burnetas1996optimal}. As anticipated, both the lenient regret and the number of samples from arms with large gaps decrease as $\epsilon$ increases. This justifies our intuition that removing the penalty from $\epsilon$-optimal arms enables algorithms to reduce the exploration of arms with $\dr{a}>\epsilon$. 

The fact that the bounds decrease with $\epsilon$ leads to another interesting conclusion -- any algorithm that matches the lower bound for some $\epsilon$ is \emph{not} consistent for any $\epsilon'<\epsilon$, since it breaks the lower bound for $\epsilon'$. This specifically holds for the standard regret and implies that there is no `free lunch' -- achieving the optimal lenient regret for some $\epsilon>0$ leads to non-logarithmic standard regret.

Surprisingly, the lower bound is sub-logarithmic when $\mu^*>1-\epsilon$. To see this, notice that in this case, there is no distribution $\nu\in\D$ such that $\E\brs*{\nu}>\mu^*+\epsilon$, and thus $\K_{\mathrm{inf}}(\nu_a,\mu^*+\epsilon,\D)=\infty$. Intuitively, if the rewards are bounded in $[0,1]$ and some arm has a mean $\mu_a>1-\epsilon$, playing it can never incur regret. Identifying that such an arm exists is relatively easy, which leads to low lenient regret. Indeed, we will later present an algorithm that achieves constant regret in this regime.

Finally, and as with most algorithms, we will focus on the set of all problems with rewards bounded in $[0,1]$. In this case, the denominator in \Cref{eq:dependent regret lower bound} is bounded by $\K_{\mathrm{inf}}(\nu_a,\mu^*+\epsilon,\D)\ge\klBin(\mu_a,\mu^*+\epsilon)$ (e.g., by applying Lemma 1 of \citep{garivier2019explore}), and equality holds when the arms are Bernoulli-distributed. Since our results should also hold for Bernoulli arms, our upper bound will similarly depend on $\klBin(\mu_a,\mu^*+\epsilon)$.

\section{Thompson Sampling for Lenient Regret}

\begin{algorithm}
\caption{$\epsilon$-TS for Bernoulli arms}
\label{alg:eps-TS}
\begin{algorithmic}[1]
\State Initialize $N_a(1)=0$, $S_a(1)=0$ and $\hat\mu_a(1)=0$ for all $a\in\brs*{\Narms}$ \Comment{\textcolor{darkblue}{Blue: Changes from TS}}
    \For{$t=1,\dots,T$}
        \For{$a=1\dots,\Narms$}
            \CIF{$\hmu{a}{t}>1-\epsilon$}
                \State \textcolor{darkblue}{$\theta_a(t)=\hmu{a}{t}$ \label{alg line: empirical mean posterior}}
            \Else
                \State $\alpha_a(t) = \floor*{\frac{S_a(t)}{\textcolor{darkblue}{1-\epsilon}}}+1$ \label{alg line: beta posterior start}
                \State $\beta_a(t) = N_a(t) +2 - \alpha_a(t)$
                \State $\theta_a(t) = \textcolor{darkblue}{(1-\epsilon)}Y$ for $Y\sim \Beta(\alpha_a(t),\beta_a(t))$ \label{alg line: beta posterior end}
            \EndIf
        \EndFor
        \State Play $a_t \in \arg\max_a\theta_a(t)$ (break ties randomly) and observe the reward $X_t$
        \State Update $N_{a_t}(t+1) = N_{a_t}(t)+1$, $\;S_{a_t}(t+1) = S_{a_t}(t)+X_t$ and $\hmu{a_t}{t+1} = \frac{S_{a_t}(t+1)}{N_{a_t}(t+1)}$
        \State For all $a\ne a_t$, set $N_{a}(t+1) = N_{a}(t)$, $\;S_{a}(t) = S_{a}(t)$ and $\hmu{a}{t+1} =\hmu{a}{t}$
    \EndFor
\end{algorithmic}
\end{algorithm}

In this section, we present a modified TS algorithm that can be applied with $\epsilon$-gap functions. W.l.o.g., we assume that the rewards are Bernoulli-distributed, i.e., $X_t\in\brc*{0,1}$; otherwise, the rewards can be randomly rounded (see \cite{agrawal2012analysis} for further details). 
To derive the algorithm, observe that the lower bound of \Cref{theorem:dependent_lower_bound} approaches zero as the optimal arm becomes closer to $1-\epsilon$. Specifically, the lower bound behaves similarly to the regret of the vanilla TS with rewards scaled to $[0,1-\epsilon]$. On the other hand, if the optimal arm is above $1-\epsilon$, we would like to give it a higher priority, so the regret in this case will be sub-logarithmic. This motivates the following $\epsilon$-TS algorithm, presented in \Cref{alg:eps-TS}: denote by $\theta_a(t)$, the sample from the posterior of arm $a$ at round $t$, and recall that TS algorithm choose arms by $a_t\in\arg\max_a \theta_a(t)$. For any arm with $\hat\mu_a(t)\le1-\epsilon$, we fix its posterior to be a scaled Beta distribution, such that the range of the posterior is $\brs*{0,1-\epsilon}$, but its mean (approximately) remains $\hat\mu_a(t)$ (lines \ref{alg line: beta posterior start}-\ref{alg line: beta posterior end}). If $\hat\mu_a(t)>1-\epsilon$, we set the posterior to $\theta_a(t)=\hat\mu_a(t)>1-\epsilon$ (line \ref{alg line: empirical mean posterior}), which gives this arm a higher priority than any arm with $\hat\mu_a(t)\le1-\epsilon$. Notice that $\epsilon$-TS \emph{does not} depend on the specific $\epsilon$-gap function. Intuitively, this is since it suffices to match the number of suboptimal plays in \Cref{eq:dependent bound count}, that only depends on $\epsilon$. The algorithm enjoys the following asymptotic lenient regret:
\begin{restatable}{theorem-rst}{dependentUpperBound}
\label{theorem:dependent upper bound asymptotic}
Let $f$ be an $\epsilon$-gap function. Then, the lenient regret of $\epsilon$-TS w.r.t. $f$ is 
\begin{align}
    \limsup_{T\to\infty}\frac{R_f(T)}{\ln T}  
    &\le \sum_{a:\dr{a}>\epsilon} \frac{f(\dr{a})}{\klBin\br*{\frac{\mu_a}{1-\epsilon},\frac{\mu^*}{1-\epsilon}}}
    \le 4(1-\epsilon)\sum_{a:\dr{a}>\epsilon} \frac{f(\dr{a})}{\klBin(\mu_a,\mu^*+\epsilon)}\enspace. \label{eq:upper bound asymptotic}
\end{align}
Moreover, if $\mu^*>1-\epsilon$, then $R_f(T)=\Ocal(1)$.
\end{restatable}
The proof can be found in the following section. In our context, the $\Ocal$ notation hides constants that depend on the mean of the arms and $\epsilon$. Notice that \Cref{theorem:dependent upper bound asymptotic} matches the lower bound of \Cref{theorem:dependent_lower_bound} for the set of all bounded distributions (and specifically for Bernoulli arms), up to an absolute constant. Notably, when $\mu^*>1-\epsilon$, we prove that the regret is constant, and not only sub-logarithmic, as the lower bound suggests. Specifically in this regime, an algorithm can achieve constant lenient regret by identifying an arm with a mean greater than $1-\epsilon$ and exploiting it. However, the algorithm does not know whether such an arm exists, and if there is no such arm, a best arm-identification scheme will perform poorly. Our algorithm naturally identifies such arms when they exist, while maintaining good lenient regret otherwise. Similarly, algorithms such as of \cite{bubeck2013bounded} cannot be applied to achieve constant regret, since they require knowing the value of the optimal arm, which is even a stronger requirement than knowing that $\mu^*>1-\epsilon$.

\begin{wrapfigure}{r}{0.36\textwidth}
  \begin{center}
    \includegraphics[width=0.34\textwidth]{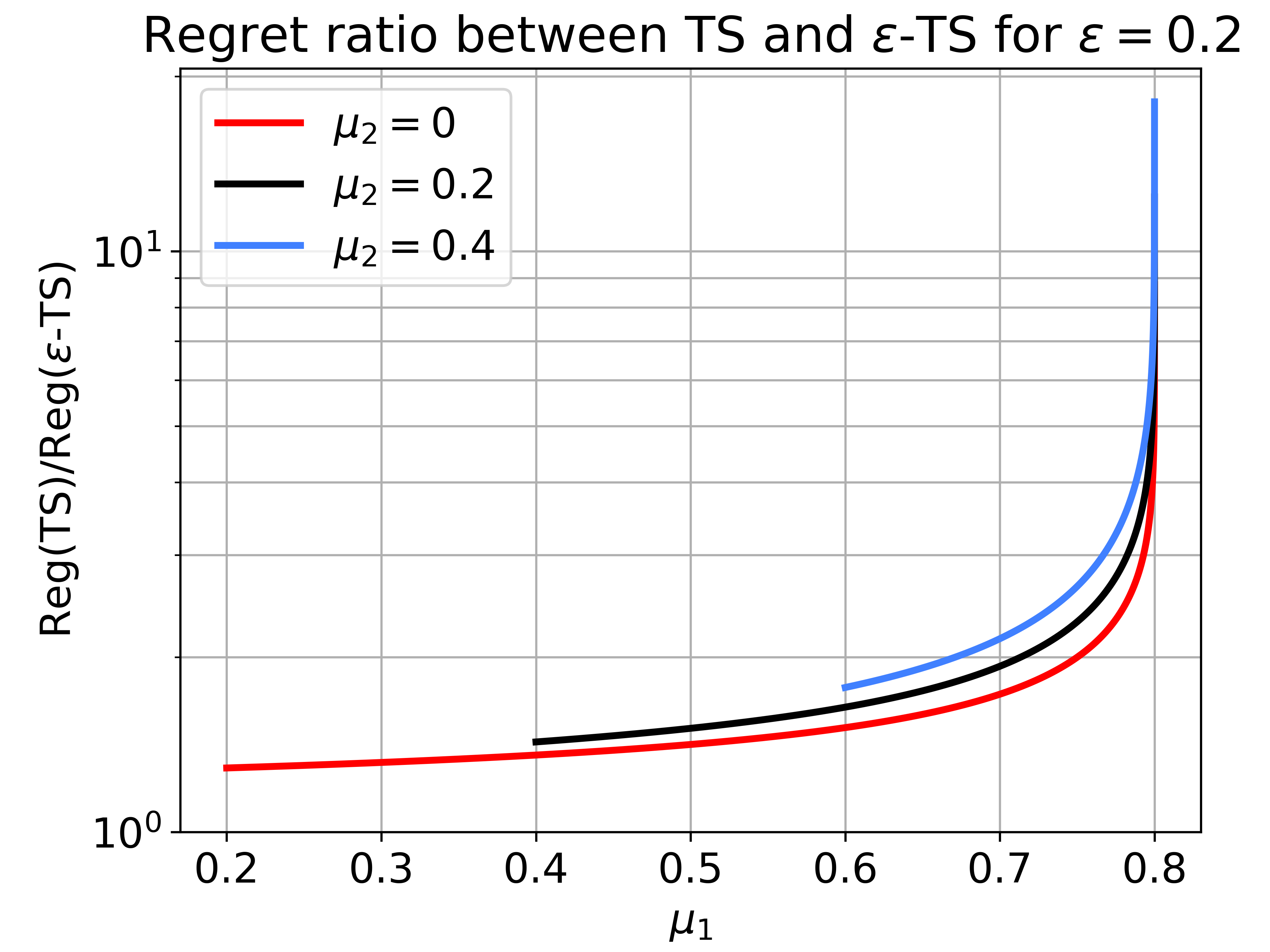}
  \end{center}
  \vspace{-.75cm}
  \caption{Ratio between the asymptotic lenient regret bounds of TS and $\epsilon$-TS for two-armed problems with $\epsilon=0.2$, as a function of the optimal arm $\mu_1$.}
  \vspace{-.3cm}
  \label{figure:regret ratio}
\end{wrapfigure}
\textbf{Comparison to MAB algorithms: } Asymptotically optimal MAB algorithms sample suboptimal arms according to the lower bound, i.e., for any suboptimal arm $a$, $ \limsup_{T\to\infty}\frac{N_a(T)}{\ln T}  
    \le \frac{1}{\klBin\br*{\mu_a,\mu^*}}$.
This, in turn, leads to a lenient regret bound of 
\begin{align}
    \limsup_{T\to\infty}\frac{R_f(T)}{\ln T}  
    &\le \sum_{a:\dr{a}>\epsilon}\frac{f(\dr{a})}{\klBin\br*{\mu_a,\mu^*}} \label{eq:lenient regret standard}
\end{align}
that holds for both the vanilla TS \citep{kaufmann2012thompson} and KL-UCB \citep{garivier2011kl}. First notice that the bound of \Cref{eq:upper bound asymptotic}, that depends on $\klBin\br*{\frac{\mu_a}{1-\epsilon},\frac{\mu^*}{1-\epsilon}}$, strictly improves the bounds for the standard algorithms (see Appendix \ref{appendix:standard lenient comparison} for further details). Moreover, $\epsilon$-TS achieves constant regret when $\mu^*>1-\epsilon$, and its regret quickly diminishes when approaching this regime. This comes in contrast to standard MAB algorithms, that achieve logarithmic regret in these regimes. To illustrate the improvement of $\epsilon$-TS, in comparison to standard algorithms, we present the ratio between the asymptotic bounds of Equations \eqref{eq:lenient regret standard} and \eqref{eq:upper bound asymptotic} in \Cref{figure:regret ratio}.

Before presenting the proof, we return to the $\epsilon$-gap function $f(\dr{}) = \dr{}\cdot\Ind{\dr{}>\epsilon}$. Recall that this function leads to the standard regret when all gaps are larger than $\epsilon$. Thus, our algorithm can be applied in this case to greatly improve the performance (from the bound of \Cref{eq:lenient regret standard} to the bound of \Cref{eq:upper bound asymptotic}), even in terms of the standard regret.

\subsection{Regret Analysis}
\label{section: regret analysis}
In this section, we prove the regret bound of \Cref{theorem:dependent upper bound asymptotic}. For the analysis, we assume w.l.o.g. that the arms are sorted in a decreasing order and all suboptimal arms have gaps $\dr{a}>\epsilon$, i.e. $\mu^*=\mu_1\ge\mu_1-\epsilon>\mu_2\ge\dots\ge\mu_\Narms$. If there are additional arms with gaps $\dr{a}\le\epsilon$, playing them will cause no regret and the overall lenient regret will only decrease (see \Cref{appendix:multiple optimal} or Appendix A in \citep{agrawal2012analysis} for further details). We also assume that $\epsilon<1$, as otherwise $f(\dr{a})=0$ for all $a\in\brs*{\Narms}$. Under these assumptions, we now state a more detailed bound for the lenient regret, that also includes a finite-time behavior:
\begin{theorem}
\label{theorem:dependent upper bound}
Let $f$ be an $\epsilon$-gap function. If $\mu_1>1-\epsilon$, there exists some constants $b=b(\mu_1,\mu_2,\epsilon)\in\br*{0,1}$, $C_b=C_b(\mu_1,\mu_2,\epsilon)$ and $L_1=L_1(\mu_1,\epsilon,b)$ such that
\begin{align}
    R_f(T) 
    &\le \sum_{a=2}^\Narms \frac{f(\dr{a})}{\klBin(1-\epsilon,\mu_a)}
    + \max_af(\dr{a})\br*{C_b+L_1+\frac{\pi^2/6}{\klBin(1-\epsilon,\mu_1)}}
    =\Ocal(1)\enspace. \label{eq:upper bound high opt}
\end{align}
If $\mu_1\le1-\epsilon$, then for any $c>0$, there exist additional constants $L_2=L_2(b,\epsilon)$ and $x_{a,c}=x_{a,c}(\mu_1,\mu_a,\epsilon)$ such that for $\eta(t)=\max\brc*{\mu_1-\epsilon,\mu_1-2\sqrt{\frac{6\ln t}{t^b}}}$,

\begin{align}
    R_f(T) 
    &\!\le \sum_{a=2}^\Narms f(\dr{a})\br*{(1+c)^2\max_{t\in\brs*{T}}\brc*{\frac{\ln t}{\klBin\br*{\frac{\mu_a}{1-\epsilon},\frac{\eta(t)}{1-\epsilon}}}}+2+\frac{1}{c} + \frac{1}{\klBin(x_{a,c},\mu_a)}} 
    + \max_af(\dr{a})\br*{C_b+L_2+6}. \label{eq:upper bound low opt}
\end{align}
\end{theorem}
\begin{proof}
We decompose the regret similarly to \citep{kaufmann2012thompson} and show that with high probability, the optimal arm is sampled polynomially, i.e., $N_1(t)=\Omega(t^b)$ for some $b\in\br*{0,1}$. Formally, let $\eta(t)$ be some function such that $\mu_1-\epsilon\le\eta(t)<\mu_1$ for all $t\in\brs*{T}$, and for brevity, let $f_{\max}=\max_a f(\dr{a})$. Also, recall that the lenient regret is defined as $R_f(T)=\E\brs*{\sum_{t=1}^T f\br*{\dr{t}}}$. Then, the lenient regret can be decomposed to
\begin{align*}
    R_f(&T) 
    = \sum_{t=1}^T  \E\brs*{f(\dr{t})\Ind{\theta_1(t)>\eta(t)}} 
    + \sum_{t=1}^T \E\brs*{f(\dr{t})\Ind{\theta_1(t)\le\eta(t)}}\\
    &\le \sum_{t=1}^T  \E\brs*{f(\dr{t})\Ind{\theta_1(t)>\eta(t)}}
    + f_{\max}\sum_{t=1}^T \E\brs*{\Ind{\theta_1(t)\le\eta(t)}} \\
    &= \sum_{t=1}^T\sum_{a=2}^\Narms f(\dr{a})  \E\brs*{\Ind{a_t=a,\theta_1(t)>\eta(t)}}
    + f_{\max}\sum_{t=1}^T \E\brs*{\Ind{\theta_1(t)\le\eta(t)}} \enspace.
\end{align*}
Replacing the expectations of indicators with probabilities and dividing the second term to the case where $a_1$ was sufficiently and insufficiently sampled, we get
\begin{align}
    R_f(&T) 
    \le \sum_{a=2}^\Narms f(\dr{a}) \underbrace{\sum_{t=1}^T \Pr\brc*{a_t=a,\theta_1(t)> \eta(t)}}_{(A)} 
    +  f_{\max} \underbrace{\sum_{t=1}^T  \Pr\brc*{\theta_1(t)\le \eta(t), N_1(t)>(t-1)^b}}_{(B)} \nonumber\\
    &\qquad+f_{\max}\underbrace{\sum_{t=1}^T  \Pr\brc*{N_1(t)\le (t-1)^b}}_{(C)}\enspace. \label{eq:regret decomposition}
\end{align}
The first part of the proof consists of bounding term $(C)$, i.e., showing that the optimal arm is sampled polynomially with high probability. We do so in the following proposition:
\begin{restatable}{proposition-rst}{optCountBound}
\label{prop:optimal arm count bound}
There exist constants $b=b(\mu_1,\mu_2,\epsilon)\in(0,1)$ and $C_b=C_b(\mu_1,\mu_2,\epsilon)<\infty$ such that
\begin{align*}
    \sum_{t=1}^T  \Pr\brc*{N_1(t)\le (t-1)^b} \le C_b\enspace.
\end{align*}
\end{restatable}
The proof follows the lines of Proposition 1 in \citep{kaufmann2012thompson} and can be found in \Cref{appendix:optCountBound}. To bound $(A)$ and $(B)$, we divide the analysis into two cases: $\mu_1>1-\epsilon$ and $\mu_1\le1-\epsilon$. 

\paragraph{\underline{First case:}} $\mu_1>1-\epsilon$.

In this case, we fix $\eta(t)=1-\epsilon$. For $(A)$, observe that if $a_t=a$ and $\theta_1(t)> 1-\epsilon$, then $\theta_a(t)> 1-\epsilon$, which also implies that $\hmu{a}{t}>1-\epsilon$ (to see this, notice that if $\hmu{a}{t}\le1-\epsilon$, then $\theta_a(t)=(1-\epsilon)Y\le1-\epsilon$). However, since $\dr{a}>\epsilon$ for all $a\ne1$, all suboptimal arms have means $\mu_a<1-\epsilon$. Thus, when $N_a(t)$ becomes large, the probabilities in $(A)$ quickly diminish and this term can be bounded by constant. Formally, we write
\begin{align*}
    \sum_{t=1}^T  \Pr\brc*{a_t=a,\theta_1(t)> 1-\epsilon}
    &\le \sum_{t=1}^T  \Pr\brc*{a_t=a,\theta_a(t)> 1-\epsilon}
    = \sum_{t=1}^T  \Pr\brc*{a_t=a,\hmu{a}{t}> 1-\epsilon},
\end{align*}
and bound this term using the following lemma (see \Cref{appendix:mean large deviation bound} for the proof):
\begin{restatable}{lemma-rst}{meanLargeDev}
\label{lemma:mean large deviation bound}
For any arm $a\in\brs*{\Narms}$, if $x>\mu_a$, then
\begin{align*}
    \sum_{t=1}^T  \Pr\brc*{a_t=a,\hmu{a}{t}>x}
    \le \frac{1}{\klBin(x,\mu_a)}\enspace.
\end{align*}
\end{restatable}
Similarly, in $(B)$, $\theta_1(t)\le 1-\epsilon$ implies that $\hmu{1}{t}\le 1-\epsilon$, and since $N_1(t)$ is large, this event has a low probability. We formalize this intuition in \Cref{lemma:optimal low when high optimal arm}, whose proof can be found in \Cref{appendix:optimal low when high optimal arm}.

\begin{restatable}{lemma-rst}{optLowHighOptimal}
\label{lemma:optimal low when high optimal arm}
Assume that $\mu_1>1-\epsilon$, and for any $b\in\br*{0,1}$, let $L_1(\mu_1,\epsilon,b)$ such that for all $t\ge L_1(\mu_1,\epsilon,b)$, it holds that $(t-1)^b\ge \frac{2\ln t}{\klBin(1-\epsilon,\mu_1)}+1$. Then,
\begin{align*}
    \sum_{t=1}^T  \Pr\brc*{\theta_1(t)\le 1-\epsilon, N_1(t)>(t-1)^b}
    \le L_1(\mu_1,\epsilon,b) + \frac{\pi^2/6}{\klBin(1-\epsilon,\mu_1)}\enspace.
\end{align*}
\end{restatable}
Substituting both lemmas and \Cref{prop:optimal arm count bound} into \Cref{eq:regret decomposition} leads to \Cref{eq:upper bound high opt}.

\paragraph{\underline{Second case:}} $\mu_1\le1-\epsilon$.

For this case, we fix $\eta(t)=\max\brc*{\mu_1-\epsilon,\mu_1-2\sqrt{\frac{6\ln t}{(t-1)^b}}}$. To bound $(A)$, we adapt the analysis of \citep{agrawal2013further} and decompose this term into two parts: (i) the event where the empirical mean $\hmu{a}{t}$ is far above $\mu_{a}$, and (ii) the event where $\hmu{a}{t}$ is close to $\mu_a$ and $\theta_a(t)$ is above $\eta(t)$. Doing so leads to \Cref{lemma:bound when theta high and optimal low}, whose proof is in \Cref{appendix:bound when theta high and optimal low}:
\begin{restatable}{lemma-rst}{boundHighThetaLowOpt}
\label{lemma:bound when theta high and optimal low}
Assume that $\mu_1\le1-\epsilon$ and $\eta(t)\in[\mu_1-\epsilon,\mu_1)$ for all $t\in\brs*{T}$. Then, for any $c>0$,
\begin{align*}
    \sum_{t=1}^T  \Pr\brc*{a_t=a,\theta_1(t)> \eta(t)}
    \le (1+c)^2\max_{t\in\brs*{T}}\brc*{\frac{\ln t}{\klBin\br*{\frac{\mu_a}{1-\epsilon},\frac{\eta(t)}{1-\epsilon}}}}+2+\frac{1}{c} + \frac{1}{\klBin(x_{a,c},\mu_a)}\enspace,
\end{align*}
where $x_{a,c}\in\br*{\mu_a,\mu_1-\epsilon}$ is such that $\klBin\br*{\frac{x_{a,c}}{1-\epsilon},\frac{\mu_1-\epsilon}{1-\epsilon}} = \frac{1}{1+c}\klBin\br*{\frac{\mu_a}{1-\epsilon},\frac{\mu_1-\epsilon}{1-\epsilon}}$.
\end{restatable}
For $(B)$, we provide the following lemma (see \Cref{appendix:prob that theta low when optimal low} for the proof):
\begin{restatable}{lemma-rst}{ProbThetaLow}
\label{lemma:prob that theta low when optimal low}
Assume that $\mu_1\le1-\epsilon$ and let $\eta(t)=\max\brc*{\mu_1-\epsilon,\mu_1-2\sqrt{\frac{6\ln t}{(t-1)^b}}}$. Also, let $L_2(b,\epsilon)\ge2$ such that for all $t\ge L_2(b,\epsilon)$, it holds that $\eta(t)>\mu_1-\epsilon$. Then,
\begin{align*}
    \sum_{t=1}^T  \Pr\brc*{\theta_1(t)\le \eta(t), N_1(t)>(t-1)^b}
    \le L_2(b,\epsilon)+6
\end{align*}
\end{restatable}
Substituting both lemmas and \Cref{prop:optimal arm count bound} into \Cref{eq:regret decomposition} results with \Cref{eq:upper bound low opt} and concludes the proof of \Cref{theorem:dependent upper bound}.
\end{proof}
\begin{proofsketch}[Proof sketch of \Cref{theorem:dependent upper bound asymptotic}]
It only remains to prove the asymptotic rate of \Cref{theorem:dependent upper bound asymptotic}, using the finite-time bound of \Cref{theorem:dependent upper bound}. To do so, notice that the denominator in \Cref{eq:upper bound low opt} asymptotically behaves as $\klBin\br*{\frac{\mu_a}{1-\epsilon},\frac{\mu_1}{1-\epsilon}}$, which leads to the first bound of the theorem. On the other hand, the denominator of the second bound depends on $\klBin\br*{\mu_a,\mu_1+\epsilon}$. We prove that when $\dr{a}>\epsilon$, these two quantities are closely related:
\begin{restatable}{lemma-rst}{modKlProperties}
\label{lemma:mod-kl-properties}
 For any $\epsilon\in\left[0,\frac{1}{2}\right)$, any  $p\in[0,1-2\epsilon)$ and any $q\in[p+\epsilon,1-\epsilon)$, 
\begin{align*}
    \klBin\br*{\frac{p}{1-\epsilon},\frac{q}{1-\epsilon}}\ge \frac{1}{4(1-\epsilon)}\klBin(p,q+\epsilon) \enspace .
\end{align*}
\end{restatable}
The proof of this lemma can be found in \Cref{appendix:mod-kl-properties}. This immediately leads to the desired asymptotic rate, but for completeness, we provide the full proof of the theorem in \Cref{appendix:dependent upper bound asymptotic}.
\end{proofsketch}

\section{Experiments}
\label{section:experiments}
In this section, we present an empirical evaluation of $\epsilon$-TS. Specifically, we compare $\epsilon$-TS to the vanilla TS on two different gap functions: $f(\dr{})=\dr{}$, which leads to the standard regret, and the hinge function $f(\dr{})=\max\brc*{\dr{}-\epsilon,0}$. All evaluations were performed for $\epsilon=0.2$ over $50,000$ different seeds and are depicted in \Cref{figure:experiments}. We also refer the readers to \Cref{appendix:experiments}, where additional statistics of the simulations are presented, alongside additional tests that were omitted due to space limits. We tested 4 different scenarios -- when the optimal arm is smaller or larger than $1-\epsilon$ (left and right columns, respectively), and when the minimal gap is larger or smaller than $\epsilon$ (top and bottom rows, respectively). Importantly, when the minimal gap is larger than $\epsilon$, the standard regret can be written using the $\epsilon$-gap function $f(\dr{})=\dr{}\cdot\Ind{\dr{}>\epsilon}$. Indeed, one can observe that when $\dr{a}>\epsilon$ for all suboptimal arms, $\epsilon$-TS greatly improves the performance, in comparison to the vanilla TS. Similarly, when $\mu^*>1-\epsilon$, the lenient regret of $\epsilon$-TS converges to a constant, as can be expected from \Cref{theorem:dependent upper bound asymptotic}. On the other hand, the lenient regret of the vanilla TS continues to increase.

Next, we move to simulations where the suboptimality gap is smaller than $\epsilon$. In such cases, the standard regret cannot be represented as an $\epsilon$-gap function, and $\epsilon$-TS is expected to perform worse on this criterion than the vanilla TS. Quite surprisingly, when $\mu^*=0.5$, $\epsilon$-TS still surpasses the vanilla TS. In \Cref{appendix:experiments}, we show that TS beats $\epsilon$-TS only after $20,000$ steps. On the other hand, when $\mu^*=0.9$, the standard regret of $\epsilon$-TS increases linearly. This is since with finite probability, the algorithm identifies that $\mu_2=0.85>1-\epsilon$ at a point where the empirical mean of the optimal arm is smaller than $1-\epsilon$. Then, the algorithm only exploits $a=2$ and will never identify that $a=1$ is the optimal arm. Nonetheless, we emphasize that $\epsilon$-TS still outperforms the vanilla TS in terms of the lenient regret, as can be observed for the hinge-function.

To conclude this section, the simulations clearly demonstrate the tradeoff when optimizing the lenient regret: when near-optimal solutions are adequate, then the performance can be greatly improved. On the other hand, in some cases, it leads to major degradation in the standard regret.

\begin{figure}[t]
\centering
\subfigure{
\includegraphics[trim=0 15 0 15,clip,width=0.37\linewidth]{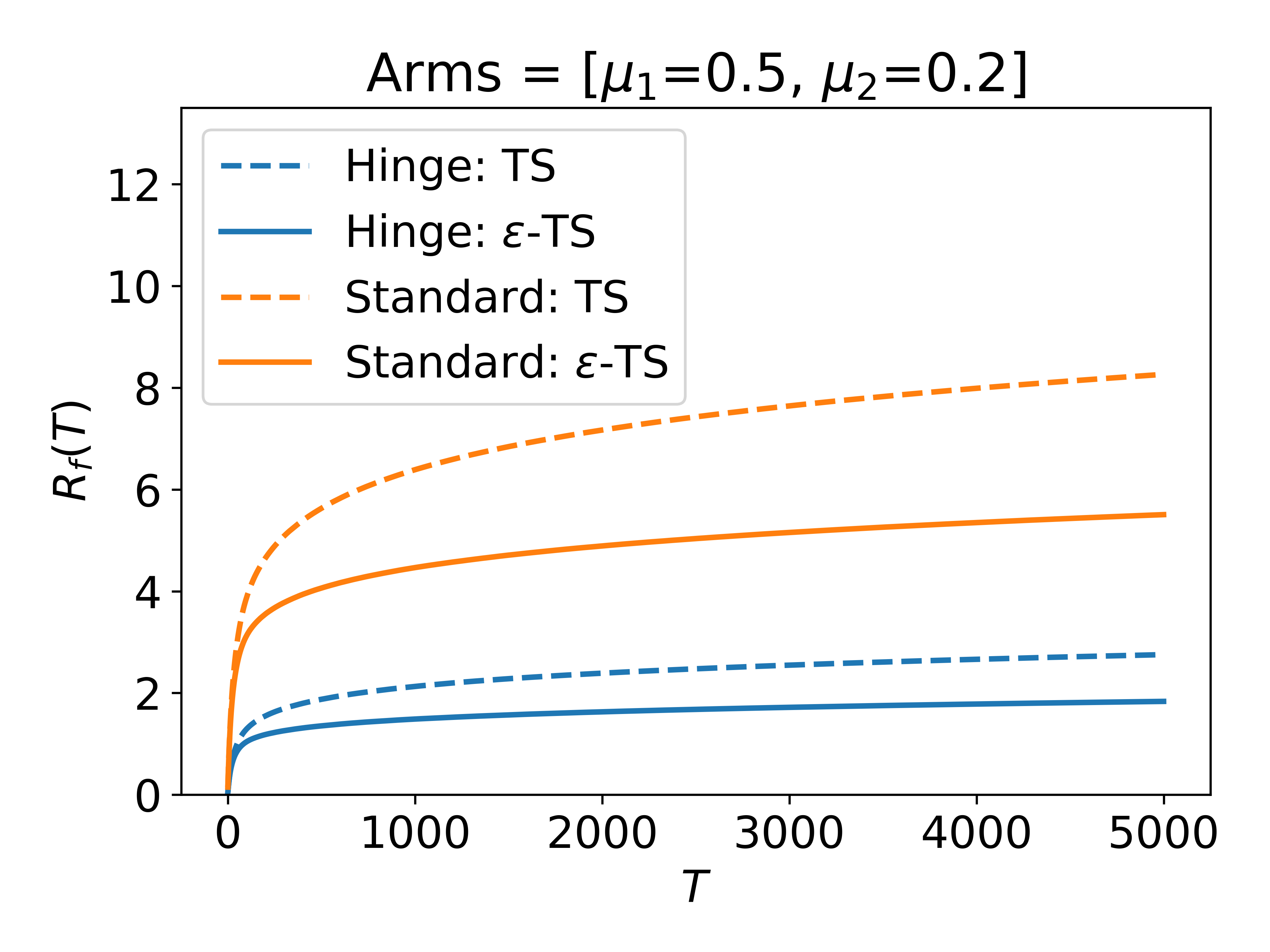}
\label{subfig:low_opt}}
\hspace{0.05\linewidth}
\subfigure{
\includegraphics[trim=0 15 0 15,clip,width=0.37\linewidth]{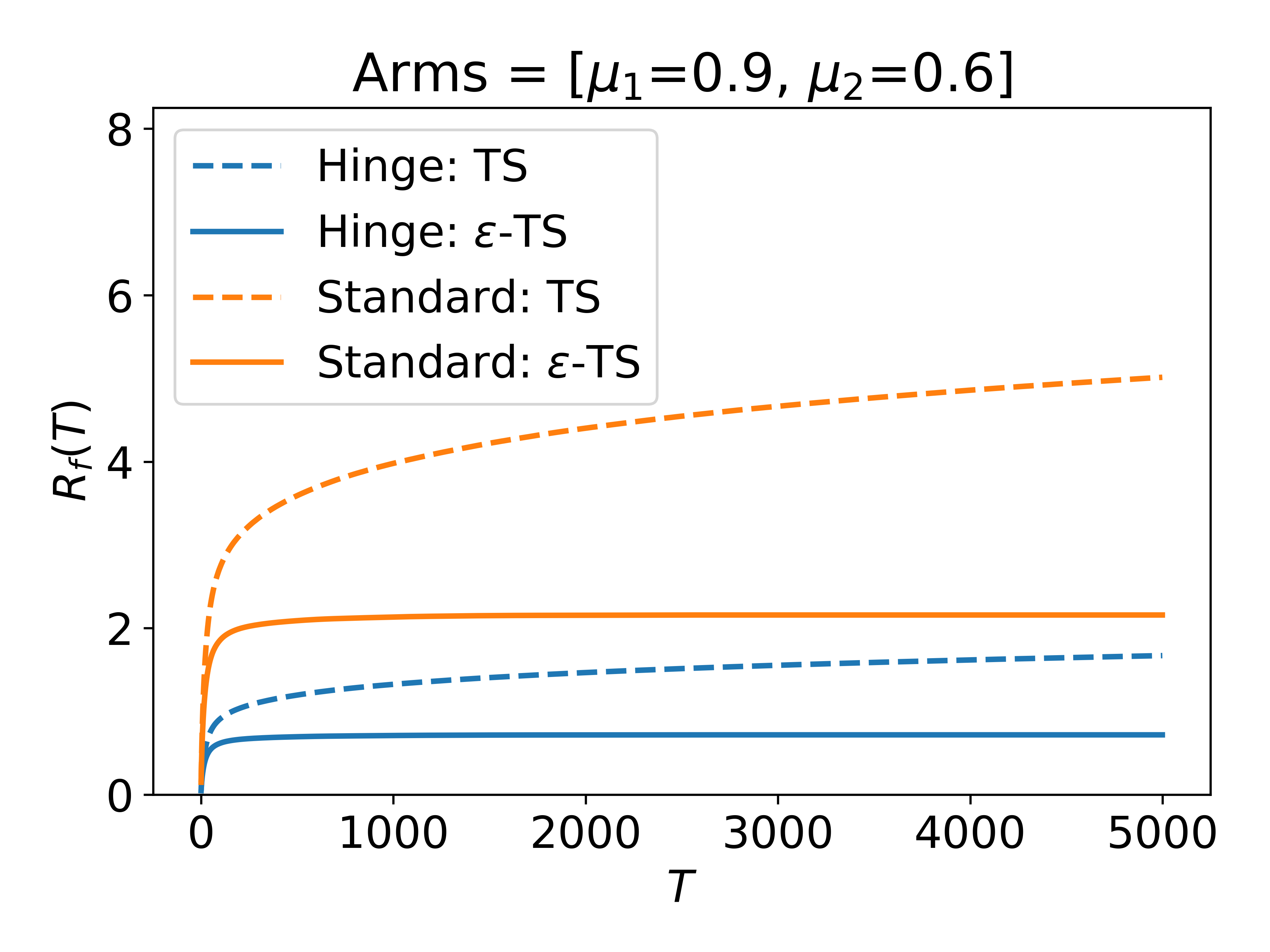} 
\label{subfig:high_opt}} 
\subfigure{
\includegraphics[trim=0 15 0 15,clip,width=0.37\linewidth]{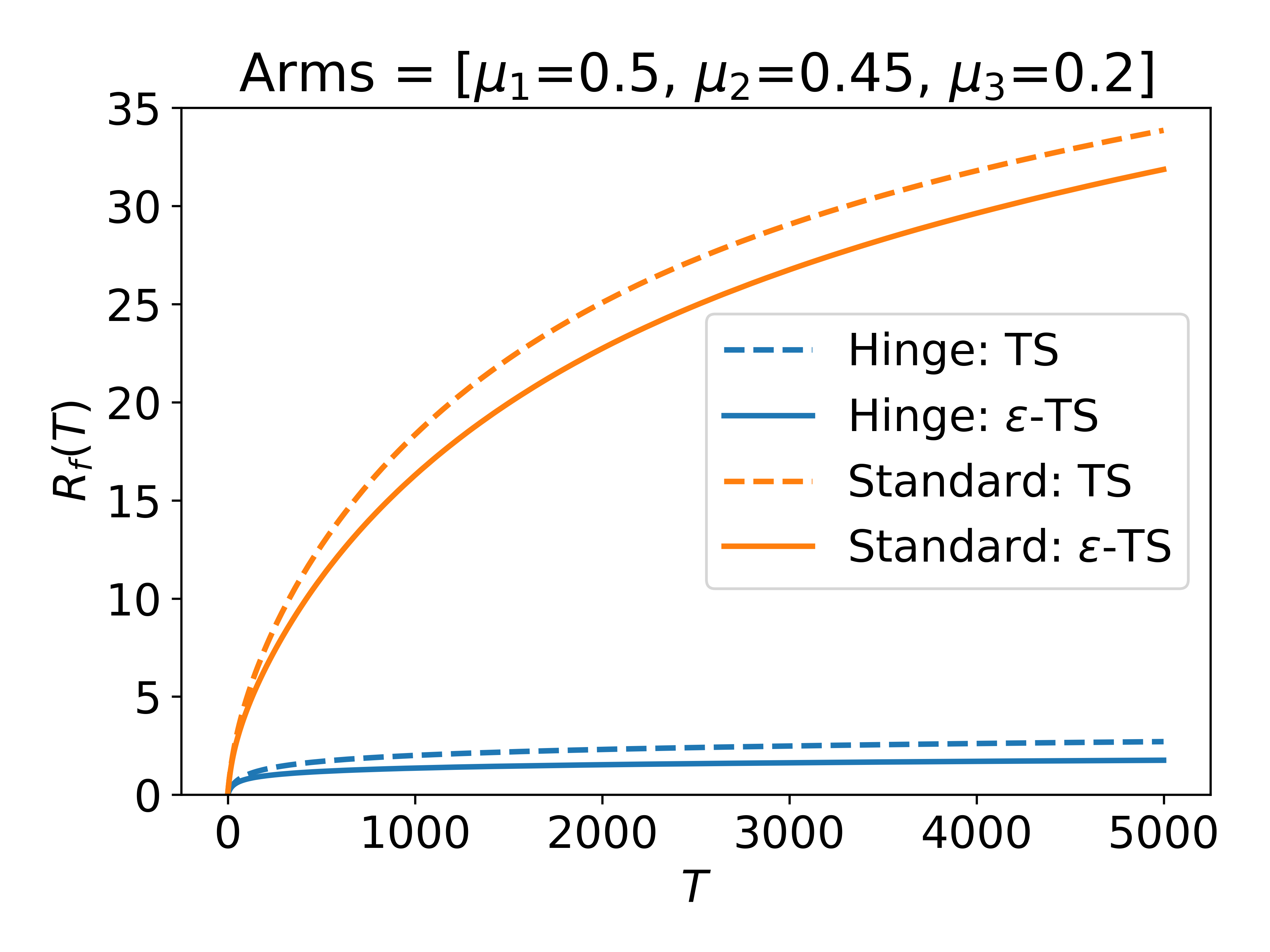}
\label{subfig:low_opt_mid_arm}}
\hspace{0.05\linewidth}
\subfigure{
\includegraphics[trim=0 15 0 15,clip,width=0.37\linewidth]{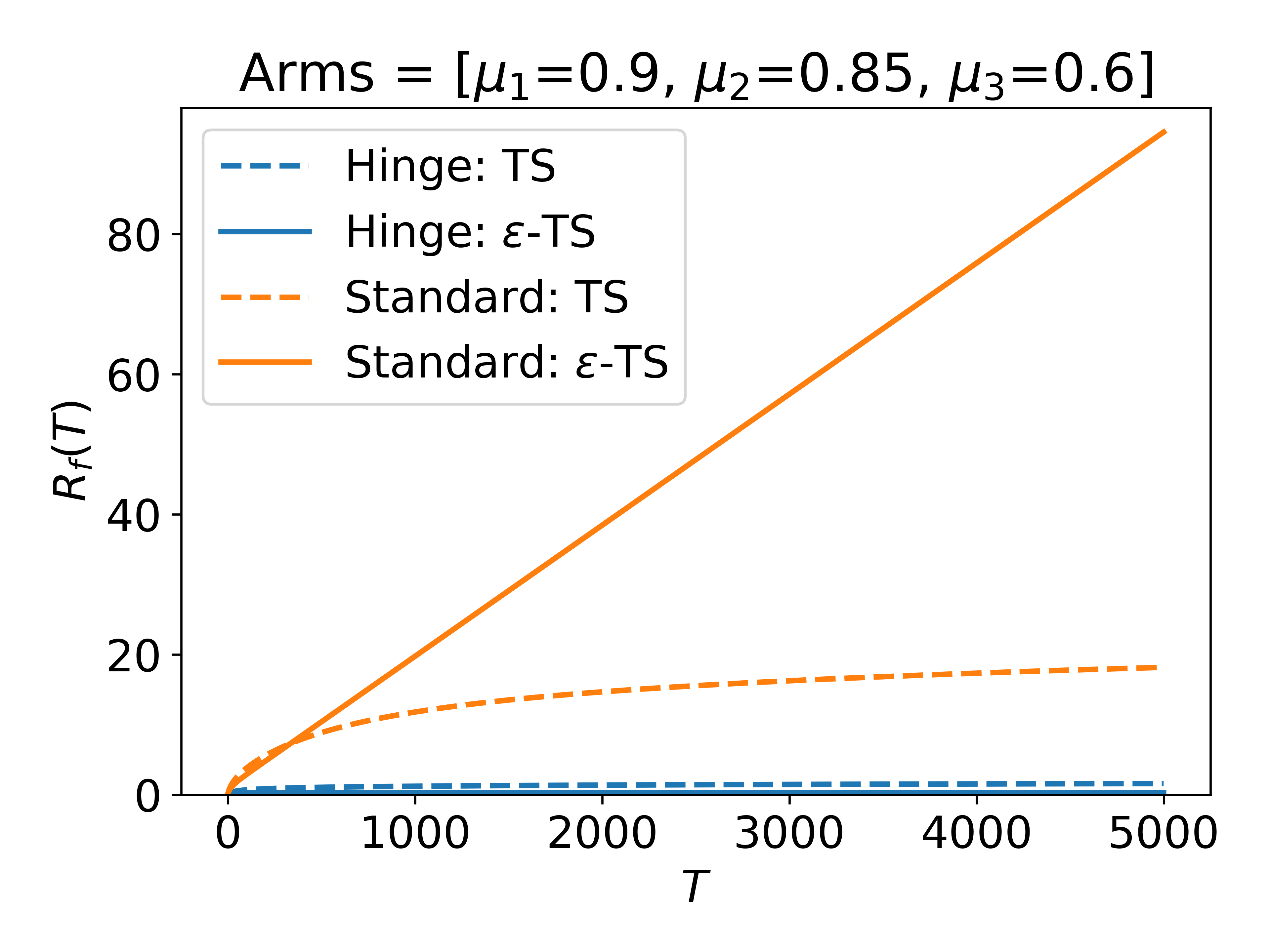} 
\label{subfig:high_opt_mid_arm}}
\caption{Evaluation of $\epsilon$-TS and vanilla TS with $\epsilon=0.2$ and Bernoulli rewards. `Hinge' is the $\epsilon$-gap function $f(\dr{})=\max\brc*{\dr{}-\epsilon,0}$ and `Standard' is the $0$-gap function $f(\dr{})=\dr{}$, which leads to the standard regret. Top row -- the minimal gap is $\dr{2}=0.3>\epsilon$; therefore, $\epsilon$-TS enjoys performance guarantees also for the standard regret. Bottom row -- the minimal gap is $\dr{2}=0.05<\epsilon$; thus, the standard regret $f(\dr{})=\dr{}$ is not an $\epsilon$-gap function, and $\epsilon$-TS has no guarantees for this case.}
\label{figure:experiments}
\end{figure}

\section{Summary and Future Work}
In this work, we introduced the notion of lenient regret w.r.t. $\epsilon$-gap functions. We proved a lower bound for this setting and presented the $\epsilon$-TS algorithm, whose performance matches the lower bound, up to a constant factor. Specifically, we showed that the $\epsilon$-TS greatly improves the performance when a lower bound on the gaps is known. Finally, we performed an empirical evaluation that demonstrates the advantage of our new algorithm when optimizing the lenient regret.

We believe that our work opens up many interesting directions. First, while we suggest a TS algorithm for our settings, it is interesting to devise its UCB counterpart. Moreover, there are alternative ways to define $\epsilon$-gap functions that should be explored, e.g., functions that do not penalize arms with mean larger than $\mu^*\cdot(1-\epsilon)$ (multiplicative leniency). This can also be done by borrowing other approximation concepts from best arm identification. For example, not penalizing arms that exceed some threshold (as in good arm identification \cite{kano2019good}), or not penalizing the choice of any one of the top $m$ of the arms \citep{chaudhuri2017pac}.

We also believe that the concept of lenient regret criteria can be extended to many different settings. It is especially relevant when problems are large, e.g., in combinatorial problems \citep{chen2016combinatorialA}, and can also be extended to reinforcement learning \citep{sutton2018reinforcement}. Notably, and as previously stated, there is some similarity between the $\epsilon$-gap function $f(\dr{})=\Ind{\dr{}>\epsilon}$ and the sample-complexity criterion in RL \citep{kakade2003sample}, and our analysis might allow proving new results for this criterion.

Finally, we explored the notion of lenient regret for stochastic MABs. Another possible direction is adapting the lenient regret to adversarial MABs, and potentially for online learning. In these settings, the convergence rates are typically $\Ocal(\sqrt{T})$, and working with weaker notions of regret might lead to logarithmic convergence rates.

\section*{Acknowledgments}
This work was partially funded by the Israel Science Foundation under ISF grant number 2199/20. Nadav Merlis is partially supported by the Gutwirth Scholarship.

\bibliographystyle{plainnat}
\bibliography{references}

\clearpage

\appendix

\section{Useful Results for the Analysis}
\label{appendix: useful results}
\paragraph {Beta and Binomial distributions:}
\begin{itemize}
    \item \underline{Beta distribution:} For any $\alpha,\beta>0$, we say that $X\sim \Beta(\alpha,\beta)$, if for any $x\in\brs*{0,1}$, its pdf is 
    \[f_X(x) = f^{\Beta}_{\alpha,\beta}(x) = \frac{1}{B(\alpha,\beta)} x^{\alpha-1}(1-x)^{\beta-1}\enspace,\]
    where $B(\alpha,\beta)$ is the Beta function.
    \item \underline{Binomial distribution:} If $n$ is a positive integer and $p\in\brs*{0,1}$, we say that $X\sim Bin(n,p)$, if for any $k\in\brs*{0,\dots,n}$, its pdf is 
    \[f_X(k) = f^{B}_{n,p}(k) = \binom{n}{k} p^k(1-p)^{n-k}\enspace.\]
\end{itemize}
The cdf of both distributions is related through the `Beta-Binomial trick' \citep{agrawal2012analysis}:
\begin{fact}
If $\alpha, \beta$ are positive integers and $x\in\brs*{0,1}$, then
\begin{align*}
    F^{\Beta}_{\alpha,\beta}(x) = 1 - F^B_{\alpha+\beta-1,x}(\alpha-1)\enspace.
\end{align*}
\end{fact}

\paragraph {Useful concentration bounds:} we now present two useful concentration bounds that will be used throughout the paper.
\begin{fact}[Hoeffding's Inequality]
Let $X_1,\dots,X_n\in\brs*{0,1}$ be independent random variables with common expectation $\mu$ and let $t\ge0$. Then,
\begin{align*}
    &\Pr\brc*{\frac{1}{n}\sum_{i=1}^n X_i \ge \mu+t}
    \le e^{-2nt^2} \\
    &\Pr\brc*{\frac{1}{n}\sum_{i=1}^n X_i \le \mu-t}
    \le e^{-2nt^2}
\end{align*}
\end{fact}
\begin{fact}[Chernoff-Hoeffding Bound]
Let $X_1,\dots,X_n\in\brc*{0,1}$ be independent Bernoulli random variables with common expectation $\mu$  and let $t\ge0$. Then,
\begin{align*}
    &\Pr\brc*{\frac{1}{n}\sum_{i=1}^n X_i \ge \mu+t}
    \le e^{-n\cdot\klBin(\mu+t,\mu)}\\
    &\Pr\brc*{\frac{1}{n}\sum_{i=1}^n X_i \le \mu-t}
    \le e^{-n\cdot\klBin(\mu-t,\mu)}
\end{align*}
\end{fact}

\clearpage

\section{Proof of Theorem \ref{theorem:dependent_lower_bound}}
\label{appendix:lower bounds}
In this appendix, we prove the lower bound of \Cref{section:lower bounds}. The proof adapts the techniques of \citep{garivier2019explore} and require a fundamental result in their paper: for any fixed bandit strategy, any sets of arm distributions $\unu,\unu'$ and any $k\in\brs*{\Narms}$ and any $T>0$, it holds that 
\begin{align}
\label{eq:kl count inequality}
    \sum_{a=1}^\Narms \E_\unu\brs*{N_a(T+1)}\kl(\nu_a,\nu_a')\ge \klBin\br*{\frac{\E_\unu\brs*{N_k(T+1)}}{T},\frac{\E_{\unu'}\brs*{N_k(T+1)}}{T}}
\end{align}
The inequality is a direct result of Equation (6) in \citep{garivier2019explore} with $Z=N_k(T+1)/T$.

\dependentLowerBound*
\begin{proof}
We start by proving \Cref{eq:dependent bound count}. To do so, we follow the proof of Theorem 1 of \citep{garivier2019explore}. Denote the arm distribution by $\unu$ and for clarity, denote the lenient regret under arm distribution $\unu$ by $R_{f,\unu}(T)$. Also, let $k$ be some suboptimal arm with $\dr{k}>\epsilon$ and let $\unu'$ be a bandit problem such that $\nu_a'=\nu_a$ for all $a\ne k$ and $\nu_k'\in\D$ is some distribution with $\E\brs*{\nu_k'}>\mu^*+\epsilon$. If such distribution does not exist, then $\K_{\mathrm{inf}}(\nu,\mu^*+\epsilon,\D)=\infty$ and the lower bound trivially holds. Next, by applying \Cref{eq:kl count inequality} and noting that for all $a\ne k$, $\kl(\nu_a,\nu_a')=0$, we get
\begin{align*}
    \E_\unu\brs*{N_k(T+1)}\kl(\nu_k,\nu_k') 
    \ge \klBin\br*{\frac{\E_\unu\brs*{N_k(T+1)}}{T},\frac{\E_{\unu'}\brs*{N_k(T+1)}}{T}}
\end{align*}
Next, see that the for all $p,q\in\brs*{0,1}$,
\begin{align*}
    \klBin(p,q) 
    &= p\ln\frac{p}{q}+(1-p)\ln\frac{1-p}{1-q} \\
    & = (1-p)\ln\frac{1}{1-q} + \underbrace{p\ln p + (1-p)\ln(1-p)}_{\ge-\ln2} + \underbrace{p\ln\frac{1}{q}}_{\ge0} \\
    & \ge (1-p)\ln\frac{1}{1-q}-\ln2
\end{align*}
and combining both inequalities yields
\begin{align}
\label{eq:dependent_ineq1}
    \E_\unu\brs*{N_k(T+1)}\kl(\nu_k,\nu_k') 
    \ge \br*{1-\frac{\E_\unu\brs*{N_k(T+1)}}{T}}\ln\frac{T}{T-\E_{\unu'}\brs*{N_k(T+1)}} - \ln 2 \enspace.
\end{align}
To further bound this term, notice that the lenient regret for bandit problem $\unu'$ can be written as $R_{f,\unu'}(T)=\sum_{a=1}^\Narms f\br*{\dr{a}'}\E_{\unu'}\brs*{N_a(T+1)}$.  
By construction, the optimal arm in $\unu'$ is $k$, with gaps  $\dr{a}'=\E\brs*{\nu_k'}-\E\brs*{\nu_a}>\epsilon$ for all $a\ne k$. Since $f$ is an $\epsilon$-gap function, this also implies that $f(\dr{a}')>0$ for all $a\ne k$ and $f(\dr{k}')=0$. Finally, as the bandit strategy is consistent, it holds that 
\begin{align*}
    T-\E_{\unu'}\brs*{N_k(T+1)} 
    &= \sum_{a\ne k} \E_{\unu'}\brs*{N_a(T+1)} \\
    & \le \frac{1}{\min_{a\ne k} f(\dr{a}')}\sum_{a\ne k} f(\dr{a}')\E_{\unu'}\brs*{N_a(T+1)} \\
    &= \frac{1}{\min_{a\ne k} f(\dr{a}')}R_{f,\unu'}(T) \\
    &= o(T^\alpha)\enspace,
\end{align*}
for all $0<\alpha\le1$. Specifically, for $T$ large enough, $T-\E_{\unu'}\brs*{N_k(T+1)} \le T^\alpha$, and thus
\begin{align}
    \liminf_{T\to\infty}\frac{1}{\ln T} \ln\frac{T}{T-\E_{\unu'}\brs*{N_k(T+1)}} 
    \ge \frac{1}{\ln T} \ln\frac{T}{T^\alpha}=1-\alpha \enspace .\label{eq:lower bound alpha}
\end{align}
Next, notice that for all arms such that $\dr{k}>\epsilon$, we have $f(\dr{k})>0$, and thus
\begin{align*}
    \E_\unu\brs*{N_k(T+1)} 
    = \frac{1}{f(\dr{k})}\br*{f(\dr{k})\E_\unu\brs*{N_k(T+1)}}
    \le \frac{1}{f(\dr{k})} R_{f,\unu}(T) = o(T)\enspace,
\end{align*}
which implies that 
\begin{align*}
    \liminf_{T\to\infty}\br*{1-\frac{\E_\unu\brs*{N_k(T+1)}}{T}} = 1
\end{align*}
Substituting this relation and \Cref{eq:lower bound alpha} into \Cref{eq:dependent_ineq1} yields
\begin{align*}
    \liminf_{T\to\infty}\frac{\E\brs*{N_k(T+1)}}{\ln T} 
    \ge\frac{1}{\kl(\nu_k,\nu_k')} \enspace.
\end{align*}
To conclude the proof of \Cref{eq:dependent bound count}, we take the supremum of the r.h.s. over all distributions $\nu_k'\in\D$ such that $\E\brs*{\nu_k'}>\mu^*+\epsilon$, which leads to $\K_{\mathrm{inf}}(\nu_k,\mu^*+\epsilon,\D)>0$ in the r.h.s.. We also remind that if such distribution does not exist, then $\K_{\mathrm{inf}}(\nu_k,\mu^*+\epsilon,\D)=\infty$ and the bound trivially holds.

For \Cref{eq:dependent regret lower bound}, we write the lenient regret as $R_{f}(T)=\sum_{a=1}^\Narms f\br*{\dr{a}}\E\brs*{N_a(T+1)}$. Substituting \Cref{eq:dependent bound count} into this relation leads to the desired result.
\end{proof}

\clearpage

\section{Proof of Proposition \ref{prop:optimal arm count bound}}
\label{appendix:optCountBound}
In this appendix, we prove \Cref{prop:optimal arm count bound}, including the additional lemmas that we require for the proof:
\optCountBound*
\begin{proof}
We closely follow the proof of Proposition 1 of \citep{kaufmann2012thompson}, with many modifications due to the different posterior distribution. Let $t$ be some fixed time index. If the $j^{th}$ play of the optimal arm happened before $t$, we denote its time by $\tau_j$. Otherwise we say that $\tau_j=t$ and also denote $\tau_0=0$. In addition, let $\xi_j=(\tau_{j+1}-1)-\tau_j$ be the number of time steps between the $j^{th}$ and the $(j+1)^{th}$ play of the optimal arm. In these steps, only suboptimal arms are played, and we have $\sum_{a=2}^{\Narms} N_a(t+1)\ge\sum_{j=0}^{N_1(t+1)}\xi_j$. Using this notation, we can bound the event that $N_1(t+1)\le t^b$ by
\begin{align*}
    \Pr\brc*{N_1(t+1)\le t^b}
    &= \Pr\brc*{\sum_{a=2}^\Narms N_a(t+1) > t-t^b}
    \le \Pr\brc*{\exists j\in\brc*{0,\dots,\floor*{t^b}}: \xi_j \ge t^{1-b}-1} \\
    &\le \sum_{j=0}^{\floor*{t^b}} \Pr\brc*{\xi_j \ge t^{1-b}-1} \enspace.
\end{align*}
Thus, we are interested in bounding the probability of the events $E_j = \brc*{\xi_j \ge t^{1-b}-1}$. Define the interval $\I_j=\brc*{\tau_j+1,\dots,\tau_j+\ceil*{t^{1-b}-1}}$, and notice that under $E_j$, it is included in $\brc*{\tau_j+1,\dots,\tau_{j+1}-1}$. This also implies that under $E_j$, for any $s\in\I_j$ we have that $s\le t$. We further decompose the interval into $\Narms$ smaller intervals, defined as 
\begin{align*}
    \I_{j,l}&=\brc*{\tau_j+ \ceil*{\frac{(l-1)(t^{1-b}-1)}{\Narms}}+1,\dots,\tau_j+ \ceil*{\frac{l(t^{1-b}-1)}{\Narms}}}\enspace, & l=1\dots,\Narms
\end{align*}
When $E_j$ holds, observe that only suboptimal arms are sampled in these intervals. Therefore, in each them, at least one of the suboptimal arms will be sampled $\frac{\abs*{\I_{j,l}}}{\Narms}$ times. Intuitively, for enough samples, the posterior of this arm will be tightly concentrated around its mean, and this arm will be sampled rarely for the rest of $\I_j$. After $\Narms-1$ such intervals, all the suboptimal arms should be highly concentrated around their mean. Then, the probability of not sampling the optimal arm in the last interval should be very low. To formalize this, we employ the notion of saturated and unsaturated arms, similarly to \citep{kaufmann2012thompson}:
\begin{definition}
Let $t$ be fixed. An arm $a\ne1$ is called saturated at time $s\le t$ if $N_a(s)\ge \frac{32\ln t}{(\dr{a}-\epsilon)^2}\triangleq C_a(t)$ and is called unsaturated otherwise. Furthermore, at any time $s\le t$, sampling an unsaturated suboptimal arm is called an interruption.
\end{definition}
Following this definition, we denote by $\G_{j,l}$, the event that by the end of interval $\I_{j,l}$, at least $l$ arms are saturated. We also let $n_{j,l}$ be the number of interruptions during $\I_{j,l}$. Then, we decompose the probability of $E_j$ to
\begin{align}
    \Pr\brc*{E_j} =  \Pr\brc*{E_j\cap \G_{j,\Narms-1}} + \Pr\brc*{E_j\cap \G_{j,\Narms-1}^c}\enspace.\label{eq:optCountBound decomp}
\end{align}
To bound both terms, we require variant of Lemma 3 of \citep{kaufmann2012thompson}. Specifically, the lemma bounds the probability that $\theta_1(s)$ is small throughout long intervals.
\begin{restatable}{lemma-rst}{thetaOptOnLongInterval}
\label{lemma:theta opt on long interval}
Let $\delta=\frac{\dr{2}-\epsilon}{2}>0$ and let $\J\subset\I_j$ be random interval such that given $\F_{\tau_j}$, $\J$ is mutually independent of $\theta_1(s)$ for all $s>\tau_j$. Then, there exists $\lambda_0=\lambda_0(\mu_1,\mu_2,\epsilon)>1$,  such that for every $x>0$, it holds that
\begin{align*}
    \Pr&\brc*{E_j\cap\brc*{\forall s\in\J: \theta_1(s)\le\mu_2+\delta}\cap\brc*{\abs{\J}}\ge x}
    \le (u_{\mu_1,\mu_2,\epsilon})^{x} + C_{\lambda,\mu_1,\mu_2,\epsilon}\frac{1}{x^\lambda}e^{-j d_{\lambda,\mu_1,\mu_2,\epsilon}}\enspace.
\end{align*}
where $u_{\mu_1,\mu_2,\epsilon}\in\br*{0,1}$, $C_{\lambda,\mu_1,\mu_2,\epsilon}>0$ and $d_{\lambda,\mu_1,\mu_2,\epsilon}>0$.
\end{restatable}
The value of the constants can be found at the end of the proof, which is located at \Cref{appendix:theta opt on long interval}. Notice that the constants slightly changed, in comparison to the original lemma, due to the different posterior distribution. Another useful lemma is a variant of Lemma 5 of \citep{agrawal2012analysis} and Lemma 4 of \citep{kaufmann2012thompson}, which states that saturated arms only rarely fall far above their mean:
\begin{restatable}{lemma-rst}{highThetaHighCount}
\label{lemma:high theta high count}
For any $a\ne 1$ and any $0<\delta<1-\mu_a$, if $C(t)=\frac{8\ln t}{\delta^2}$, then
\begin{align*}
    \Pr\brc*{\exists s\le t: \theta_a(s)>\mu_a+\delta,N_a(s)>C(t)} 
    \le \frac{2}{t^2}
\end{align*}
\end{restatable}
See \Cref{appendix:high theta high count} for the proof. Specifically, we choose $\delta_a = \frac{\dr{a}-\epsilon}{2}$ and define $\delta\triangleq \delta_2$. Notice that for all suboptimal arms,  $\mu_a+\delta_a\le \mu_2+\delta_2 = \mu_2+\delta$, and using the union bound, we get 
\begin{align}
\label{eq:saturated arms probability}
    \Pr&\brc*{\exists s\le t,a\ne1: \theta_a(s)>\mu_2+\delta,N_a(s)>C_a(t)} \nonumber\\
    &\qquad\qquad\le \Pr\brc*{\exists s\le t,a\ne1: \theta_a(s)>\mu_a+\delta_a,N_a(s)>C_a(t)}\nonumber\\
    &\qquad\qquad\le \frac{2(\Narms-1)}{t^2}\enspace.
\end{align}
We are now ready to bound both terms of \Cref{eq:optCountBound decomp}.

\paragraph{Bounding the first term of \Cref{eq:optCountBound decomp}:}

In this part of the proof, we aim to bound $\Pr\brc*{E_j\cap \G_{j,\Narms-1}}$. Under $E_j\cap \G_{j,\Narms-1}$, all suboptimal arms are saturated in $\I_{j,\Narms}$. Therefore, we utilize \Cref{eq:saturated arms probability} to get
\begin{align*}
    \Pr\brc*{E_j\cap \G_{j,\Narms-1}}
    &\le \Pr\brc*{\brc*{\exists s\in\I_{j,\Narms}, a\ne 1: \theta_a(s)>\mu_2+\delta} \cap E_j \cap \G_{j,\Narms-1}} \\
    &\quad + \Pr\brc*{\brc*{\forall s\in\I_{j,\Narms}, a\ne 1: \theta_a(s)\le\mu_2+\delta} \cap E_j \cap \G_{j,\Narms-1}} \\
    &\overset{(1)}\le \Pr\brc*{\exists s\in\I_{j,\Narms}, a\ne 1: \theta_a(s)>\mu_2+\delta,N_a(s)>C_a(t)} \\
    &\quad + \Pr\brc*{\brc*{\forall s\in\I_{j,\Narms}, a\ne 1: \theta_a(s)\le\mu_2+\delta} \cap E_j \cap \G_{j,\Narms-1}} \\
    & \overset{(2)} \le \frac{2(\Narms-1)}{t^2} + \Pr\brc*{E_j\cap\brc*{\forall s\in \I_{j,\Narms}: \theta_1(s)\le \mu_2+\delta}}
\end{align*}
For $(1)$, recall that all suboptimal arms are saturated. Thus, they were sampled sampled at least $C_a(t)$ times. In $(2)$, we used \Cref{eq:saturated arms probability} for the first term. For the second term, recall that under $E_j$, $a_s\ne1$ for all $s\in\I_{j,\Narms}$; therefore, we have $\theta_1(s)\le \theta_a(s)\le \mu_2+\delta$ for all $a\ne1$ and $s\in\I_{j,\Narms}$. Next, notice that $\I_{j,\Narms}\in\I_j$ is independent of $\brc*{\theta_1(s)}_{s>\tau_j}$ given $\F_{\tau_j}$. Therefore, we can apply \Cref{lemma:theta opt on long interval} with some $\lambda\in\br*{1,\lambda_0}$ and $x=\abs*{I_{j,\Narms}}=\floor*{\frac{t^{1-b}-1}{\Narms}}$:
\begin{align*}
    \Pr\brc*{E_j\cap\brc*{\forall s\in \I_{j,\Narms}: \theta_1(s)\le \mu_2+\delta}}
    &\le (u_{\mu_1,\mu_2,\epsilon})^{\floor*{\frac{t^{1-b}-1}{\Narms}}} + C_{\lambda,\mu_1,\mu_2,\epsilon}\br*{\floor*{\frac{t^{1-b}-1}{\Narms}}}^{-\lambda}e^{-j d_{\lambda,\mu_1,\mu_2,\epsilon}}\\
    &\triangleq g(\mu_1,\mu_2,\epsilon,b,j,t)
\end{align*}
Hence, we have 
\begin{align*}
    \Pr\brc*{E_j\cap \G_{j,\Narms-1}}
    \le \frac{2(\Narms-1)}{t^2} + g(\mu_1,\mu_2,\epsilon,b,j,t)\enspace,
\end{align*}
and one can easily observe that if $L_0^g(b)=\br*{\Narms+1}^{1/(1-b)}$, it holds that
\begin{align*}
    \sum_{t\ge L_0^g(b)}\sum_{j\le t^b}g(\mu_1,\mu_2,\epsilon,b,j,t)< \infty\enspace.
\end{align*}

\paragraph{Bounding the second term of \Cref{eq:optCountBound decomp}:}

Similarly to \citep{kaufmann2012thompson}, the prove is by induction. Specifically, we show that if $t$ is larger then an absolute constant $L_0^h=L_0^h(\mu_1,\mu_2,\epsilon,b)$, then for all $2\le l \le \Narms$.
\begin{align*}
    \Pr\brc*{E_j\cap \G_{j,l-1}^c} \le (l-2)\br*{\frac{2(\Narms-1)}{t^2}+h(\mu_1,\mu_2,\epsilon,b,j,t)}
\end{align*}
for some function $h$ such that $\sum_{t\ge L_0^h}\sum_{j\le t^b} h(\mu_1,\mu_2,\epsilon,b,j,t) <\infty$. Specifically, we choose $L_0^h$ such that for all $t\ge L_0^h$, it holds that $\floor*{\frac{t^{1-b}-1}{\Narms^2}}\ge \max_{a\ne1}C_a(t)=C_2(t)$. 

\underbar{Base case:} Proving that for all $t\ge L_0^h$, it holds that $\Pr\brc*{E_j\cap \G_{j,1}^c}=0$.

Under $E_j$, recall that only suboptimal arms are sampled in $\I_{j,1}$. As the length of $\I_{j,1}$ is larger than $\floor*{\frac{t^{1-b}-1}{\Narms}}$, at least one suboptimal arm is sampled $\floor*{\frac{t^{1-b}-1}{\Narms^2}}$ times. Specifically, for $t\ge L_0^h$, this arm is sampled at least $C_a(t)$ times, and is therefore saturated. Thus,  for $t\ge L_0^h$, at least one arm is saturated by the end of $\I_{j,1}$, and $\Pr\brc*{E_j\cap \G_{j,1}^c}=0$.

\underbar{Induction step:} Assume that for some $2\le l\le K-1$, if $t\ge L_0^h$, then 
\begin{align*}
    \Pr\brc*{E_j\cap \G_{j,l-1}^c} \le (l-2)\br*{\frac{2(\Narms-1)}{t^2}+h(\mu_1,\mu_2,\epsilon,b,j,t)}\enspace.
\end{align*}
Under this assumption, we decompose $\brc*{E_j\cap \G_{j,l}^c}$ to:
\begin{align}
    \Pr\brc*{E_j\cap \G_{j,l}^c} 
    &\le \Pr\brc*{E_j\cap \G_{j,l-1}^c
    } +\Pr\brc*{E_j\cap \G_{j,l-1} \cap \G_{j,l}^c} \nonumber\\
    &\le(l-2)\br*{\frac{2(\Narms-1)}{t^2}+h(\mu_1,\mu_2,\epsilon,b,j,t)}+\Pr\brc*{E_j\cap \G_{j,l-1} \cap \G_{j,l}^c} \enspace. \label{eq: count induction step}
\end{align}
When the event $\brc*{E_j\cap \G_{j,l-1} \cap \G_{j,l}^c}$ holds, there are exactly $l-1$ saturated arms by the end of $\I_{j,l-1}$ and no additional arm was saturated during $\I_{j,l}$. Thus, during $\I_{j,l}$, unsaturated arms are sampled at most $\max_aC_a(t)=C_2(t)$ times, and the total number of interruptions in $\I_{j,l}$ is bounded by $\Narms C_2(t)$. Specifically, this implies that 
\begin{align*}
    \Pr\brc*{E_j\cap \G_{j,l-1} \cap \G_{j,l}^c}
    \le \Pr\brc*{E_j\cap \G_{j,l-1} \cap \brc*{n_{j,l}\le \Narms C_2(t)}} \enspace.
\end{align*}
Let $\mathcal{S}_l$ be the set of saturated arms at the end of $\I_{j,l}$. We continue bounding the probability by 
\begin{align*}
    \Pr&\brc*{E_j\cap \G_{j,l-1} \cap \brc*{n_{j,l}\le \Narms C_2(t)}} \\
    & \qquad\qquad\quad\le \underbrace{\Pr\brc*{E_j\cap \G_{j,l-1} \cap \brc*{\exists s\in\I_{j,l}, a\in\mathcal{S}_l: \theta_a(s)>\mu_a+\delta_a}}}_{(A)}\\
    & \qquad\qquad\quad\quad + \underbrace{\Pr\brc*{E_j\cap \G_{j,l-1} \cap \brc*{n_{j,l}\le \Narms C_2(t)}\cap \brc*{\forall s\in\I_{j,l}, a\in\mathcal{S}_l: \theta_a(s)\le\mu_a+\delta_a}}}_{(B)}
\end{align*}
Term $(A)$ can be bounded similarly to \Cref{eq:saturated arms probability}, i.e.,
\begin{align*}
    (A) 
    \le\Pr\brc*{\exists s\le t, a\ne1: \theta_a(s)>\mu_a+\delta_a, N_a(s)>C_a(t)}
    \le \frac{2(\Narms-1)}{t^2}\enspace.
\end{align*}
For the second term, let $\J_k$ be the time interval between the $k^{th}$ and $(k+1)^{th}$ interruption in $\I_{j,l}$, for $k\in\brc*{0,\dots,n_{j,l}-1}$, and define $\J_k=\emptyset$ for $k\ge n_{j,l}$. Now, recall that $\abs*{\I_{j,l}} \ge \floor*{\frac{t^{1-b}-1}{\Narms}}$. When the event in $(B)$ holds, there are at most $\Narms C_2(t)$ interruptions during this interval; therefore, there are two interruptions that lie at least $\floor*{\frac{t^{1-b}-1}{\Narms^2 C_2(t)}}$ apart one from another. Furthermore, between these two interruptions, only saturated arms are sampled. Since under $(B)$, all saturated arms have values $\theta_a(s)\le \mu_a+\delta_a$, it implies that for all $a\ne1$, $\theta_a(s)\le \max_a\brc*{\mu_a+\delta_a}=\mu_2+\delta$. Moreover, under $E_j$, it also implies that $\theta_1(s)\le\mu_2+\delta$. Therefore, we can bound $(B)$ by
\begin{align*}
   (B)
   & \le \Pr\brc*{\brc*{\exists k\in\brc*{0,\dots,n_{j,l}-1}: \abs*{\J_k} \ge \floor*{\frac{t^{1-b}-1}{\Narms^2 C_2(t)}}} 
           E_j\cap \G_{j,l-1} \cap \brc*{\forall s\in\I_{j,l}, a\in\mathcal{S}_l: \theta_a(s)\le\mu_a+\delta_a}} \\
   & \quad \le \sum_{k=0}^{\Narms C_2(t)-1} \Pr\brc*{ \brc*{\abs*{\J_k} \ge \floor*{\frac{t^{1-b}-1}{\Narms^2 C_2(t)}}} \cap  \brc*{\forall s\in\J_k, a\in\mathcal{S}_l: \theta_a(s)\le\mu_a+\delta_a} \cap E_j} \\
   & \quad \le \sum_{k=0}^{\Narms C_2(t)-1} \Pr\brc*{ \brc*{\abs*{\J_k} \ge \floor*{\frac{t^{1-b}-1}{\Narms^2 C_2(t)}}} \cap  \brc*{\forall s\in\J_k: \theta_1(s)\le\mu_a+\delta_a}\cap E_j} 
\end{align*}
Finally, we want to apply  \Cref{lemma:theta opt on long interval} on $\J_k\subset\I_j$. However, $\J_k$ is not conditionally independent of $\theta_1(s)$ for $s>\tau_j$. To overcome this issue, we define a modified interval $\J'_k$, that contain samples between the two interruption in a modified problem that runs in parallel to the original algorithm but avoids choosing the optimal arm for all $s>\tau_j$. Importantly, on $E_j$, the optimal arm is not played anyway, and for any interval $J$, $\brc*{\brc*{\J'_k=J}\cap E_j}=\brc*{\brc*{\J_k=J}\cap E_j}$. Thus, we can derive the bound on $\J_k$ by using \Cref{lemma:theta opt on long interval} with $\J'_k$: 
\begin{align*}
   (B)
   & \le \Narms C_2(t)\br*{(u_{\mu_1,\mu_2,\epsilon})^{\floor*{\frac{t^{1-b}-1}{\Narms^2 C_2(t)}}} + C_{\lambda,\mu_1,\mu_2,\epsilon}\floor*{\frac{t^{1-b}-1}{\Narms^2 C_2(t)}}^{-\lambda}e^{-j d_{\lambda,\mu_1,\mu_2,\epsilon}}}\\
   & \triangleq h(\mu_1,\mu_2,\epsilon,b,j,t)\enspace,
\end{align*}
and notice that similarly to $g(\mu_1,\mu_2,\epsilon,b,j,t)$, for any $b<1-\frac{1}{\lambda}$ we have 
\begin{align*}
    \sum_{t\ge L_0^h}\sum_{j\le t^b}h(\mu_1,\mu_2,\epsilon,b,j,t)< \infty\enspace.
\end{align*}
To conclude the induction step, combining $(A)$ and $(B)$ yields
\begin{align*}
    \Pr\brc*{E_j\cap \G_{j,l-1} \cap \G_{j,l}^c} 
    & \le \Pr\brc*{E_j\cap \G_{j,l-1} \cap \brc*{n_{j,l}\le \Narms C_2(t)}} 
    \le \frac{2(\Narms-1)}{t^2} + h(\mu_1,\mu_2,\epsilon,b,j,t)
\end{align*}
and substituting back into \Cref{eq: count induction step} leads to the desired result.

\paragraph{Combining the bounds back into \Cref{eq:optCountBound decomp}:}

Combining both parts, we get that for any $b<1-\frac{1}{\lambda}$ and any $t\ge L_0 = L_0(\mu_1,\mu_2,\epsilon,b)\triangleq\max\brc*{L_0^g(b), L_0^h(\mu_1,\mu_2,\epsilon,b)}$,
\begin{align*}
    \Pr\brc*{E_j} 
    &=  \Pr\brc*{E_j\cap \G_{j,\Narms-1}} + \Pr\brc*{E_j\cap \G_{j,\Narms-1}^c} \\
    & \le \frac{2(\Narms-1)}{t^2} + g(\mu_1,\mu_2,\epsilon,b,j,t)+ (\Narms-2)\br*{\frac{2(\Narms-1)}{t^2}+h(\mu_1,\mu_2,\epsilon,b,j,t)}\enspace,
\end{align*}
and summing over all possible values of $t$ and $j$ leads to the desired result:
\begin{align*}
    \sum_{t=1}^T\Pr\brc*{N_1(t)\le (t-1)^b} 
    &\le L_0(\mu_1,\mu_2,\epsilon,b)+\sum_{t\ge L_0}\Pr\brc*{N_1(t+1)\le t^b}  \\
    &\le  L_0(\mu_1,\mu_2,\epsilon,b)+\sum_{t\ge L_0}\sum_{j=0}^{\floor*{t^b}} \Pr\brc*{E_j} \\
    &\le L_0(\mu_1,\mu_2,\epsilon,b) + 2(\Narms-1)^2\sum_{t\ge1} \frac{1}{t^{2-b}} \\
    & \quad + \sum_{t\ge L_0}\sum_{j\le t^b} \br*{\Narms h(\mu_1,\mu_2,\epsilon,b,j,t) + g(\mu_1,\mu_2,\epsilon,b,j,t)} \\
    & \triangleq C_b(\mu_1,\mu_2,\epsilon)
    \enspace,
\end{align*}
for some constant $C_b(\mu_1,\mu_2,\epsilon)<\infty$.
\end{proof}

\clearpage

\subsection{Proof of Lemma \ref{lemma:theta opt on long interval}}
\label{appendix:theta opt on long interval}
\thetaOptOnLongInterval*
\begin{proof}
First notice that under $E_j$, there are no new samples of the optimal arm in $\I_j$ and thus also in $\J\subset \I_j$. As a result, the posterior distribution of the optimal arm is fixed according to the statistics at time $\tau_j$. Also, recall the assumption that $\dr{a}>\epsilon$ for all $a\ne1$. Then, the event $\theta_1(s)\le \mu_2+\delta<1-\epsilon$ necessarily implies that $\hmu{1}{\tau_j}\le 1-\epsilon$. Thus, conditioned on $\F_{\tau_j}$, the samples of $\theta_1(s)$ are an i.i.d. sequence with a scaled Beta distribution. We define $\theta_1(s)=Y_s$, where $Y_s$ are i.i.d with distribution $\frac{Y_s}{1-\epsilon}~\sim \Beta(\alpha_1(\tau_j),\beta_1(\tau_j))$ given $\F_{\tau_j}$. Since given $\F_{\tau_j}$, the interval $\J$ is independent of $\theta_1(s)$ for all $s\in\I_j$, we have
\begin{align*}
    \Pr\brc*{Y_s\le \mu_2+\delta\vert s\in \J,\F_{\tau_j}} 
    &= F^{\Beta}_{\alpha(\tau_j), j+2 - \alpha(\tau_j)}\br*{\frac{\mu_2+\delta}{1-\epsilon}} 
    = 1-F^{B}_{j+1, \frac{\mu_2+\delta}{1-\epsilon}}\br*{\alpha(\tau_j)-1} \\
    & = 1-F^{B}_{j+1, \frac{\mu_2+\delta}{1-\epsilon}}\br*{\floor*{\frac{S_1(\tau_j)}{1-\epsilon}}}\enspace, 
\end{align*}
where the second equality is due to the Beta-Binomial trick and the third equality is a direct substitution of $\alpha(\tau_j)$. Using the conditional independence of the samples, we get
\begin{align*}
    \Pr\brc*{\forall s\in \J:Y_s\le \mu_2+\delta\vert  \F_{\tau_j},\J} 
    &= \br*{1-F^{B}_{j+1, \frac{\mu_2+\delta}{1-\epsilon}}\br*{\floor*{\frac{S_1(\tau_j)}{1-\epsilon}}}}^{\abs*{J}} \enspace,
\end{align*}
and thus 
\begin{align*}
    \Pr&\brc*{E_j\cap\brc*{\forall s\in\J: \theta_1(s)\le\mu_2+\delta}\cap\brc*{\abs{\J}}\ge x}\\
    &\qquad\qquad=\E\brs*{ \Ind{\brc*{\abs{\J}}\ge x}\cdot \E\brs*{\Ind{E_j} \Ind{\forall s\in\J: \theta_1(s)\le\mu_2+\delta}\vert \F_{\tau},\J }} \\
    &\qquad\qquad=\E\brs*{ \Ind{\brc*{\abs{\J}}\ge x}\cdot \E\brs*{\Ind{E_j} \Ind{\forall s\in\J: Y_s\le\mu_2+\delta}\vert \F_{\tau},\J }} \\
    &\qquad\qquad\le\E\brs*{ \Ind{\brc*{\abs{\J}}\ge x}\cdot \E\brs*{\Ind{\forall s\in\J: Y_s\le\mu_2+\delta}\vert \F_{\tau},\J }} \\
    &\qquad\qquad=\E\brs*{ \Ind{\brc*{\abs{\J}}\ge x}\cdot \br*{1-F^{B}_{j+1, \frac{\mu_2+\delta}{1-\epsilon}}\br*{\floor*{\frac{S_1(\tau_j)}{1-\epsilon}}}}^{\abs*{J}} } \\
    &\qquad\qquad\le\E\brs*{ \Ind{\brc*{\abs{\J}}\ge x}\cdot \br*{1-F^{B}_{j+1, \frac{\mu_2+\delta}{1-\epsilon}}\br*{\floor*{\frac{S_1(\tau_j)}{1-\epsilon}}}}^{x} } \\
    &\qquad\qquad\le\E\brs*{\br*{1-F^{B}_{j+1, \frac{\mu_2+\delta}{1-\epsilon}}\br*{\floor*{\frac{S_1(\tau_j)}{1-\epsilon}}}}^{x} } 
\end{align*}
Directly calculating the expectation leads to:
\begin{align*}
    \E\brs*{\br*{1-F^{B}_{j+1, \frac{\mu_2+\delta}{1-\epsilon}}\br*{\floor*{\frac{S_1(\tau_j)}{1-\epsilon}}}}^{x}}
    & = \sum_{n=0}^j\br*{1-F^B_{j+1,\frac{\mu_2+\delta}{1-\epsilon}}\br*{\floor*{\frac{n}{1-\epsilon}}}}^{x}f^B_{j,\mu_1}(n) \\
    & \le \sum_{n=0}^j\br*{1-F^B_{j+1,\frac{\mu_2+\delta}{1-\epsilon}}\br*{n}}^{x}f^B_{j,\mu_1}(n) \\ 
    & \triangleq \sum_{n=0}^j\br*{1-F^B_{j+1,y}\br*{n}}^{x}f^B_{j,\mu_1}(n)
\end{align*}
where in the first inequality we used the fact that $F^B_{n,p}(k)$ is increasing in $k$ and in the last equality we defined $y\triangleq \frac{\mu_2+\delta}{1-\epsilon}$. Notice that $\mu_2+\delta<\mu_1-\epsilon$, and therefore $y< \frac{\mu_1-\epsilon}{1-\epsilon}\le\mu_1$.

From here, observe that we got the same expression as in the proof of Lemma 3 in \citep{kaufmann2012thompson} (or, alternatively, Lemma 15 of \citep{trinh2019solving}). Therefore, we get the same bound as them i.e.,
\begin{align*}
    \Pr\brc*{E_j\cap\brc*{\forall s\in\J: \theta_1(s)\le\mu_2+\delta}\cap\brc*{\abs{\J}}\ge x}
    &\le \sum_{n=0}^j\br*{1-F^B_{j+1,y}\br*{n}}^{x}f^B_{j,\mu_1}(n) \\
    &\le (u_{\mu_1,\mu_2,\epsilon})^{x} + C_{\lambda,\mu_1,\mu_2,\epsilon}\frac{1}{x^\lambda}e^{-j d_{\lambda,\mu_1,\mu_2,\epsilon}}\enspace.
\end{align*}

We now explicitly state all of the constants in the bounds.
\begin{itemize}
    \item Recall that $y=y(\mu_1,\mu_2,\epsilon)=\frac{\mu_2 + \delta}{1-\epsilon} = \frac{\mu_2 + \frac{\dr{2}-\epsilon}{2}}{1-\epsilon} = \frac{\mu_1+\mu_2-\epsilon}{2(1-\epsilon)}<\mu_1$
    \item $u_{\mu_1,\mu_2,\epsilon} = \br*{\frac{1}{2}}^{1-y}$.
    \item $\lambda_0(\mu_1,\mu_2,\epsilon)$ can be calculated by 
    \begin{align*}
        \lambda_0(\mu_1,\mu_2,\epsilon) 
        = 1 + \frac{\klBin(y,\mu_1)}{y\ln\frac{1}{y} + (1-y)\ln\frac{1}{1-y}}>1
    \end{align*}
    \item $d_{\lambda,\mu_1,\mu_2,\epsilon}$ equals to
    \begin{align*}
        d_{\lambda,\mu_1,\mu_2,\epsilon}
        = \lambda\br*{y\ln y + (1-y)\ln(1-y)} - \br*{y\ln \mu_1 + (1-y)\ln(1-\mu_1)}>0\enspace,
    \end{align*}
    where the inequality holds for any $\lambda\in \br*{1,\lambda_0}$.
    \item Define $R_\lambda = \frac{\mu_1(1-y)^\lambda}{y^\lambda(1-\mu_1)}>1$, where the inequality holds for all $\lambda\in \br*{1,\lambda_0}$. Then,
    \begin{align*}
        C_{\lambda,\mu_1,\mu_2,\epsilon} = \frac{(\lambda/e)^\lambda}{(1-y)^\lambda}\frac{R_\lambda}{R_\lambda-1}>0\enspace.
    \end{align*}
\end{itemize}
\end{proof}

\clearpage

\subsection{Proof of Lemma \ref{lemma:high theta high count}}
\label{appendix:high theta high count}
\highThetaHighCount*
\begin{proof}
The proof resembles Lemma 5 of \citep{agrawal2012analysis}, with some modifications due to the different posterior distribution.  Similarly to their proof, we decompose the l.h.s. to
\begin{align}
    \Pr\brc*{\exists s\le t: \theta_a(s)>\mu_a+\delta,N_a(s)>C(t)}
    &= \Pr\brc*{\exists s\le t: \hmu{a}{s}>\mu_a+\frac{\delta}{2},N_a(s)>C(t)}  \nonumber\\
    &\quad+ \Pr\brc*{\exists s\le t: \hmu{a}{s}\le\mu_a+\frac{\delta}{2},\theta_a(s)>\mu_a+\delta,N_a(s)>C(t)} \label{eq:highThetaHighCount decomp}
\end{align}
We will now show that both terms can be bounded by $1/t^2$, which will conclude that proof. The first term of \Cref{eq:highThetaHighCount decomp} can be bounded by 
\begin{align}
\label{eq:highThetaHighCount decomp t1}
    \Pr\brc*{\exists s\le t: \hmu{a}{s}>\mu_a+\frac{\delta}{2},N_a(s)>C(t)} 
    \le \sum_{s=1}^t \Pr\brc*{\hmu{a}{s}>\mu_a+\frac{\delta}{2},N_a(s)>C(t)} 
\end{align}
Next, we bound each of the individual terms:
\begin{align*}
    \Pr\brc*{\hmu{a}{s}>\mu_a+\frac{\delta}{2},N_a(s)>C(t)}
    &\le \Pr\brc*{\exists n\in\brc*{\ceil*{C(t)},\dots,s}: \hmu{a}{s}>\mu_a+\frac{\delta}{2},N_a(s)=n} \\
    & = \Pr\brc*{\exists n\in\brc*{\ceil*{C(t)},\dots,s}: \hat\mu_{a,n}>\mu_a+\frac{\delta}{2},N_a(s)=n} \\
    & \le \Pr\brc*{\exists n\in\brc*{\ceil*{C(t)},\dots,s}: \hat\mu_{a,n}>\mu_a+\frac{\delta}{2}} \\
    & \le \sum_{n=\ceil*{C(t)}}^s \Pr\brc*{\hat\mu_{a,n}>\mu_a+\frac{\delta}{2}} \\
    & \overset{(1)}\le \sum_{n=\ceil*{C(t)}}^s \exp\brc*{-2n\br*{\frac{\delta}{2}}^2}\\ 
    & \le s \exp\brc*{-2C(t)\br*{\frac{\delta}{2}}^2}\\
    & \overset{(2)}= \frac{1}{t^3}
\end{align*}
where $(1)$ uses Hoeffding's inequality and $(2)$ is a direct substitution of $C(t)$ and $s\le t$. Substituting back into \Cref{eq:highThetaHighCount decomp t1} yields
\begin{align*}
    \Pr\brc*{\exists s\le t: \hmu{a}{s}>\mu_a+\frac{\delta}{2},N_a(s)>C(t)} 
    \le \sum_{s=1}^t \frac{1}{t^3} = \frac{1}{t^2}\enspace.
\end{align*}
To bound the second term of \Cref{eq:highThetaHighCount decomp}, first observe that if $\mu_a+\delta>1-\epsilon$, then the event $\theta_a(s)>\mu_a+\delta$ can only occur if $\hmu{a}{s}>1-\epsilon$, and then $\hmu{a}{s}=\theta_a(s)>\mu_a+\delta$. However, the event also requires that $\hmu{a}{s}\le\mu_a+\frac{\delta}{2}$, and the event cannot hold for any $\delta\ge0$:
\begin{align*}
     \Pr&\brc*{\exists s\le t: \hmu{a}{s}\le\mu_a+\frac{\delta}{2},\theta_a(s)>\mu_a+\delta,N_a(s)>C(t)} =0\enspace.
\end{align*}
Otherwise, $\hmu{a}{s}\le\mu_a+\frac{\delta}{2}<1-\epsilon$ and $\theta_a(s)$ is has a scaled Beta-distribution. Using this fact, we continue similarly to \citep{agrawal2012analysis}; for any $s\le t$, we bound
\begin{align*}
    \Pr&\brc*{\hmu{a}{s}\le\mu_a+\frac{\delta}{2},\theta_a(s)>\mu_a+\delta,N_a(s)>C(t)}\\
    &\qquad\le \Pr\brc*{\exists n\in\brc*{\ceil*{C(t)},\dots,s}: \hmu{a}{s}\le\mu_a+\frac{\delta}{2},\theta_a(s)>\mu_a+\delta,N_a(s)=n} \\
    &\qquad \le \Pr\brc*{\exists n\in\brc*{\ceil*{C(t)},\dots,s}:\theta_a(s)>\hmu{a}{s}+\frac{\delta}{2},N_a(s)=n} \\
    &\qquad\le \sum_{n=\ceil*{C(t)}}^s \Pr\brc*{\theta_a(s)>\hmu{a}{s}+\frac{\delta}{2},N_a(s)=n} \\
    &\qquad \overset{(1)}=\sum_{n=\ceil*{C(t)}}^s \E\brs*{\Ind{N_a(s)=n}\E\brs*{\Ind{\theta_a(s)>\hmu{a}{s}+\frac{\delta}{2}}\Big\vert\F_{s-1}}} \\
    &\qquad= \sum_{n=\ceil*{C(t)}}^s \E\brs*{\Ind{N_a(s)=n}\br*{1-F^{\Beta}_{\alpha_a(s),\beta_a(s)}\br*{\frac{\hmu{a}{s}+\delta/2}{1-\epsilon}}}}\\
    &\qquad \overset{(2)}= \sum_{n=\ceil*{C(t)}}^s\E\brs*{\Ind{N_a(s)=n}F^{B}_{\alpha_a(s)+\beta_a(s)+1,\frac{\hmu{a}{s}+\delta/2}{1-\epsilon}}\br*{\alpha_a(s)-1}}\\
    &\qquad \overset{(3)}=\sum_{n=\ceil*{C(t)}}^s\E\brs*{\Ind{N_a(s)=n}F^{B}_{n+1,\frac{\hmu{a}{s}+\delta/2}{1-\epsilon}}\br*{\floor*{\frac{n\hmu{a}{s}}{1-\epsilon}}}} \\
    & \qquad\overset{(4)}=\sum_{n=\ceil*{C(t)}}^s\E\brs*{\exp\brc*{-2(n+1)\br*{\frac{\hmu{a}{s}+\delta/2}{1-\epsilon} - \frac{1}{n+1}\floor*{\frac{n\hmu{a}{s}}{1-\epsilon}}^2}}} \\
    & \qquad\le s \exp\brc*{-2C(t)\br*{\frac{\delta/2}{1-\epsilon}^2}} \\
    &\qquad \overset{(5)}\le \frac{1}{t^3}
\end{align*}
For $(1)$, we used the tower property while recalling that $N_a(s)$ is $\F_{s-1}$-measurable and $(2)$ is due to the Beta-Binomial trick. In $(3)$, we substituted $\alpha_a(s)$ and $\beta_a(s)$, using $N_a(s)=n$, and in $(4)$, we used Hoeffding's inequality and threw the indicator over $N_a(s)$. Finally, $(5)$ is a direct substitution of $C(t)$ and $s\le t$. Using the union bound over different values of $s\in\brs*{t}$ leads to a bound of $\frac{1}{t^2}$ for the second term of \Cref{eq:highThetaHighCount decomp} and concludes the proof.
\end{proof}

\clearpage

\section{Proofs of Upper Bounds}
\subsection{Adding suboptimal arms with small gaps}
\label{appendix:multiple optimal}
In this section, we prove that adding arms with gaps $\dr{a}\le\epsilon$ only decreases the expected regret. 

Denote by $R_f^\pi(T;\Narms)$ the lenient regret w.r.t. $f$ when running a bandit strategy $a_t=\pi_t$ in a $\Narms$-armed bandit problem. For any $\Narms$-armed bandit problem, consider a modified $(\Narms+1)$-armed problem where arm $\Narms+1$ has a gap $\dr{\Narms+1}\le\epsilon$. We aim to prove that adding the $(\Narms+1)^{th}$ arm reduces the lenient regret of $\epsilon$-TS, i.e., $R_f^\pi(T;\Narms+1)\le R_f^\pi(T;\Narms)$.

We define a sequence of policies $\psi^s$ for $s\in\brc*{0,\dots,T}$ as follows:
\begin{align*}
    \psi^s_t = 
    \begin{cases}
    \arg\max_{a\in\brs*{\Narms+1}} \theta_a(t) & t\le s \\
    \arg\max_{a\in\brs*{\Narms}} \theta_a(t) & t>s
    \end{cases}
\end{align*}
Notice that $\psi^T$ runs $\epsilon$-TS on all arms $a\in\brs*{\Narms+1}$ while $\psi^0$ runs the same algorithm on the first $\Narms$ arms. We will now prove that for all $s\in\brc*{0,\dots,T-1}$, we have $R_f^{\psi^{s+1}}(T;\Narms+1)\le R_f^{\psi^s}(T;\Narms+1)$. This also implies that $R_f^{\psi^T}(T;\Narms+1) \le R_f^{\psi^{0}}(T;\Narms+1)$, which concludes the proof. 

Denote the (random) action when playing policy $\pi$ at time $t$ by $\pi_t$. We use a coupling argument and assume that both $\psi^s$ and $\psi^{s+1}$ are played in parallel on the same data, with the same internal randomness. Specifically, since the policies are identical up to time $s$, this implies that $\psi^{s+1}_t=\psi^s_t$ for all $t\le s$ and both policies play through the same history $\F_s$ w.p. 1.

To prove the inequality, we decompose the regret of $\psi^{s+1}$ as follows:
\begin{align*}
    R_f^{\psi^{s+1}}(T;\Narms+1)
    &= \E\brs*{\sum_{t=1}^T f(\dr{\psi^{s+1}_t})} \\
    & = \E\brs*{\sum_{t=1}^s f(\dr{\psi^{s+1}_t})} + \E\brs*{\sum_{t=s+1}^T f(\dr{\psi^{s+1}_t})} \\
    & = \E\brs*{\sum_{t=1}^s f(\dr{\psi^s_t})} + \E\brs*{\sum_{t=s+1}^T f(\dr{\psi^{s+1}_t})}
\end{align*}
where the last equality is since $\psi^s$ and $\psi^{s+1}$ are the same up to time $s$. Next, notice that given any history $\F_s$, if $\psi^{s+1}_{s+1}\ne \Narms+1$, then it will choose an action $\psi^{s+1}_{s+1}=\arg\max_{a\in\brs*{\Narms}}\theta_a(t)$, Just as $\psi^s$ would choose. On the following steps, both policies are also the same, and therefore, their expected regret is equal:
\begin{align*}
    \E\brs*{\Ind{\psi^{s+1}_{s+1}\ne \Narms+1}\sum_{t=s+1}^T f(\dr{\psi^{s+1}_t})\bigg\vert \F_s}
    = \E\brs*{\Ind{\psi^{s+1}_{s+1}\ne \Narms+1}\sum_{t=s+1}^T f(\dr{\psi^s_t})\bigg\vert \F_s}
\end{align*}
If $\psi^{s+1}_{{k+1}}= \Narms+1$, then $\psi^{s+1}$ will not suffer any regret at this step and will continue identically to $\psi^s$ for the remaining steps, according to the history until time $s$. This is equivalent to playing $\psi^s$ for one less step:
\begin{align*}
    \E\brs*{\Ind{\psi^{s+1}_{s+1}= \Narms+1}\sum_{t=s+1}^T f(\dr{\psi^{s+1}_t})\bigg\vert \F_s}
    = \E\brs*{\Ind{\psi^{s+1}_{s+1}= \Narms+1}\sum_{t=s+1}^{T-1} f(\dr{\psi^s_t})\bigg\vert \F_s}
\end{align*}
Combining both leads to the desired result:
\begin{align*}
    R_f^{\psi^{s+1}}(T;\Narms+1)
    &= \E\brs*{\sum_{t=1}^s f(\dr{\psi^s_t})} 
    + \E\brs*{\E\brs*{\Ind{\psi^{s+1}_{s+1}\ne \Narms+1}\sum_{t=s+1}^T f(\dr{\psi^s_t})\bigg\vert \F_s}} \\
    &\hspace{4.325cm}\; + \E\brs*{\E\brs*{\Ind{\psi^{s+1}_{s+1}= \Narms+1}\sum_{t=s+1}^{T-1} f(\dr{\psi^s_t})\bigg\vert \F_s}} \\
    & = \E\brs*{\sum_{t=1}^s f(\dr{\psi^s_t})} + \E\brs*{\sum_{t=s+1}^T f(\dr{\psi^s_t})}
    - \E\brs*{\E\brs*{\Ind{\psi^{s+1}_{s+1}= \Narms+1} f(\dr{\psi^s_T})\bigg\vert \F_s}}\\
    &= R_f^{\psi^s}(T;\Narms+1) - \underbrace{\E\brs*{\E\brs*{\Ind{\psi^{s+1}_{s+1}= \Narms+1} f(\dr{\psi^s_T})\bigg\vert \F_s}}}_{\ge0} \\
    & \le R_f^{\psi^s}(T;\Narms+1)\enspace.
\end{align*}

\clearpage

\subsection{Proof of Lemma \ref{lemma:mean large deviation bound}}
\label{appendix:mean large deviation bound}
\meanLargeDev*
\begin{proof}
To bound the l.h.s, we divide the sum into different values of $N_a(t)$ as follows:
\begin{align*}
    \sum_{t=1}^T  \Pr\brc*{a_t=a,\hmu{a}{t}> x}
    & = \E\brs*{\sum_{t=1}^T\sum_{s=0}^t \Ind{a_t=a,\hmu{a}{t}> x,N_a(t)=s}} \\
    & = \E\brs*{\sum_{t=1}^T\sum_{s=0}^t \Ind{a_t=a,\hat\mu_{a,s}> x,N_a(t)=s}} \\
    & = \E\brs*{\sum_{s=0}^T\Ind{\hat\mu_{a,s}> x}\underbrace{\sum_{t=\max\brc*{s,1}}^T\Ind{a_t=a,N_a(t)=s}}_{\le1} } \\
    & \overset{(1)}\le \E\brs*{\sum_{s=0}^T\Ind{\hat\mu_{a,s}> x}}\\
    & \overset{(2)}\le \sum_{s=1}^T \Pr\brc*{\hat\mu_{a,s}> x} \\
    & \overset{(3)}\le \sum_{s=1}^T \exp\brc*{-s\klBin(x,\mu_a)} \\
    & \le \frac{1}{\klBin(x,\mu_a)}
\end{align*}
In $(1)$ we used the fact that if $a_t=a$ and $N_a(t)=s$, then $N_a(\tau)\ge s+1$ for all $\tau>t$; therefore, only one of the indicators $\Ind{a_t=a,N_a(t)=s}$ can be equal to one. In $(2)$, recall the initialization $\hat\mu_{a,0}=0$, which implies that $\brc*{\hat\mu_{a,0}>x}$ cannot occur for any $x>\mu_a\ge0$, and we can remove $s=0$ from the summation. For $(3)$ we used Chernoff-Hoeffding bound for $x>\mu_a$.
\end{proof}

\clearpage
\subsection{Proof of Lemma \ref{lemma:optimal low when high optimal arm}}
\label{appendix:optimal low when high optimal arm}
\optLowHighOptimal*
\begin{proof}
For ease of notations, let $L_1=L_1(\mu_1,\epsilon,b)$
We start by dividing the sum to 
\begin{align*}
   \sum_{t=1}^T  \Pr\brc*{\theta_1(t)\le 1-\epsilon, N_1(t)>(t-1)^b}
    \le L_1 + \sum_{t=L_1+1}^T  \Pr\brc*{\theta_1(t)\le 1-\epsilon, N_1(t)>(t-1)^b}\enspace.
\end{align*}
Notice that $\theta_1(t)\le1-\epsilon$ if and only if $\hmu{1}{t}\le 1-\epsilon$. Thus, the remaining term can be bounded by
\begin{align*}
    \sum_{t=L_1+1}^T  \Pr\brc*{\theta_1(t)\le 1-\epsilon, N_1(t)>(t-1)^b}
    & = \sum_{t=L_1+1}^T  \Pr\brc*{\hmu{1}{t}\le 1-\epsilon, N_1(t)>(t-1)^b} \\
    & \le \sum_{t=L_1+1}^T\sum_{s=\ceil*{(t-1)^b}}^t  \Pr\brc*{\hmu{1}{t}\le 1-\epsilon, N_1(t)=s} \\
    & = \sum_{t=L_1+1}^T\sum_{s=\ceil*{(t-1)^b}}^t  \Pr\brc*{\hat\mu_{1,s}\le 1-\epsilon, N_1(t)=s} \\
    & \le \sum_{t=L_1+1}^T\sum_{s=\ceil*{(t-1)^b}}^t  \Pr\brc*{\hat\mu_{1,s}\le 1-\epsilon} \\
    & \overset{(1)}\le \sum_{t=L_1+1}^T\sum_{s=\ceil*{(t-1)^b}}^t \exp\brc*{-s\klBin(1-\epsilon,\mu_1)} \\
    & \le \sum_{t=L_1+1}^T \frac{ \exp\brc*{-\br*{(t-1)^b-1}\cdot\klBin(1-\epsilon,\mu_1)}}{\klBin(1-\epsilon,\mu_1)} \\
    & \overset{(2)}\le \sum_{t=L_1+1}^T \frac{1}{t^2\klBin(1-\epsilon,\mu_1)} \\
    & \le \frac{\pi^2/6}{\klBin(1-\epsilon,\mu_1)}
\end{align*}
where $(1)$ uses Chernoff-Hoeffding bound and $(2)$ uses the definition of $L_1(\mu_1,\epsilon,b)$.
\end{proof}

\clearpage

\subsection{Proof of Lemma \ref{lemma:bound when theta high and optimal low}}
\label{appendix:bound when theta high and optimal low}
\boundHighThetaLowOpt*
\begin{proof}

First notice that if $a_t=a$ and $\theta_1(t)> \eta(t)$, then necessarily $\theta_a(t)> \eta(t)$, and we can write
\begin{align*}
    \sum_{t=1}^T  \Pr\brc*{a_t=a,\theta_1(t)> \eta(t)} \le \sum_{t=1}^T  \Pr\brc*{a_t=a,\theta_a(t)> \eta(t)}\enspace.
\end{align*}
Next, let $x_a(t)$ be some sequence such that $\mu_a< x_a(t)<\eta(t)$ for all $t\in\brs*{T}$. The exact value of $x_a(t)$ will be determined by the end of the proof. We further decompose this term to 
\begin{align}
    \sum_{t=1}^T  \Pr\brc*{a_t=a,\theta_1(t)> \eta(t)} 
    &\le \sum_{t=1}^T  \Pr\brc*{a_t=a,\hmu{a}{t}> x_a(t)} \nonumber\\
    &\quad+ \sum_{t=1}^T  \Pr\brc*{a_t=a,\theta_a(t)> \eta(t),\hmu{a}{t}\le x_a(t)} \enspace. \label{eq:highThetaLowOpt decomp}
\end{align}
The first term can be bounded using \Cref{lemma:mean large deviation bound}:
\begin{align}
    \sum_{t=1}^T  \Pr\brc*{a_t=a,\hmu{a}{t}> x_a(t)}
    &\le \sum_{t=1}^T  \Pr\brc*{a_t=a,\hmu{a}{t}> \min_tx_a(t)} 
    \le \frac{1}{\klBin(\min_tx_a(t),\mu_a)}\enspace. \label{eq:highThetaLowOpt p1}
\end{align}
Next, define 
$B_a(T)=(1+c)\max_{\tau\in\brs*{T}}\brc*{\frac{\ln \tau}{\klBin\br*{\frac{x_a(\tau)}{1-\epsilon},\frac{\eta(\tau)}{1-\epsilon}}}}$. 
For the remaining term of \Cref{eq:highThetaLowOpt decomp}, we first write the probabilities as the expectation of indicators. Then, we divide the sum to times where $N_a(t)> B_a(T)$ and times where $N_a(t)\le B_a(T)$:
\begin{align}
    \sum_{t=1}^T  &\Ind{a_t=a,\theta_a(t)> \eta(t),\hmu{a}{t}\le x_a(t)}\nonumber\\
    &\qquad = \sum_{t=1}^T \Ind{a_t=a,N_a(t)>B_a(T), \hmu{a}{t}\le x_a(t),\theta_a(t)>\eta(t)} \nonumber\\
    &\qquad\quad+ \sum_{t=1}^T \Ind{a_t=a,N_a(t)\le B_a(T), \hmu{a}{t}\le x_a(t),\theta_a(t)>\eta(t)} \nonumber\\
    & \qquad\le  \sum_{t=1}^T \Ind{a_t=a,N_a(t)>B_a(T), \hmu{a}{t}\le x_a(t),\theta_a(t)>\eta(t)} \nonumber\\
    &\qquad\quad+ \sum_{t=1}^T \Ind{a_t=a,N_a(t)\le B_a(T)} \nonumber\\
    &\qquad \le B_a(T)+1 + \sum_{t=1}^T \Ind{a_t=a,N_a(t)>B_a(T), \hmu{a}{t}\le x_a(t),\theta_a(t)>\eta(t)}\label{eq:highThetaLowOpt p2}
\end{align}
where in the last inequality we used the fact that if $a_t=a$, then $N_a(t+1)=N_a(t)+1$; therefore, the events in the second summation can occur at most $B_a(T)+1$ times. For the indicators in the remaining summation, we take an expectation and use the tower rule with $\F_{t-1}$:
\begin{align}
    &\E\brs*{\Ind{a_t=a,N_a(t)>B_a(T), \hmu{a}{t}\le x_a(t),\theta_a(t)>\eta(t)}} \nonumber\\
    &\qquad\qquad = \E\brs*{\Ind{N_a(t)>B_a(T), \hmu{a}{t}\le x_a(t)}\E\brs*{\Ind{a_t=a,\theta_a(t)>\eta(t)}\vert \F_{t-1}}} \nonumber\\
    &\qquad \qquad \;\le \E\brs*{\Ind{N_a(t)>B_a(T), \hmu{a}{t}\le x_a(t)}\E\brs*{\Ind{\theta_a(t)>\eta(t)}\vert \F_{t-1}}} \nonumber\\
    & \qquad\qquad \;= \E\brs*{\Ind{N_a(t)>B_a(T), \hmu{a}{t}\le x_a(t)}\Pr\brc*{\theta_a(t)>\eta(t)\vert \F_{t-1}}} \label{eq:highThetaLowOpt p3}
\end{align}
Notice that $x_a(t)<\eta(t)<1-\epsilon$; therefore, the condition $\hmu{a}{t}\le x_a(t)$ implies that given $\F_{t-1}$, $\frac{\theta_a(t)}{1-\epsilon}\sim \Beta(\alpha_a(t),\beta_a(t))$. Using Beta-Binomial trick, we get
\begin{align*}
    \Pr\brc*{\theta_a(t)>\eta(t)\vert \F_{t-1}}
    = 1-F^{\Beta}_{\alpha_a(t),\beta_a(t)}\br*{\frac{\eta(t)}{1-\epsilon}} = F^B_{N_a(t)+1,\eta(t)/(1-\epsilon)}\br*{\alpha_a(t)-1}\enspace.
\end{align*}
Next, notice that we are only interested in history sequences $\F_{t-1}$ such that $\hmu{a}{t}\le x_a(t)$. Then,
\begin{align}
    \frac{\alpha_a(t)-1}{N_a(t)+1}
    = \frac{1}{N_a(t)+1}\floor*{\frac{\hmu{a}{t}N_a(t)}{1-\epsilon}} 
    \le \frac{\hmu{a}{t}}{1-\epsilon}
    \le \frac{x_a(t)}{1-\epsilon}\enspace,\label{eq:alpha beta calc}
\end{align}
and since $\frac{x_a(t)}{1-\epsilon}< \frac{\eta(t)}{1-\epsilon}$, we can apply Chernoff-Hoeffding inequality, which yields
\begin{align*}
     \Pr\brc*{\theta_a(t)>\eta(t)\vert \F_{t-1}}
    &\le \exp\brc*{-(N_a(t)+1)\klBin\br*{\frac{\alpha_a(t)-1}{N_a(t)+1},\frac{\eta(t)}{1-\epsilon}}} \\
    &\le \exp\brc*{-(N_a(t)+1)\klBin\br*{\frac{x_a(t)}{1-\epsilon},\frac{\eta(t)}{1-\epsilon}}} 
\end{align*}
where the last inequality is due to \Cref{eq:alpha beta calc} and the fact that $\klBin(x,y)$ decreases in $x$ for $x\le y$. Substituting back into \Cref{eq:highThetaLowOpt p3}, we get
\begin{align*}
    &\E\brs*{\Ind{a_t=a,N_a(t)>B_a(T), \hmu{a}{t}\le x_a(t),\theta_a(t)>\eta(t)}} \\
     & \qquad\qquad \;\le E\brs*{\Ind{N_a(t)>B_a(T), \hmu{a}{t}\le x_a(t)}\exp\brc*{-(N_a(t)+1)\klBin\br*{\frac{x_a(t)}{1-\epsilon},\frac{\eta(t)}{1-\epsilon}}}} \\
     &\qquad \qquad \; \le E\brs*{\Ind{N_a(t)>B_a(T), \hmu{a}{t}\le x_a(t)}\exp\brc*{-(B_a(T)+1)\klBin\br*{\frac{x_a(t)}{1-\epsilon},\frac{\eta(t)}{1-\epsilon}}}} \\
     &\qquad\qquad \; \le \exp\brc*{-(B_a(T)+1)\klBin\br*{\frac{x_a(t)}{1-\epsilon},\frac{\eta(t)}{1-\epsilon}}}\\
     &\qquad\qquad \; \le\frac{1}{t^{1+c}}\enspace,
\end{align*}
where for the last inequality, we used
\begin{align*}
    B_a(T) 
    = (1+c)\max_{\tau\in\brs*{T}}\brc*{\frac{\ln \tau}{\klBin\br*{\frac{x_a(\tau)}{1-\epsilon},\frac{\eta(\tau)}{1-\epsilon}}}} 
    \ge (1+c)\frac{\ln t}{\klBin\br*{\frac{x_a(t)}{1-\epsilon},\frac{\eta(t)}{1-\epsilon}}}\enspace.
\end{align*}
Substituting back into \Cref{eq:highThetaLowOpt p2} and recalling that $\sum_{t=1}^T \frac{1}{t^{1+c}}\le 1 + \frac{1}{c}$ leads to
\begin{align}
    \E\brs*{\sum_{t=1}^T \Ind{a_t=a,\theta_a(t)> \eta(t),\hmu{a}{t}\le x_a(t)}}
    \le B_a(T)+2 + \frac{1}{c}\enspace.\label{eq:highThetaLowOpt p4}
\end{align}
To conclude the proof, recall that $\eta(t)\ge\mu_1-\epsilon>\mu_a$ for all $a\ne1$ and $t\in\brs*{T}$. Therefore, we can fix $x_a(t)$ to be the solution of the equation
\begin{align*}
    \klBin\br*{\frac{x_a(t)}{1-\epsilon},\frac{\eta(t)}{1-\epsilon}} = \frac{1}{1+c}\klBin\br*{\frac{\mu_a}{1-\epsilon},\frac{\eta(t)}{1-\epsilon}}
\end{align*}
in the interval $\br*{\mu_a,\eta(t)}$. Specifically, using the bound $\eta(t)\ge\mu_1-\epsilon$ combined with \Cref{lemma:kl equation solution increases} (which can be found in \Cref{appendix:kl equation solution increases}), we get $x_a(t)\ge x_{a,c}$, for $x_{a,c}\in\br*{\mu_a,\mu_1-\epsilon}$ such that 
\begin{align*}
    \klBin\br*{\frac{x_{a,c}}{1-\epsilon},\frac{\mu_1-\epsilon}{1-\epsilon}} = \frac{1}{1+c}\klBin\br*{\frac{\mu_a}{1-\epsilon},\frac{\mu_1-\epsilon}{1-\epsilon}}\enspace.
\end{align*}

Then, substituting \Cref{eq:highThetaLowOpt p1} and \Cref{eq:highThetaLowOpt p4} back into \Cref{eq:highThetaHighCount decomp} leads to the desired result:
\begin{align*}
    \sum_{t=1}^T  \Pr\brc*{a_t=a,\theta_1(t)> \eta(t)} 
    &\le (1+c)\max_{\tau\in\brs*{T}}\brc*{\frac{\ln \tau}{\klBin\br*{\frac{x_a(\tau)}{1-\epsilon},\frac{\eta(\tau)}{1-\epsilon}}}} + 2 + \frac{1}{c} + \frac{1}{\klBin(\min_tx_a(t),\mu_a)} \\
    & \le (1+c)^2\max_{\tau\in\brs*{T}}\brc*{\frac{\ln \tau}{\klBin\br*{\frac{\mu_a}{1-\epsilon},\frac{\eta(\tau)}{1-\epsilon}}}} + 2 + \frac{1}{c} + \frac{1}{\klBin(x_{a,c},\mu_a)}
\end{align*}
\end{proof}

\clearpage

\subsection{Proof of Lemma \ref{lemma:prob that theta low when optimal low}}
\label{appendix:prob that theta low when optimal low}
\ProbThetaLow*
\begin{proof}
For brevity, we write $L_2=L_2(b,\epsilon)$. Then, for all $t>L_2$, we have $\eta(t)=\mu_1-2\sqrt{\frac{6\ln t}{(t-1)^b}}$, and thus, 
\begin{align*}
    \sum_{t=1}^T & \Pr\brc*{\theta_1(t)\le \eta(t), N_1(t)>(t-1)^b} \\
    & \quad\le L_2-1 + \sum_{t=L_2}^T  \Pr\brc*{\theta_1(t)\le \mu_1-2\sqrt{\frac{6\ln t}{(t-1)^b}}, N_1(t)>(t-1)^b} \\
    & \quad\le L_2-1 + \sum_{t=L_2}^T  \Pr\brc*{\theta_1(t)\le \mu_1-2\sqrt{\frac{6\ln t}{N_1(t)}}, N_1(t)>(t-1)^b} \\
    & \quad\le L_2-1 + \underbrace{\sum_{t=L_2}^T  \Pr\brc*{\theta_1(t)\le \mu_1-2\sqrt{\frac{6\ln t}{N_1(t)}},\frac{S_1(t)}{N_1(t)+1}>\mu_1-\sqrt{\frac{6\ln t}{N_1(t)}}, N_1(t)>(t-1)^b}}_{(A)}\\
    &\qquad+ \underbrace{\sum_{t=L_2}^T  \Pr\brc*{\frac{S_1(t)}{N_1(t)+1}\le\mu_1-\sqrt{\frac{6\ln t}{N_1(t)}}, N_1(t)>(t-1)^b}}_{(B)}
\end{align*}
\paragraph{Bounding term $(A)$:} Denote $y(t)=y\br*{t,S_1(t),N_1(t)}=\frac{S_1(t)}{N_1(t)+1}-\sqrt{\frac{6\ln t}{N_1(t)}}$. Combining both inequalities of the event, we get
\begin{align*}
(A) 
&\le \sum_{t=L_2}^T  \Pr\brc*{N_1(t)>(t-1)^b,\theta_1(t)\le \frac{S_1(t)}{N_1(t)+1}-\sqrt{\frac{6\ln t}{N_1(t)}}}\\ 
&= \sum_{t=L_2}^T  \Pr\brc*{N_1(t)>(t-1)^b,\theta_1(t)\le y(t)} \\
&\le \underbrace{\sum_{t=L_2}^T   \Pr\brc*{N_1(t)>(t-1)^b,\theta_1(t)\le y(t),y(t)<1-\epsilon}}_{(i)}
+\underbrace{\sum_{t=L_2}^T  \Pr\brc*{y(t)\ge1-\epsilon}}_{(ii)}
\end{align*}
For the first term, see that the inequality $\theta_1(t)\le y(t)<1-\epsilon$ implies that $\theta_1(t)$ has a scaled beta-distribution. Using the tower rule while noticing that $y(t)$ is $\F_{t-1}$-measurable, we write
\begin{align*}
(i)&= \sum_{t=L_2}^T \Pr\brc*{N_1(t)>(t-1)^b,\theta_1(t)\le y(t),y(t)<1-\epsilon} \\
 & = \sum_{t=L_2}^T \E\brs*{\Ind{N_1(t)>(t-1)^b,y(t)<1-\epsilon} \Pr\brc*{\theta_1(t)\le y(t)\vert \F_{t-1}}}\\
& = \sum_{t=L_2}^T \E\brs*{\Ind{N_1(t)>(t-1)^b,y(t)<1-\epsilon}
\br*{1-F^B_{N_1(t)+1,\frac{y(t)}{1-\epsilon}}\br*{\frac{1}{N_1(t)+1}\floor*{\frac{S_1(t)}{1-\epsilon}}}} }
\end{align*}
where the last equality is due to the Beta-Binomial trick, and we define $F^B_{n,p}(y)=1$ when $p\le0$ and $y\ge0$. Next, one can easily observe that for any $t\ge2$  and $N_1(t)\ge1$, we have $\sqrt{\frac{3\ln t}{2N_1(t)}}\ge \frac{1}{N_1(t)+1}$, and thus
\begin{align}
    \frac{1}{N_1(t)+1}\floor*{\frac{S_1(t)}{1-\epsilon}}
    &\ge \frac{1}{N_1(t)+1}\br*{\frac{S_1(t)}{1-\epsilon}-1} \nonumber\\
    &\ge \frac{1}{N_1(t)+1}\frac{S_1(t)}{1-\epsilon}-\sqrt{\frac{3\ln t}{2N_1(t)}}\nonumber\\
    &\ge \frac{y(t)}{1-\epsilon}+\sqrt{\frac{6\ln t}{N_1(t)}}-\sqrt{\frac{3\ln t}{2N_1(t)}} \nonumber\\
    & = \frac{y(t)}{1-\epsilon}+\sqrt{\frac{3\ln t}{2N_1(t)}}\label{eq:mean to y bound}
\end{align}
Then, we can use Hoeffding's inequality:
\begin{align*}
(i)&\le\E\brs*{\sum_{t=L_2}^T \Ind{N_1(t)>(t-1)^b} \exp\brc*{-(N_1(t)+1)\br*{\frac{1}{N_1(t)+1}\floor*{\frac{S_1(t)}{1-\epsilon}} -\frac{y(t)}{1-\epsilon} }^2} }\\
& \overset{(1)}\le  \E\brs*{\sum_{t=L_2}^T \Ind{N_1(t)>(t-1)^b} \exp\brc*{-(N_1(t)+1)\br*{\sqrt{\frac{3\ln t}{2N_1(t)}}}^2} } \\
& \le \E\brs*{\sum_{t=1}^T \Ind{N_1(t)>(t-1)^b} \exp\brc*{-\frac{3\ln t}{2}} } \\
& \le \sum_{t=1}^T  t^{-3/2}
\le 3
\end{align*}
where in $(1)$, we substituted \Cref{eq:mean to y bound}. Next, we bound the event where $y(t)\ge1-\epsilon$ as follows:

\begin{align*}
    (ii)&=\sum_{t=1}^T  \Pr\brc*{y(t)\ge1-\epsilon}\\
    &= \sum_{t=1}^T  \Pr\brc*{\frac{S_1(t)}{N_1(t)+1}-\sqrt{\frac{6\ln t}{N_1(t)}}\ge1-\epsilon} \\
    &\le \sum_{t=1}^T  \Pr\brc*{\hmu{1}{t}-\sqrt{\frac{6\ln t}{N_1(t)}}\ge1-\epsilon} \\
    & \overset{(1)}\le \sum_{t=1}^T\sum_{s=1}^t \Pr\brc*{N_1(t)=s,\hmu{1}{t}-\sqrt{\frac{6\ln t}{N_1(t)}}\ge1-\epsilon} \\
    & \le \sum_{t=1}^T\sum_{s=1}^t \Pr\brc*{\hat\mu_{1,s}-\sqrt{\frac{6\ln t}{s}}\ge1-\epsilon} \\
    & \overset{(2)}\le \sum_{t=1}^T\sum_{s=1}^t \Pr\brc*{\hat\mu_{1,s}\ge \mu_1 +\sqrt{\frac{6\ln t}{s}}} \\
    & \overset{(3)}\le \sum_{t=1}^T\sum_{s=1}^t \exp \brc*{-2s\br*{\sqrt{\frac{6\ln t}{s}}}^2} \\
    & = \sum_{t=1}^T\sum_{s=1}^t t^{-12} 
    \le 2
\end{align*}
In $(1)$, notice that when $N_1(t)=0$, the l.h.s. of the inequality is $-\infty$, and the inequality cannot hold, so we can assume that $N_1(t)\ge1$. Next, $(2)$ uses the assumption $\mu_1\le 1-\epsilon$ and $(3)$ is by Hoeffding's inequality. Combining both parts, we get 
\begin{align*}
    (A)\le 5\enspace.
\end{align*}

\paragraph{Bounding term $(B)$:}
\begin{align*}
    (B)  
    & \le \sum_{t=L_2}^T \sum_{s=\ceil*{(t-1)^b}}^t  \Pr\brc*{\frac{S_1(t)}{N_1(t)+1}\le\mu_1-\sqrt{\frac{6\ln t}{N_1(t)}}, N_1(t)=s} \\
    &  = \sum_{t=L_2}^T \sum_{s=\ceil*{(t-1)^b}}^t  \Pr\brc*{S_1(t)\le(s+1)\br*{\mu_1-\sqrt{\frac{6\ln t}{s}}}, N_1(t)=s} \\
    &  \le \sum_{t=L_2}^T \sum_{s=\ceil*{(t-1)^b}}^t  \Pr\brc*{S_1(t)\le s\mu_1+1-\sqrt{6s\ln t}} 
\end{align*}
Next, notice that for all $t\ge L_2\ge2$, we have $\sqrt{1.5\ln t}>1$. This also implies that for all $t\ge L_2$ and all $s>(t-1)^b$, we have $\sqrt{1.5s\ln t}>1$, which can be equivalently written as $1-\sqrt{6s\ln t}<-\sqrt{1.5s\ln t}$. Substituting this relation yields:
\begin{align*}
    (B) \le
    & \sum_{t=L_2(b,\epsilon)}^T \sum_{s=\ceil*{(t-1)^b}}^t  \Pr\brc*{S_1(t)\le s\mu_1-\sqrt{1.5s\ln t}} 
    \overset{(*)}\le \sum_{t=1}^T \sum_{s=1}^t  t^{-3} 
    \le 2
\end{align*}
where $(*)$ is due to Hoeffding's inequality.

\paragraph{Combining both parts:} For any $L_2(b,\epsilon)\ge2$ such that for all $t>L_2(b,\epsilon)$, $\eta(t)>\mu_1-\epsilon$, we have
\begin{align*}
    \sum_{t=L_2}^T  \Pr\brc*{\theta_1(t)\le \mu_1-2\sqrt{\frac{6\ln t}{(t-1)^b}}, N_1(t)>(t-1)^b}
    &\le  L_2(b,\epsilon)-1+(A)+(B) \\
    &\le L_2(b,\epsilon)+6
\end{align*}
\end{proof}

\clearpage

\subsection{Proof of Theorem \ref{theorem:dependent upper bound asymptotic}}
\label{appendix:dependent upper bound asymptotic}

\dependentUpperBound*

\begin{proof}
We start from the bound of \Cref{theorem:dependent upper bound} and prove that it asymptotically leads to the desired result. Specifically, we use the notations and assumptions described at the beginning of \Cref{section: regret analysis}.

First notice that the result trivially holds when $\mu_1>1-\epsilon$. Thus, we focus on the regime $\mu_1\le1-\epsilon$. Let $\tau_m(T)$ be the largest time index such that the maximum in \Cref{eq:upper bound low opt} is achieved, i.e.,
\begin{align*}
    \tau_m(T) = \max\brc*{\tau: \tau\in\arg\max_{t\in\brs*{T}}\frac{\ln t}{\klBin\br*{\frac{\mu_a}{1-\epsilon},\frac{\eta(t)}{1-\epsilon}}}}\enspace.
\end{align*}
By definition, $\tau_m(T)$ increases with $T$, and therefore $\lim_{T\to\infty}\tau_m(T)$ is well-defined. If $\lim_{T\to\infty}\tau_m(T)<\infty$, one can easily observe that 
\begin{align*}
    \limsup_{T\to\infty}\frac{R_f(T)}{\ln T}=0\enspace. 
\end{align*}
Otherwise, for any $c>0$ we have
\begin{align}
    \limsup_{T\to\infty}\frac{R_f(T)}{\ln T}
    &\le (1+c)^2\sum_{a=2}^\Narms f(\dr{a})\limsup_{T\to\infty}\frac{1}{\ln T} \frac{\ln \tau_m(T)}{\klBin\br*{\frac{\mu_a}{1-\epsilon},\frac{\eta(\tau_m(T))}{1-\epsilon}}} \nonumber\\
    &\le (1+c)^2\sum_{a=2}^\Narms f(\dr{a})\lim_{T\to\infty}\frac{1}{\klBin\br*{\frac{\mu_a}{1-\epsilon},\frac{\eta(\tau_m(T))}{1-\epsilon}}}\nonumber\\
    &= (1+c)^2\sum_{a=2}^\Narms\frac{f(\dr{a})}{\klBin\br*{\frac{\mu_a}{1-\epsilon},\frac{\mu_1}{1-\epsilon}}}\label{eq:asymptotic bound modKL}
    \enspace. 
\end{align}
and as the result holds with any $c>0$, it also holds for $c=0$, which leads to the first bound of the theorem. Notice that if $\mu_1=1-\epsilon$, then the denominator equals $\klBin\br*{\frac{\mu_a}{1-\epsilon},\frac{\mu_1}{1-\epsilon}}=\infty$ and the asymptotic bound holds. Otherwise, $\mu_1<1-\epsilon$. In this case, we relate the denominator to $\klBin(\mu_a,\mu_1+\epsilon)$ using the following lemma:
\modKlProperties*
The proof can be found in \Cref{appendix:mod-kl-properties}. To apply the lemma, observe that if $\epsilon\ge\frac{1}{2}$ and $\mu_1<1-\epsilon$, then $\dr{a}\le\epsilon$ for all suboptimal arms, and the lenient regret is zero. Also, since we assume that for all suboptimal arms $\dr{a}>\epsilon$, we have $\mu_a<\mu_1-\epsilon<1-2\epsilon$. Therefore, we can apply this lemma with \Cref{eq:asymptotic bound modKL}, which concludes the proof. 
\end{proof}

\clearpage

\section{Proofs of Auxiliary Results}
\subsection{Proof of Claim \ref{claim: epsilon to standard relation}}
\label{appendix: epsilon to standard relation}
\epsToStandardRelation*
\begin{proof}
First notice that for any $a\in\brs*{\Narms}$, we can write
\begin{align*}
    \dr{a}=\int_{\epsilon=0}^1 f_\epsilon(\dr{a})d\epsilon\enspace.
\end{align*}
Recalling that the regret can be written as $R(T) = \sum_{a=1}^{\Narms} \dr{a} \E\brs*{N_a(T+1)}$, we get
\begin{align*}
    R(T) &= \sum_{a=1}^{\Narms} \dr{a} \E\brs*{N_a(T+1)} \\
    &= \sum_{a=1}^{\Narms} \br*{\int_{\epsilon=0}^1 f_\epsilon(\dr{a})d\epsilon} \E\brs*{N_a(T+1)} \\
    &= \int_{\epsilon=0}^1 \br*{\sum_{a=1}^{\Narms} f_\epsilon(\dr{a})\E\brs*{N_a(T+1)}} d\epsilon \\
    &= \int_{\epsilon=0}^1 \br*{\sum_{t=1}^{T} f_\epsilon(\dr{t})} d\epsilon \\
    &= \int_{\epsilon=0}^1 R_{f_\epsilon}(T)d\epsilon \enspace.
\end{align*}
\end{proof}

\subsection{Proof of Lemma \ref{lemma:mod-kl-properties}}
\label{appendix:mod-kl-properties}
\modKlProperties*
\begin{proof}
Define 
\begin{align*}
    \modklBin(p,q) = p\ln\frac{p}{q} + (1-p-\epsilon)\ln\frac{1-p-\epsilon}{1-q-\epsilon} = (1-\epsilon)\klBin\br*{\frac{p}{1-\epsilon},\frac{q}{1-\epsilon}}\enspace.
\end{align*}
Then, we want to prove that $\modklBin(p,q)\ge \frac{1}{4}\klBin(p,q+\epsilon)$. 
We start by stating the partial derivatives of $\modklBin(p,q)$ and $\klBin(p,q)$ with respect to $q$ and $\epsilon$:
\begin{align}
    \frac{\partial \modklBin(p,q)}{\partial \epsilon} &= \frac{1-p-\epsilon}{1-q-\epsilon} - \ln \frac{1-p-\epsilon}{1-q-\epsilon} -1\nonumber \\
    \frac{\partial \modklBin(p,q)}{\partial q} &= \frac{(1-\epsilon)(q-p)}{q(1-q-\epsilon)} \nonumber\\
     \frac{\partial \klBin(p,q)}{\partial q} &= \frac{q-p}{q(1-q)} \label{eq:mod kl deriv}
\end{align}
Next, for any fixed $p\in[0,1-2\epsilon)$. define the function
\begin{align*}
    g(q,\epsilon) = \modklBin(p,q) - \frac{1}{4}\klBin(p,q+\epsilon)\enspace,
\end{align*}
Our goal is to prove that $g(q,\epsilon)\ge0$ for all $p,q$ and $\epsilon$. The derivative of $g(q,\epsilon)$ w.r.t. $q$ is
\begin{align*}
    \frac{\partial g(q,\epsilon)}{\partial q} & = \frac{(1-\epsilon)(q-p)}{q(1-q-\epsilon)} - \frac{1}{4}\frac{q+\epsilon-p}{(q+\epsilon)(1-q-\epsilon)} \\
    & = \frac{(1-\epsilon)(q-p)(q+\epsilon) - \frac{1}{4}(q+\epsilon-p)q}{q(q+\epsilon)(1-q-\epsilon)} \\
    & = \frac{(1-\epsilon-\frac{1}{4})q(q+\epsilon-p) - p\epsilon(1-\epsilon)}{q(q+\epsilon)(1-q-\epsilon)} \enspace.
\end{align*}
Next, as $\epsilon\le \frac{1}{2}$, notice that all of the coefficient of $p$ are negative, and since $q\ge p+\epsilon$, the derivative can be lower bounded by taking $p=q-\epsilon$:
\begin{align*}
    \frac{\partial g(q,\epsilon)}{\partial q} & \ge \frac{(1-\epsilon-\frac{1}{4})q\brs*{q+\epsilon-(q-\epsilon)} - (q-\epsilon)\epsilon(1-\epsilon)}{q(q+\epsilon)(1-q-\epsilon)} \\
    & = \frac{(1-\epsilon-\frac{1}{2})\epsilon q +\epsilon^2(1-\epsilon)}{q(q+\epsilon)(1-q-\epsilon)}
\end{align*}
Specifically, since $\epsilon\in\brs*{0,\frac{1}{2}}$, the derivative is nonnegative for all $q\in[p+\epsilon,1-\epsilon)$, and thus $g(q,\epsilon)\ge g(p+\epsilon,\epsilon)$ for all such $q$. Next we lower bound $g(p+\epsilon,\epsilon)$ for any $\epsilon$ and any $0\le p < 1-2\epsilon$. We start by treating it as a function of $\epsilon$, namely
\begin{align*}
    g(p+\epsilon,\epsilon) & = \modklBin(p,p+\epsilon) - \frac{1}{4}\klBin(p,p+2\epsilon) \triangleq h(\epsilon)
\end{align*}
We now bound the derivative of $h(\epsilon)$. The derivative of the first term of $h(\epsilon)$ can be bounded as follows:
\begin{align*}
    \frac{\partial \modklBin(p,p+\epsilon)}{\partial \epsilon} &=-\frac{p}{p+\epsilon} +  2\frac{1-p-\epsilon}{1-p-2\epsilon} - \ln \frac{1-p-\epsilon}{1-p-2\epsilon} -1 \\
    & \overset{(*)}{\ge} -\frac{p}{p+\epsilon} +  \frac{1-p-\epsilon}{1-p-2\epsilon} \\
    & = \br*{\frac{p+\epsilon}{p+2\epsilon}-\frac{p}{p+\epsilon}} +  \br*{\frac{1-p-\epsilon}{1-p-2\epsilon} - \frac{p+\epsilon}{p+2\epsilon}} \\
    & = \frac{\epsilon^2}{(p+\epsilon)(p+2\epsilon)} + \frac{\epsilon}{(p+2\epsilon)(1-p-2\epsilon)} \enspace ,
\end{align*}
where $(*)$ is due to the inequality $\ln x \le x-1$. The derivative of the second term equals to 
\begin{align*}
    \frac{\partial \klBin(p,p+2\epsilon)}{\partial \epsilon} &=-2\frac{p}{p+2\epsilon} +  2\frac{1-p}{1-p-2\epsilon} = \frac{4\epsilon}{(p+2\epsilon)(1-p-2\epsilon)} \enspace .
\end{align*}
Combining, we get
\begin{align*}
    h'(\epsilon) &= \frac{\epsilon^2}{(p+\epsilon)(p+2\epsilon)} + \frac{\epsilon}{(p+2\epsilon)(1-p-2\epsilon)} - \frac{1}{4}\frac{4\epsilon}{(p+2\epsilon)(1-p-2\epsilon)} \\
    & = \frac{\epsilon^2}{(p+\epsilon)(p+2\epsilon)}
    \ge 0
\end{align*}
Therefore $h(\epsilon)$ is increasing in $\epsilon$, and for all $\epsilon\in\brs*{0,\frac{1}{2}}$, $h(\epsilon)\ge h(0)$. The proof is concluded by noting that $h(0)=0$ and therefore, for all valid $p,q$ and $\epsilon$, it holds that $g(q,\epsilon)\ge g(p+\epsilon,\epsilon)\ge0$.
\end{proof}

\clearpage

\subsection{Lemma \ref{lemma:kl equation solution increases}}
\label{appendix:kl equation solution increases}
\begin{lemma}
\label{lemma:kl equation solution increases}
For any $0\le p \le q <1-\epsilon$ and any $c\ge0$, let $x_c(p,q)\in\brs*{p,q}$ be the solution of the equation
\[\klBin\br*{\frac{x}{1-\epsilon},\frac{q}{1-\epsilon}} = \frac{1}{1+c}\klBin\br*{\frac{p}{1-\epsilon},\frac{q}{1-\epsilon}}\enspace.\]
Then, $x_c(p,q) \le \frac{c}{1+c}q + \frac{1}{1+c}p$, and for any $\mu\in[q,1-\epsilon)$, it holds that $x_c(p,q) \le x_c(p,\mu)$.
\end{lemma}
\begin{proof}
Similarly to \Cref{lemma:mod-kl-properties}, we define 
\begin{align*}
    \modklBin(p,q) = p\ln\frac{p}{q} + (1-p-\epsilon)\ln\frac{1-p-\epsilon}{1-q-\epsilon} = (1-\epsilon)\klBin\br*{\frac{p}{1-\epsilon},\frac{q}{1-\epsilon}}\enspace.
\end{align*}
Specifically, see that we can equivalently find the solution $x_c(p,q)\in\brs*{p,q}$ of the equation 
\[\modklBin\br*{x,q} = \frac{1}{1+c}\modklBin\br*{p,q}\enspace.\]
Next, note that $\modklBin(x,q)$ strictly decreases in $x$ for $x\in\brs*{p,q}$, and therefore for all $c\ge0$, there exist a unique solution to the equation in this region. For the first part of the proof, notice that $\modklBin$ is a scaled linear transformation of $\klBin$. Since $\klBin$ is convex in its first argument, so does $\modklBin$, and we have
\begin{align*}
    \modklBin(\frac{c}{1+c}q + \frac{1}{1+c}p,q) 
    \le \frac{c}{1+c}\modklBin(q,q) + \frac{1}{1+c}\modklBin(p,q)=\frac{1}{1+c}\modklBin(p,q)\enspace,
\end{align*}
and as $\modklBin(p,q)$ decreases in $p$ for any $p\le q$, we conclude that 
\begin{align}
    x_c(p,q) \le \frac{c}{1+c}q + \frac{1}{1+c}p\enspace. \label{eq:kl equation solution bound}
\end{align}
Next, for $\mu\in[q,1-\epsilon)$, define the function
\begin{align*}
    g(\mu) = \modklBin(x_c(p,q),\mu) - \frac{1}{1+c}\modklBin(p,\mu) \enspace,
\end{align*}
whose derivative is (see \Cref{eq:mod kl deriv})
\begin{align*}
    g'(\mu) &= (1-\epsilon)\frac{\mu-x_c(p,q)}{(\mu+\epsilon)(1-\mu-\epsilon)} - (1-\epsilon)\frac{1}{1+c}\frac{\mu-p}{(\mu+\epsilon)(1-\mu-\epsilon)} \\
    &=\frac{1-\epsilon}{(\mu+\epsilon)(1-\mu-\epsilon)}\br*{\frac{c}{1+c}\mu + \frac{1}{1+c}p-x_c(p,q)}\\
    &\ge 0 \enspace.
\end{align*}
where the inequality is by \Cref{eq:kl equation solution bound} and since $\mu\ge q$. By definition, we have $g(q)=0$, and therefore, for all $\mu\in[q,1-\epsilon)$, it holds that $g(\mu)\ge0$, or 
\[\modklBin(x_c(p,q),\mu) \ge \frac{1}{1+c}\modklBin(p,\mu)\enspace.\]
Finally, recalling that $\modklBin(p,q)$ decreases in $p$ leads to $x_c(p,q)\le x_c(p,\mu)$, which concludes the proof.
\end{proof}

\clearpage

\subsection{Comparison between the bounds for standard and lenient regret}
\label{appendix:standard lenient comparison}
\begin{lemma}\label{lemma:standard lenient comparison}
For any $c\ge1$ and any $0\le p\le q \le\frac{1}{c}$, it holds that 
$$\klBin(cp,cq) \ge c\cdot \klBin(p,q).$$
\end{lemma}
Specifically, for any $\epsilon\in\left[0,1\right)$, fixing $c=\frac{1}{1-\epsilon}$ leads to the following bound:
\begin{align*}
    \klBin\br*{\frac{p}{1-\epsilon}.\frac{q}{1-\epsilon}} 
    \ge \frac{1}{1-\epsilon}\cdot \klBin(p,q)
    \ge \klBin(p,q)\enspace,
\end{align*}
which proves that the bound of \Cref{eq:upper bound asymptotic} is tighter than the bound of \Cref{eq:lenient regret standard}. Also, combining with \Cref{lemma:mod-kl-properties}, we get that for any $\epsilon\in\left[0,\frac{1}{2}\right)$, any  $p\in[0,1-2\epsilon)$ and any $q\in[p+\epsilon,1-\epsilon)$, 
\begin{align*}
    \klBin\br*{\frac{p}{1-\epsilon},\frac{q}{1-\epsilon}}\ge \frac{1}{1-\epsilon}\max\brc*{\frac{1}{4}\klBin(p,q+\epsilon),\klBin(p,q)} \enspace .
\end{align*}
\begin{proof}
Without loss of generality, we assume that $c>1$ and $p\ne q$, as otherwise, the bound trivially holds. Then, we can also assume that $q<\frac{1}{c}$, since $q=\frac{1}{c}$ leads to infinite l.h.s. and finite r.h.s., so the bound holds. Next, define 
$$g(p,q) = \klBin(cp,cq) - c\cdot \klBin(p,q),$$
whose partial derivative w.r.t. $q$ is (e.g., by \Cref{eq:mod kl deriv} and the chain rule)
\begin{align*}
    \frac{\partial g(p,q)}{\partial q} 
    &= c\frac{cq-cp}{cq(1-cq)} - c\frac{q-p}{q(1-q)}
    = \frac{c(q-p)}{q}\br*{\frac{1}{1-cq}-\frac{1}{1-q}}
    = \frac{c(q-p)}{q}\frac{cq-q}{(1-cq)(1-q)}\\
    &=\frac{c(c-1)(q-p)}{(1-cq)(1-q)}
    \ge 0
\end{align*}
where the last inequality holds for any $p\le q$ and $c\ge1$. Specifically, $g(p,q)$ increases in $q$ for $p\le q < \frac{1}{c}$ and thus $g(p,q)\ge g(p,p)=0$, which concludes the proof.
\end{proof}

\clearpage

\section{Additional experimental results}
\label{appendix:experiments}
\subsection{Additional Statistics for the Experiments in Section~\ref{section:experiments}}
\setlength{\tabcolsep}{5pt}
\begin{table}[h]
\centering
\caption{Additional statistics of the empirical evaluation in \Cref{section:experiments}. All evaluations where performed with $\epsilon=0.2$. The table presents the statistics at $T=5000$, using $50,000$ different seeds. std is the standard deviation, and $99\%$ represents the $99^{th}$ percentile. `Standard' measures the standard regret ($f(\dr{})=\dr{}$) and `Hinge' measures the lenient regret w.r.t. the hinge loss ($f(\dr{}) = \max\brc*{\dr{}-\epsilon,0}$).}
\label{table:statistics}
\begin{tabular}{|c|c?c|c|c|c?c|c|c|c|}\hline
  &  & \multicolumn{4}{c?}{\textbf{Thompson sampling}} &  \multicolumn{4}{c|}{\textbf{$\epsilon$-Thompson sampling}} \\ \hhline {~~--------}
   \multirow{-2}{3.1cm}{\centering\textbf{Arm Values}}&  \multirow{-2}{*}{\centering\textbf{Regret Type}} & mean & std & $99\%$ & max & 
    mean & std & $99\%$ & max \\ \hline 
  \rowcolor[gray]{0.9} &
  Standard  & 5.01 & 3.44 & 16.2 & 96.9 & 2.16 & 9.13 & 25.8 & 749.4 \\ 
  \hhline{|*1{>{\arrayrulecolor[gray]{.9}}-}>{\arrayrulecolor[gray]{.5}}|*9{-}|}
   \rowcolor[gray]{0.9}  \multirow{-2}{3.1cm}{\centering$\mu_1=0.9,\,\mu_2=0.6$}& 
   Hinge & 1.67 & 1.15 & 5.4 & 32.3 & 0.72 & 3.04 & 8.6 & 249.8 \\ 
  \arrayrulecolor{black}\hline
  &
  Standard  & 18.18 & 17.8 & 80.85 & 176.15 & 94.53 & 120.78 & 255.5 & 307.65  \\ 
  \hhline{~>{\arrayrulecolor[gray]{.5}}|*9{-}|}
   \multirow{-2}{3.1cm}{\centering$\mu_1=0.9,\,\mu_2=0.85,$ $\mu_3=0.6$}  &  
   Hinge & 1.6 & 1.01 & 4.8 & 11 & 0.33 & 0.89 & 3.8 & 63 \\ 
  \arrayrulecolor{black}\hline
  \rowcolor[gray]{0.9} &
  Standard  & 8.26 & 4.09 & 20.7 & 180 & 5.5 & 3.97 & 18.6 & 180.3  \\ 
  \hhline{|*1{>{\arrayrulecolor[gray]{.9}}-}>{\arrayrulecolor[gray]{.5}}|*9{-}|}
   \rowcolor[gray]{0.9} 
   \rowcolor[gray]{0.9} \multirow{-2}{3.1cm}{\centering$\mu_1=0.5,\mu_2=0.2$ }   &  
   Hinge & 2.75 & 1.36 & 6.9 & 60 & 1.83 & 1.32 & 6.2 & 60.1 \\ 
  \arrayrulecolor{black}\hline
  &
  Standard  & 33.85 & 27.94 & 151.3 & 271.75 & 31.86 & 44.6 & 252.8 & 277.45  \\
  \hhline{~>{\arrayrulecolor[gray]{.5}}|*9{-}|} \multirow{-2}{3.1cm}{\centering$\mu_1=0.5,\,\mu_2=0.45,$ $\mu_3=0.2$ } &  
   Hinge & 2.71 & 1.23 & 6.5 & 12.8 & 1.76 & 1 & 5.2 & 20 \\ 
  \arrayrulecolor{black}\hline
\end{tabular}
\end{table}

\subsection{Reevaluation of the Experiments of Section~\ref{section:experiments} with Smaller Leniency Parameter}
In this appendix, we present experiments similar to the ones of \Cref{section:experiments}, with $\epsilon=0.05$. The experiments were built to be as similar as possible to the original experiments: In the experiments with the low-optimal arm, it remained $\mu^*=0.5$, while in the experiments with the high optimal arm, it was fixed to $\mu^*=1-\frac{\epsilon}{2}$ (as in the main paper). All gaps were reduced by a factor of $4$, which is the ratio of $\epsilon$ between the two experiment sets, so that their size as a function $\epsilon$ will remain the same. Each scenario was evaluated for $100,000$ time steps on $5,000$ different random seeds. The results are presented in \Cref{figure:experiments smaller eps}, while the simulation statistics are presented in \Cref{table:statistics smaller eps}. As expected, the simulations exhibit similar behavior to the ones in the main paper - $\epsilon$-TS enjoys better performance when working with $\epsilon$-gap functions, especially when the optimal arm is high (the constant regret regime). Moreover, $\epsilon$-TS behaves surprisingly well on the standard regret when the optimal arm is low while suffering linear regret when the optimal arm is higher than $1-\epsilon$.

\setlength{\tabcolsep}{2.9pt}
\begin{table}[H]
\centering
\caption{Additional statistics of the experiments in this section. All evaluations where performed with $\epsilon=0.05$. The table presents the statistics at $T=100,000$, using $5,000$ different seeds. std is the standard deviation, and $99\%$ represents the $99^{th}$ percentile. `Standard' measures the standard regret ($f(\dr{})=\dr{}$) and `Hinge' measures the lenient regret w.r.t. the hinge loss ($f(\dr{}) = \max\brc*{\dr{}-\epsilon,0}$).}
\label{table:statistics smaller eps}
\begin{tabular}{|c|c?c|c|c|c?c|c|c|c|}\hline
  &  & \multicolumn{4}{c?}{\textbf{Thompson sampling}} &  \multicolumn{4}{c|}{\textbf{$\epsilon$-Thompson sampling}} \\ \hhline {~~--------}
   \multirow{-2}{3.35cm}{\centering\textbf{Arm Values}}&  \multirow{-2}{*}{\centering\textbf{Regret Type}} & mean & std & $99\%$ & max & 
    mean & std & $99\%$ & max \\ \hline 
  \rowcolor[gray]{0.9} &
  Standard  & 7.09& 5.97& 24.68 & 184.72 & 3.69 & 19.86 & 53.51 & 889.35\\ 
  \hhline{|*1{>{\arrayrulecolor[gray]{.9}}-}>{\arrayrulecolor[gray]{.5}}|*9{-}|}
   \rowcolor[gray]{0.9}  \multirow{-2}{3.35cm}{\centering$\mu_1=0.975,\,\mu_2=0.9$}& 
   Hinge & 2.36 & 2 & 8.26 & 61.57 & 1.23 & 6.62 & 17.84 & 296.45 \\ 
  \arrayrulecolor{black}\hline
  &
  Standard  & 27.91 & 45.13 & 116.326 & 1259.04 & 479.72 & 607.32 & 1258.877 & 1474.4 \\ 
  \hhline{~>{\arrayrulecolor[gray]{.5}}|*9{-}|}
   \multirow{-2}{3.35cm}{\centering\hspace{1em}$\mu_1=0.975,$\hspace{1.5em} $\mu_2=0.9625,\,\mu_3=0.9$}  &  
   Hinge & 2.22 & 1.42 & 6.45 & 11.3 & 0.49 & 2.016 & 6.55 & 89.8 \\ 
  \arrayrulecolor{black}\hline
  \rowcolor[gray]{0.9} &
  Standard  & 35.7 & 28.81 & 108 & 719.93 & 34.37 & 63.15 & 132.24 & 2469.08  \\ 
  \hhline{|*1{>{\arrayrulecolor[gray]{.9}}-}>{\arrayrulecolor[gray]{.5}}|*9{-}|}
   \rowcolor[gray]{0.9} 
   \rowcolor[gray]{0.9} \multirow{-2}{3.35cm}{\centering$\mu_1=0.5,\mu_2=0.425$ }   &  
   Hinge & 11.9 & 9.6 & 36 & 239.98 & 11.46 & 21.05 & 44.08 & 823.03 \\ 
  \arrayrulecolor{black}\hline
  &
  Standard  & 148.31 & 132.58 & 746.26 & 1314.18 & 146.8 & 160.21 & 1124.04 & 1309.69  \\
  \hhline{~>{\arrayrulecolor[gray]{.5}}|*9{-}|} \multirow{-2}{3.35cm}{\centering$\mu_1=0.5,\,\mu_2=0.4875,$ $\mu_3=0.425$ } &  
   Hinge & 11.39 & 6.38 & 30.28 & 51.7 & 10.17 & 6.1 & 28.6 & 50.83 \\ 
  \arrayrulecolor{black}\hline
\end{tabular}
\end{table}
\setlength{\tabcolsep}{5pt}

\begin{figure}[H]
\centering
\subfigure{
\includegraphics[trim=0 15 0 15,clip,width=0.425\linewidth]{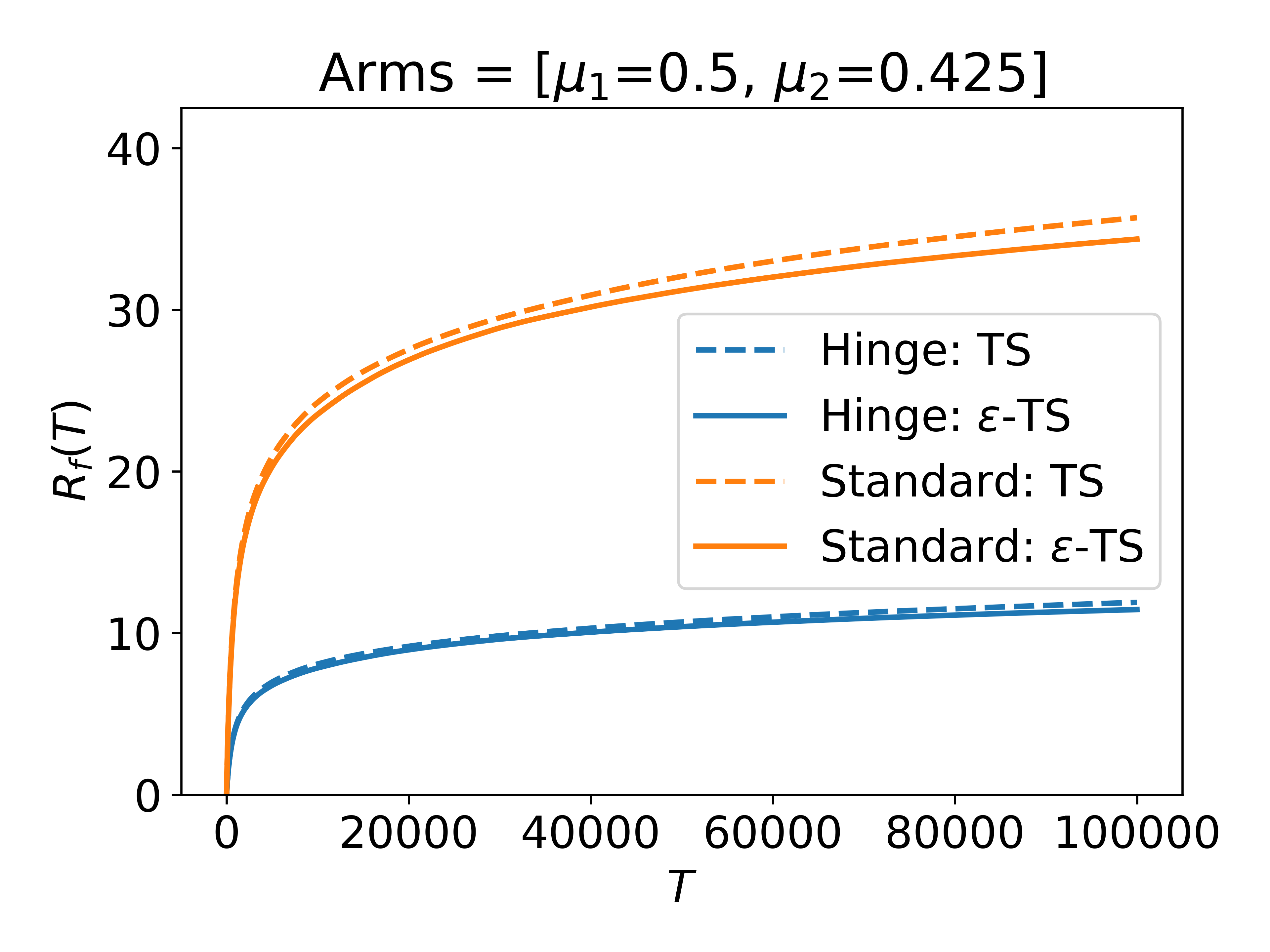}
}
\hspace{0.05\linewidth}
\subfigure{
\includegraphics[trim=0 15 0 15,clip,width=0.425\linewidth]{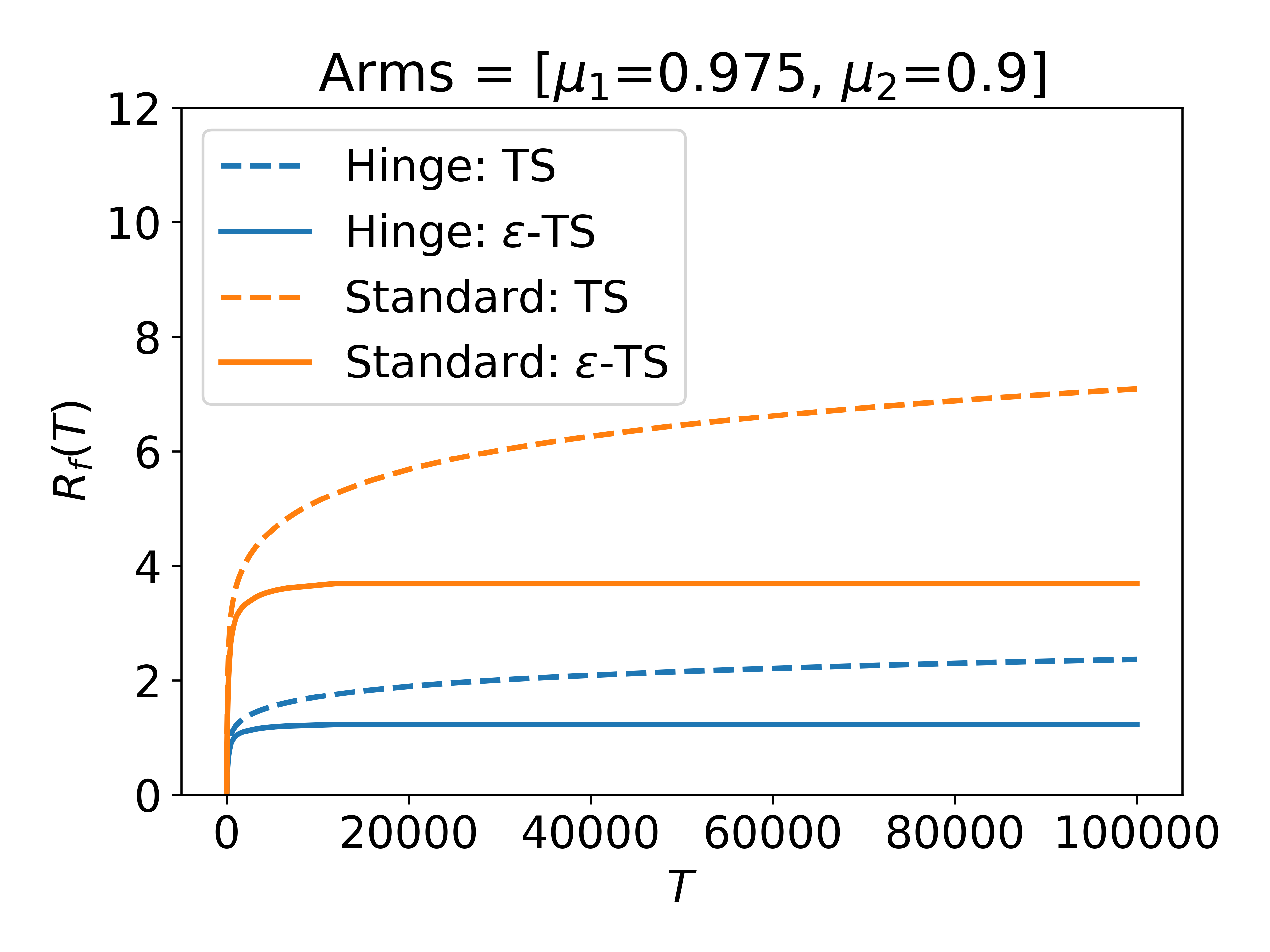} 
} 
\subfigure{
\includegraphics[trim=0 15 0 15,clip,width=0.425\linewidth]{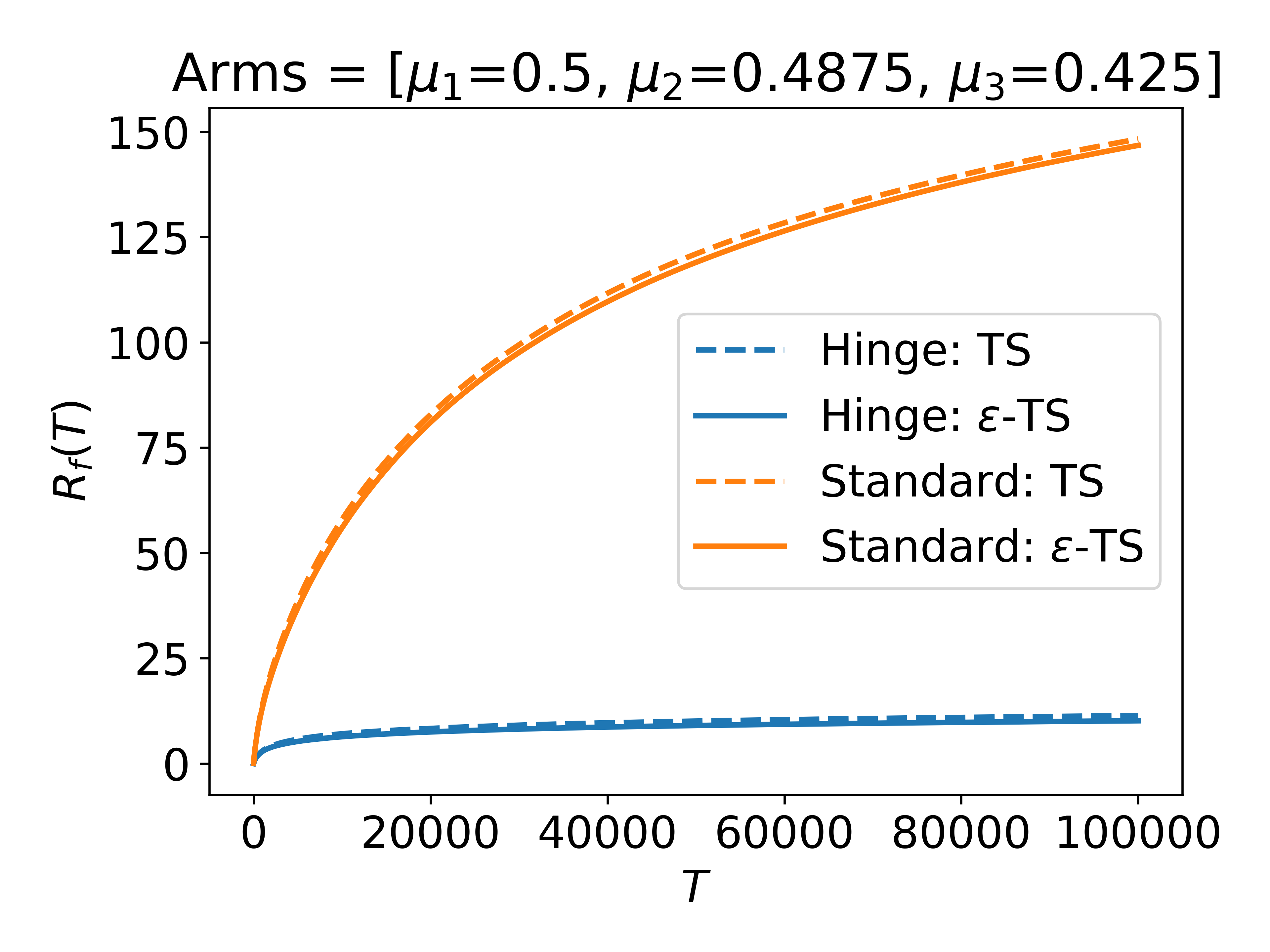}
}
\hspace{0.05\linewidth}
\subfigure{
\includegraphics[trim=0 15 0 15,clip,width=0.425\linewidth]{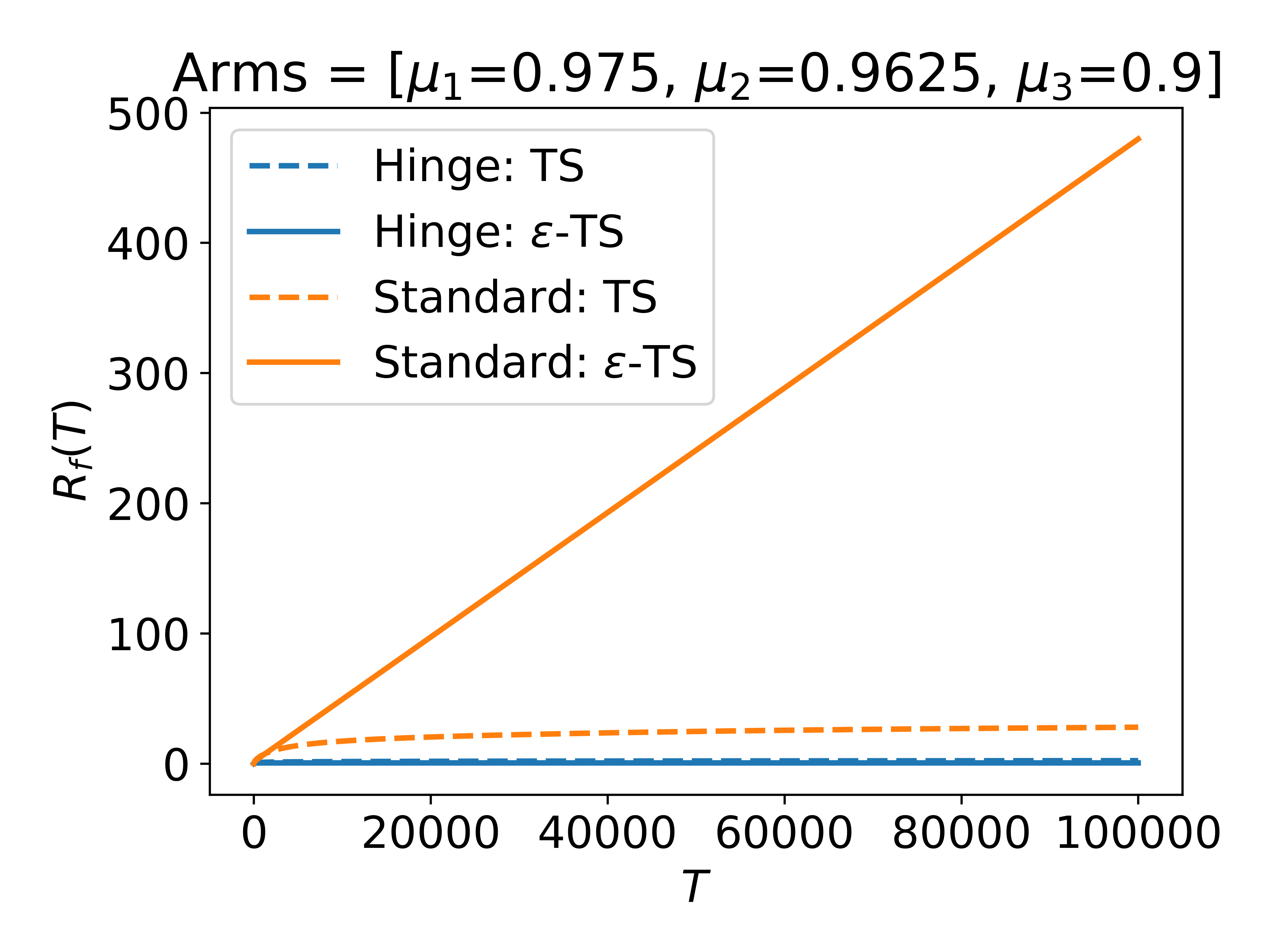} 
}
\caption{Evaluation of $\epsilon$-TS and vanilla TS with $\epsilon=0.05$ and Bernoulli rewards. `Hinge' is the $\epsilon$-gap function $f(\dr{})=\max\brc*{\dr{}-\epsilon,0}$ and `Standard' is the $0$-gap function $f(\dr{})=\dr{}$, which leads to the standard regret.}
\label{figure:experiments smaller eps}
\end{figure}

\subsection{Additional Experiments}
In this subsection, we present additional experiments that were omitted from the paper due to space limits. All simulations were done with $\epsilon=0.2$. In \Cref{subfig:limit_opt} and \Cref{subfig:near_eps_arm} we simulated edge cases that were omitted due to the similarity to the simulations in \Cref{figure:experiments}. In \Cref{subfig:limit_opt}, we simulate the transition point where $\mu^*=1-\epsilon$. Notably, we observe that in this case, the simulation behaves very similarly to the case where $\mu^*>1-\epsilon$. In \Cref{subfig:limit_opt}, we study the case where $\mu_2$ is slightly below $\mu_1-\epsilon$. As expected, the behavior is very similar to the case where $\mu_2<\mu_1-\epsilon$ and the gap is larger, with the only difference that the hinge-loss is smaller. Next, in \Cref{subfig:low_opt_mid_arm_long}, we present a longer version of the bottom-left experiment in \Cref{figure:experiments}, where $\epsilon$-TS surpasses TS also in terms of the standard regret. When running a longer experiment, we see that this continues until the $20,000^{th}$ step, and then TS achieves lower regret. We believe that this is since $\epsilon$-TS gives less focus to arm $a_3$, but still has a good chance for identifying that $a_1$ is optimal. Then, it takes many steps until the possible mistakes of $\epsilon$-TS in identifying $a_1$ are more harmful than the exploration of $a_3$ by TS. Finally, we simulated a problem with 20 randomly-generated arms as follows: the optimal arm was selected to $\mu^*=0.6$. Then, 9 arms were uniformly generated in $\brs*{0.4,0.6}$ and 10 arms were generated in $\brs*{0,0.4}$. The resulting arms are presented in \Cref{table:rand 20 arm vals}. These arms were then fixed, and we simulated the lenient regret for $1,000$ seeds. The results of this simulation are in \Cref{subfig:rand 20}. Interestingly, we see a similar phenomena to that of \Cref{subfig:low_opt_mid_arm_long} -- even after $100,000$ steps, $\epsilon$-TS enjoys better \emph{standard} regret than $TS$. Both simulations hint that for short horizons, $\epsilon$-TS achieves better performance than the vanilla TS, and we believe it is interesting to further study the performance of $\epsilon$-TS in short horizons. Finally, we supply additional statistics on the simulations in \Cref{table:statistics appendix experiments}.
\begin{table}[H]
\centering
\caption{Arm values in \Cref{subfig:rand 20}, sorted in an increasing order}
\begin{tabular}{|c| c | c | c| c| c | c | c|c | c| }\hline
 0.153 & 0.169 & 0.175 & 0.218 & 0.22 & 0.241 & 0.258 & 0.286 & 0.357 & 0.385 \\ \hline
 \rowcolor[gray]{0.9}0.404 & 0.414 & 0.417 & 0.506 & 0.514 & 0.55 & 0.558 & 0.567 & 0.585 & 0.6 \\
 \hline
\end{tabular}
\label{table:rand 20 arm vals}
\end{table}
\begin{table}[H]
\centering
\caption{Additional statistics of the empirical evaluation in \Cref{figure:appendix experiments}. The statistics are calculated at the end of the run and are as in \Cref{table:statistics}.}
\label{table:statistics appendix experiments}
\begin{tabular}{|c|c?c|c|c|c?c|c|c|c|}\hline
  &  & \multicolumn{4}{c?}{\textbf{Thompson sampling}} &  \multicolumn{4}{c|}{\textbf{$\epsilon$-Thompson sampling}} \\ \hhline {~~--------}
   \multirow{-2}{*}{\centering\textbf{Scenario}}&  \multirow{-2}{*}{\centering\textbf{Regret Type}} & mean & std & $99\%$ & max & 
    mean & std & $99\%$ & max \\ \hline 
  \rowcolor[gray]{0.9} &
  Standard  & 6.88 & 4.57 & 20.4 & 280.2 & 2.99 & 12.7 & 26.1 & 1344.9 \\ 
  \hhline{|*1{>{\arrayrulecolor[gray]{.9}}-}>{\arrayrulecolor[gray]{.5}}|*9{-}|}
   \rowcolor[gray]{0.9}  \multirow{-2}{*}{\centering\Cref{subfig:limit_opt}}& 
   Hinge & 2.3 & 1.52 & 6.8 & 93.4 & 1 & 4.23 & 8.7 & 448.3 \\ 
  \arrayrulecolor{black}\hline
  &
  Standard  & 10.69 & 7.48 & 30.66 & 549.15 & 7.66 & 13.68 & 34.02 & 1046.64  \\ 
  \hhline{~>{\arrayrulecolor[gray]{.5}}|*9{-}|}
   \multirow{-2}{*}{\centering \Cref{subfig:near_eps_arm} }  &  
   Hinge & 0.51 & 0.36 & 1.46 & 26.15 & 0.36 & 0.65 & 1.62 & 49.84 \\ 
  \arrayrulecolor{black}\hline
  \rowcolor[gray]{0.9} &
  Standard  & 59.55 & 50.16 & 164.55 & 1550.35 & 71.34 & 322.73 & 597.33 & 5017.9  \\ 
  \hhline{|*1{>{\arrayrulecolor[gray]{.9}}-}>{\arrayrulecolor[gray]{.5}}|*9{-}|}
   \rowcolor[gray]{0.9} 
   \rowcolor[gray]{0.9} \multirow{-2}{*}{\centering \Cref{subfig:low_opt_mid_arm_long} }   &  
   Hinge & 4.08 & 1.6 & 8.6 & 11.6 & 2.58 & 1.23 & 6.3 & 10 \\ 
  \arrayrulecolor{black}\hline
  &
  Standard  & 472.64 & 165.14 & 1110.64 & 1948.38 & 331.02 & 272.78 & 1739.39 & 2060.39 \\
  \hhline{~>{\arrayrulecolor[gray]{.5}}|*9{-}|} \multirow{-2}{*}{\centering \Cref{subfig:rand 20} } &  
   Hinge & 42.26 & 5.24 & 55.57 & 58.48 & 24.31 & 3.78 & 34.91 & 38.54 \\ 
  \arrayrulecolor{black}\hline
\end{tabular}
\end{table}
\setlength{\tabcolsep}{6pt}

\begin{figure}[H]
\centering
\subfigure[50,000 seeds, $\mu^*=1-\epsilon$]{
\includegraphics[trim=0 15 0 15,clip,width=0.4\linewidth]{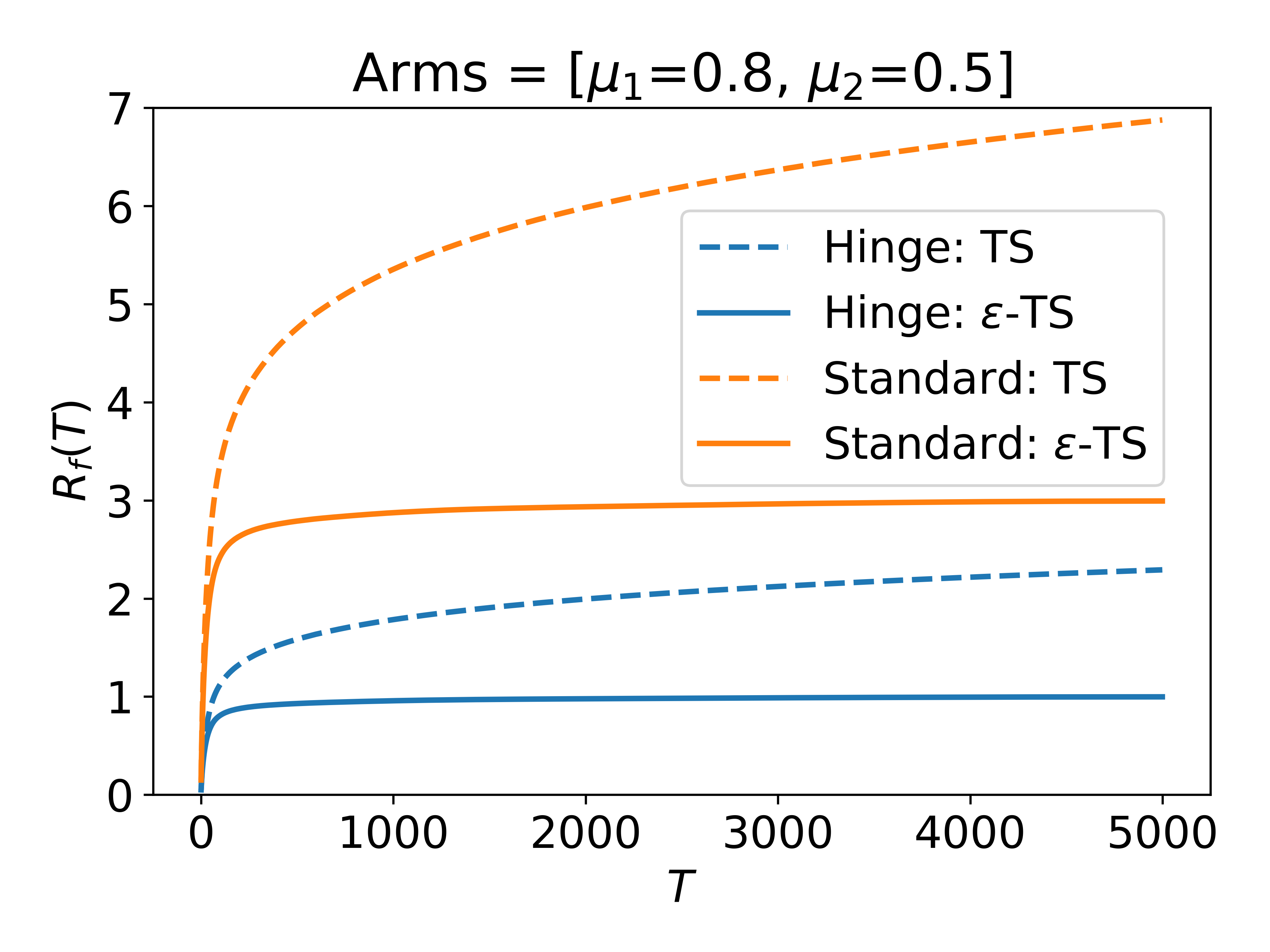}
\label{subfig:limit_opt}
}
\hspace{0.05\linewidth}
\subfigure[50,000 seeds, $\mu_2$ just below $\mu^*-\epsilon$]{
\includegraphics[trim=0 15 0 15,clip,width=0.4\linewidth]{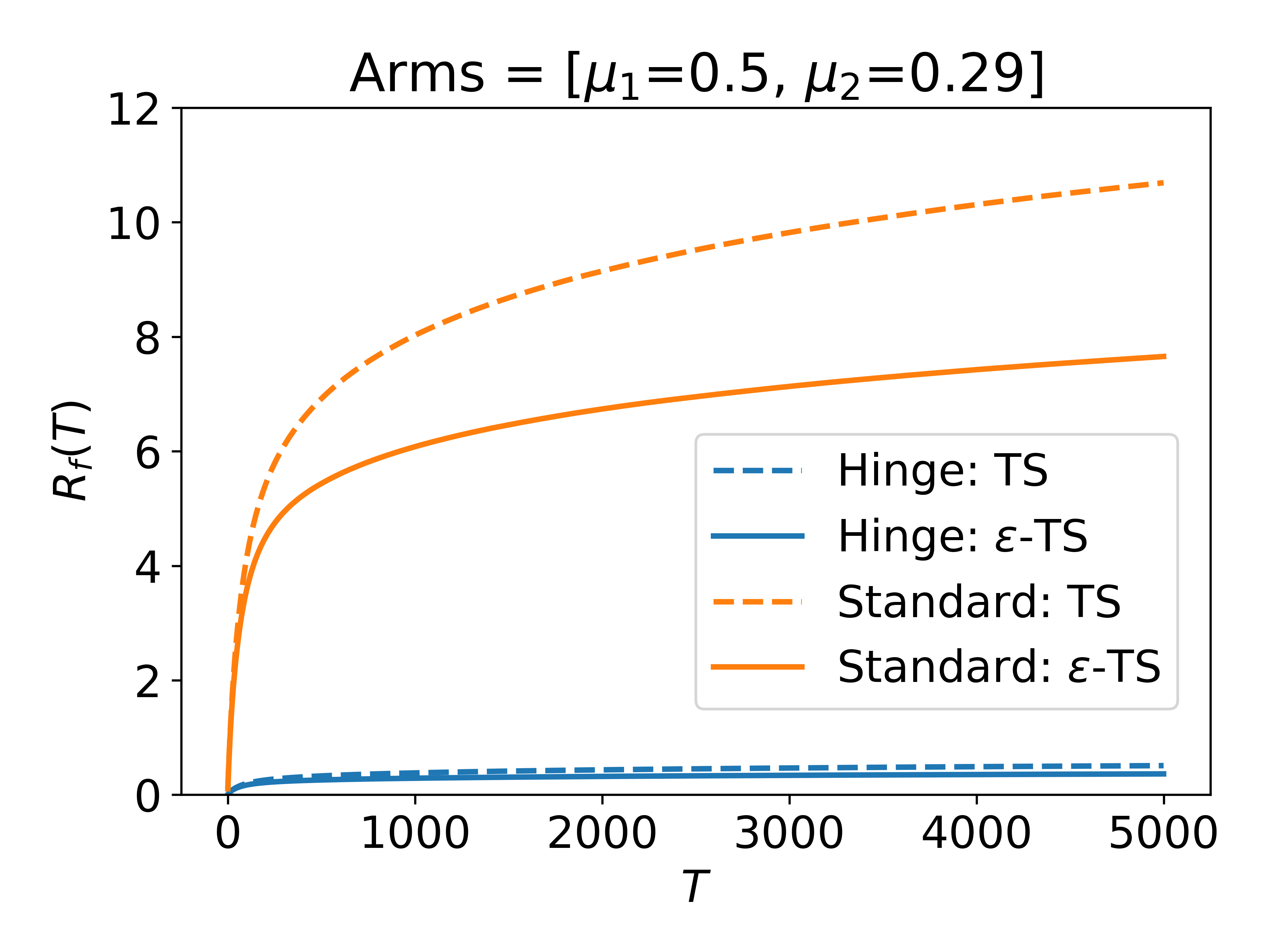} 
\label{subfig:near_eps_arm}} 
\subfigure[5,000 seeds, longer horizon for the experiment in \Cref{figure:experiments}]{
\includegraphics[trim=0 15 0 15,clip,width=0.44\linewidth]{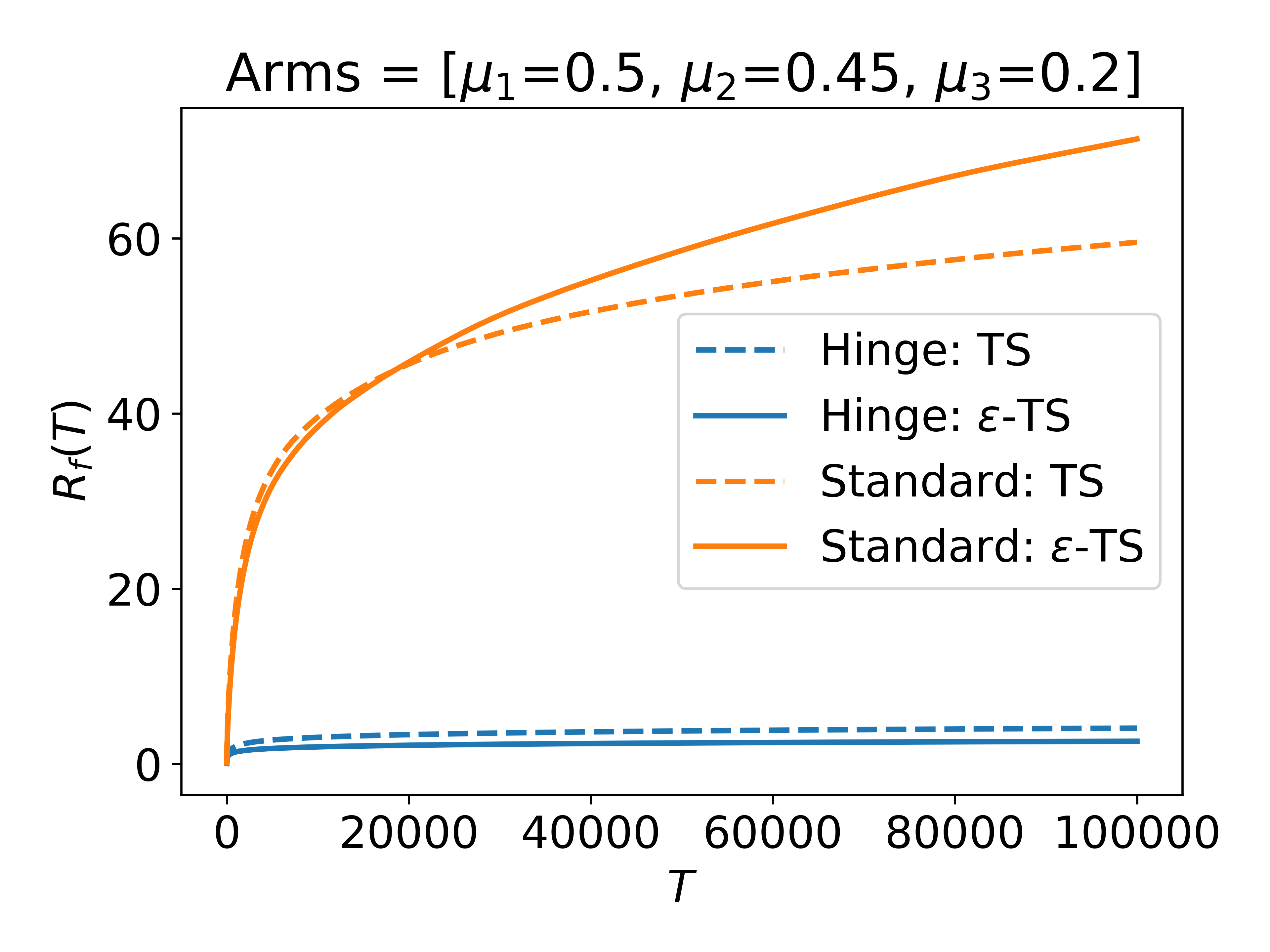}
\label{subfig:low_opt_mid_arm_long}}
\hspace{0.04\linewidth}
\subfigure[1,000 seeds, 20 random arms that were generated once with $\mu^*=0.6$]{
\includegraphics[trim=0 15 0 15,clip,width=0.44\linewidth]{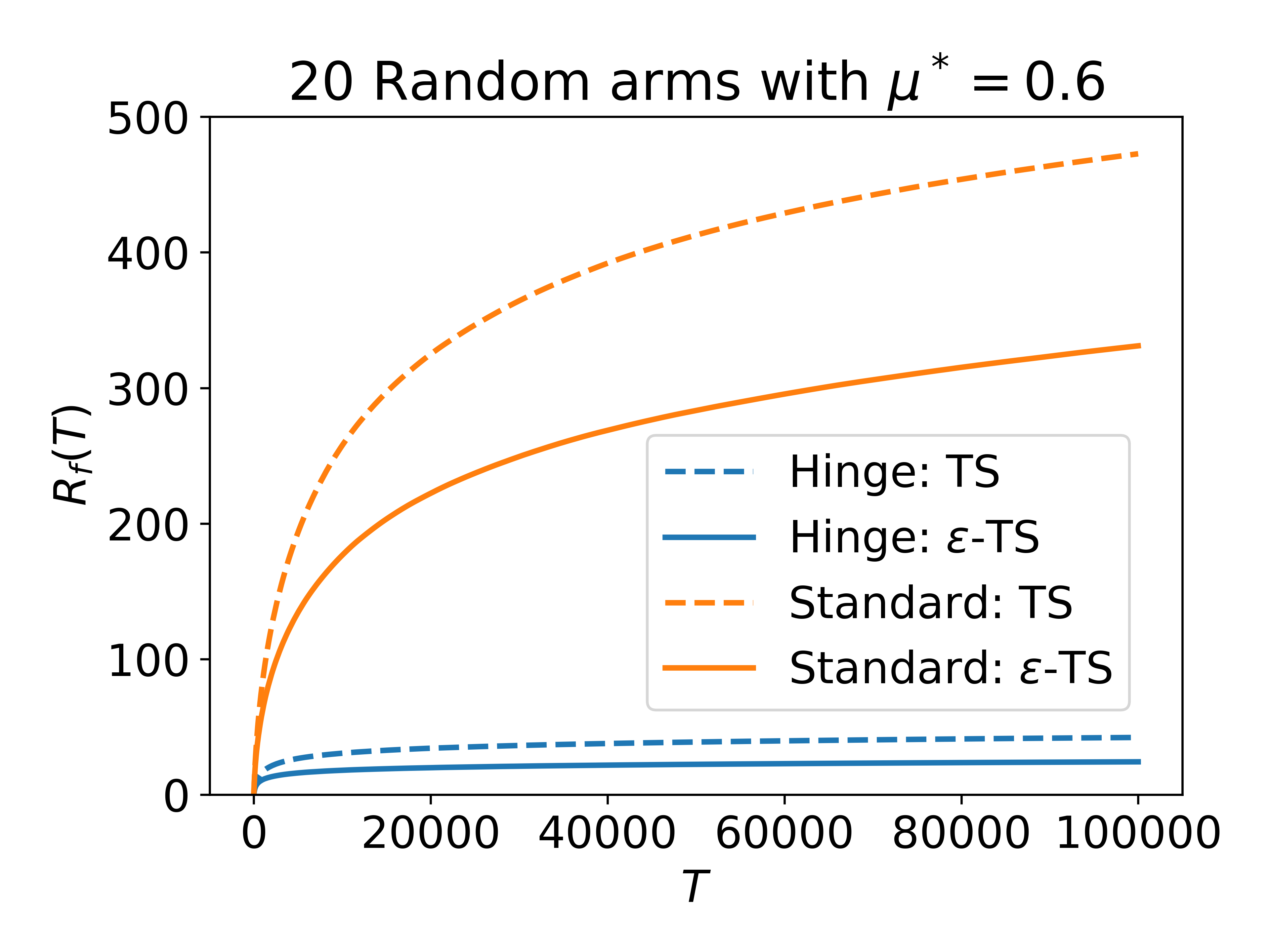} 
\label{subfig:rand 20}}
\caption{Evaluation of $\epsilon$-TS and vanilla TS with $\epsilon=0.2$ and Bernoulli rewards.`Hinge' is the $\epsilon$-gap function $f(\dr{})=\max\brc*{\dr{}-\epsilon,0}$ and `Standard' is the $0$-gap function $f(\dr{})=\dr{}$, which leads to the standard regret.}
\label{figure:appendix experiments}
\end{figure}
\end{document}